\documentclass[twoside]{article}

\usepackage[preprint]{aistats2026}
\usepackage[round]{natbib}

\usepackage{xcolor}         
\usepackage{bbm}

\usepackage{etoolbox}
\usepackage{tocloft}

\newcommand{\appendixtocname}{\ }
\newcommand{\appendixtableofcontents}{
  \begingroup
  \renewcommand{\cftsecleader}{\cftdotfill{\cftdotsep}} 
  \renewcommand{\cftsecfont}{}
  \renewcommand{\cftsecpagefont}{}
  \setcounter{tocdepth}{2}
  \section*{\appendixtocname}
  \addcontentsline{toc}{section}{\appendixtocname}
  \begingroup
    \let\clearpage\relax
    \tableofcontents
  \endgroup
  \endgroup
}

\usepackage{color}
\usepackage{colortbl}

\usepackage{times}
\usepackage{epsfig}
\usepackage{graphicx}
\usepackage{amsmath}
\usepackage{amssymb}

\usepackage{multirow}
\usepackage{booktabs}

\usepackage{enumitem}
\usepackage{nicefrac}
\usepackage{tcolorbox}

\usepackage{tikz}
\usepackage{tikz-3dplot}
\usetikzlibrary {arrows.meta}
\usetikzlibrary{matrix,fit}

\usepackage{ifthen}

\usetikzlibrary{decorations.markings, decorations.pathmorphing, patterns, shapes.geometric}

\usetikzlibrary{positioning}

\usepackage{amsfonts}
\usepackage{bm}
\usepackage{bbm}
\usepackage{amsmath,amsthm,amssymb,rotating, mathrsfs}
\usepackage{footmisc}

\usepackage{bibunits}

\newtheorem{theorem}{Theorem}

\newtheorem{proposition}{Proposition}
\newtheorem{corollary}{Corollary}
\newtheorem{lemma}{Lemma}

\theoremstyle{definition}

\newtheorem{example}{Example}

\newtheorem{assumption}{Assumption}

\theoremstyle{remark}
\newtheorem{remark}{Remark}

\usepackage[linesnumbered,lined,boxed,commentsnumbered,ruled,vlined]{algorithm2e}

\SetCommentSty{mycommfont}

\usepackage{bm}
\definecolor{myblue}{rgb}{0.21,0.49,0.74}
\definecolor{mypurple}{RGB}{100, 50, 168}
\definecolor{myred}{rgb}{0.82, 0.1, 0.26}
\usepackage[pagebackref,breaklinks,colorlinks,citecolor=myblue,linkcolor=myblue]{hyperref}
\usepackage[capitalise]{cleveref}
\crefformat{equation}{(#2#1#3)}
\usepackage{crossreftools}
\crefname{assumption}{Assumption}{Assumptions}  
\Crefname{assumption}{Assumption}{Assumptions}  

\usepackage{mathtools}
\usepackage{cancel} 

\usepackage{xparse}

\newcommand{\cD}{\mathcal{D}}

\newcommand{\cF}{\mathcal{F}}

\newcommand{\cH}{\mathcal{H}}

\newcommand{\cL}{\mathcal{L}}
\newcommand{\cM}{\mathcal{M}}
\newcommand{\cN}{\mathcal{N}}
\newcommand{\cO}{\mathcal{O}}
\newcommand{\cP}{\mathcal{P}}

\newcommand{\cR}{\mathcal{R}}

\newcommand{\cY}{\mathcal{Y}}

\newcommand{\bbB}{\mathbb{B}}

\newcommand{\bbE}{\mathbb{E}}

\newcommand{\bbP}{\mathbb{P}}

\newcommand{\bbR}{\mathbb{R}}

\NewDocumentCommand{\norm}{mG{2}}{\big\|#1\big\|_{#2}}

\DeclareMathOperator{\dist}{dist}

\DeclareMathOperator{\diam}{diam}

\newcommand{\argmin}{\mathop{\rm argmin}}

\NewDocumentCommand{\seqp}{mG{n}}{{#1}_1-\cdots+ {#1}_{#2}}
\NewDocumentCommand{\seqm}{mG{n}}{{#1}_1-\cdots- {#1}_{#2}}

\newcommand{\myparagraph}[1]{\smallskip\noindent\textbf{#1.}}

\def\M{{\cal M}}

\def\st{\mbox{ s.t. }}
\def\nm{\Vert}

\renewcommand{\and}{\mbox{$\wedge$}}

\newcommand{\ec}{\end{center}}
\newcommand{\ee}{\end{equation}}
\newcommand{\ed}{\end{displaymath}}
\newcommand{\ea}{\end{array}}
\newcommand{\ben}{\begin{enumerate}}
\newcommand{\een}{\end{enumerate}}
\newcommand{\bit}{\begin{itemize}}
\newcommand{\eit}{\end{itemize}}
\newcommand{\beq}{\begin{eqnarray}}
\newcommand{\eeq}{\end{eqnarray}}
\newcommand{\btab}{\begin{tabular}}
\newcommand{\etab}{\end{tabular}}
\newcommand{\mathbfig}{\begin{figure}}
\newcommand{\efig}{\end{figure}}
\newcommand{\btp}{\begin{tikzpicture}}
\newcommand{\etp}{\end{tikzpicture}}

\newcommand{\nmm}[1]{ \nm #1 \nm }

\newcommand{\IP}[2]{ \langle #1 , #2 \rangle }

\def\nmsl1{\nm_{{\rm SL1}}}

\def\exp{{\mathsf{exp}}}

\def\poly{{\mathsf{poly}}}

\begin{document}

%
\runningtitle{Recovery Guarantees for Continual Learning of Dependent Tasks}

%
\runningauthor{Peng, Tadipatri, Xu, Eaton, Vidal}

\twocolumn[

\aistatstitle{Recovery Guarantees for Continual Learning of Dependent Tasks: Memory, Data-Dependent Regularization, and Data-Dependent Weights}

\aistatsauthor{ \centering Liangzu Peng$^*$ \And  Uday Kiran Reddy Tadipatri$^*$ \And Ziqing Xu \And Eric Eaton \quad \quad \quad Ren\'e Vidal \quad \quad \quad \quad  }

\aistatsaddress{University of Pennsylvania} ]

\begingroup
\renewcommand\thefootnote{}\footnotetext{$^*$: Equal contribution.}
\endgroup

\begin{abstract}
    Continual learning (CL) is concerned with learning multiple tasks sequentially without forgetting previously learned tasks. Despite substantial empirical advances over recent years, the theoretical development of CL remains in its infancy. At the heart of developing CL theory lies the challenge that the data distribution varies across tasks, and we argue that properly addressing this challenge requires understanding this variation--dependency among tasks. To explicitly model task dependency, we consider nonlinear regression tasks and propose the assumption that these tasks are dependent in such a way that the data of the current task is a nonlinear transformation of previous data. With this model and under natural assumptions, we prove statistical recovery guarantees (more specifically, bounds on estimation errors) for several CL paradigms in practical use, including experience replay with data-independent regularization and data-independent weights that balance the losses of tasks, replay with data-dependent weights, and continual learning with data-dependent regularization (e.g., knowledge distillation). To the best of our knowledge, our bounds are informative in cases where prior work gives vacuous bounds.
\end{abstract}

\section{INTRODUCTION}
Continual learning (CL) aims to learn multiple tasks presented sequentially, with a key goal to address the situation of \textit{catastrophic forgetting} \citep{Nccloskey-1989}: learning new tasks risks performance degradation on previously learned tasks. To reduce forgetting, \textit{memory-based methods} store some past data, to be used with new data for training the new task \citep{Robins-1993,Shin-NeurIPS2017,Aljundi-NeurIPS2019,Chaudhry-arXiv2019v4,Prabhu-ECCV2020,Verwimp-ICCV2021,Bang-CVPR2021,Wang-ICLR2022}; \textit{regularization-based methods} optimize the current task with a regularization term that encourages proximity to the model previously learned \citep{Li-TPAMI2017,Kirkpatrick-2017,Zenke-ICML2017,Rebuffi-CVPR2017,Park-ICCV2019,Buzzega-NeurIPS2020}; \textit{constrained optimization methods} enforce non-forgetting requirements as optimization constraints while solving the current task \citep{Lopez-NeurIPS2017,Chaudhry-ICLR2019,Peng-ICML2023,Elenter-arXiv2023v2,Li-arXiv2024v3model}. 

In contrast to the abundance of empirical advances, theoretical investigations in CL are relatively scarce. Existing theoretical contributions are of at least two types. The first type considers linear models or the kernel regime \citep{Bennani-arXiv2020,Doan-AISTATS2021,Heckel-AISTATS2022,Evron-COLT2022,Lin-ICML2023,Li-CoLLAs2023,Wwartworth-NeurIPS2023,Goldfarb-ICLR2024,Zhao-ICML2024,Ding-ICML2024b,Banayeeanzade-TMLR2025,Evron-arXiv2025}. With such specific settings, it is possible to derive meaningful and even tight bounds, but the resulting theory might not apply to existing CL methods that are specifically designed for deep networks or general nonlinear models. The other type of work considers general models (e.g., within a PAC-Bayes framework or based on the Rademacher complexity) \citep{Pentina-NeurIPS2015,Yin-arXiv2020v3,Ye-NeurIPS2022,Friedman-arXiv2024v2}. With such general settings, it is possible to derive CL guarantees by leveraging tools from classic (PAC) learning theory \citep{Shalev-book2014,Mohri-book2018,Alquier-FTML2024}, but the resulting theory often fails to capture the benefits of learning multiple tasks sequentially, or the resulting bound does not necessarily tend to zero even in the presence of infinitely many samples.

In light of the above, we are motivated to seek a middle ground that combines the best of both worlds: derive meaningful error bounds for general nonlinear models and commonly used CL paradigms. 

\myparagraph{Importance of Task Dependency} We argue that attaining the above goal requires modeling how tasks are related to each other. We illustrate this by way of  examples. 
\begin{example}\label{remark:discrepancy-intro}
    The error bounds (specifically, generalization bounds) of \citet{Ye-NeurIPS2022,Friedman-arXiv2024v2,Mansour-arXiv2009} grow linearly with the ``\textit{distance}'' between the data distributions of two tasks. For two zero-mean Gaussian distributions with variances $1$ and $s^2$, respectively, this distance is proportional to $|s^2-1|$, and thus grows unbounded as $s$ increases (cf. \cref{remark:discrepancy}). These are two basic distributions related in a simple way, i.e., one is a scalar multiple of the other, yet the bound depending linearly on $s^2$ becomes vacuous even for a mildly large $s$. 
\end{example}
\begin{example}\label{example:ICL-intro}
    In the work of \citet{Peng-ICML2023}, the data within each task are i.i.d., and data across tasks are dependent. The proof of \citet{Peng-ICML2023} reduces from a complicated CL situation to the standard scenario of learning a single task from i.i.d. data. As a side effect of this reduction, the bound of \textit{every} task in  \citet{Peng-ICML2023} depends logarithmically on the total number $T$ of seen tasks and thus worsens as $T$ grows; this is somehow counterintuitive as it indicates learning more tasks enlarges the error bound of \textit{every} task. In hindsight, analyzing dependent data appears to be intractable under the general setting of \citet{Peng-ICML2023} as they do not specify \textit{how} the tasks depend on each other. This leaves open the challenge of identifying task dependency that avoids degrading, or ideally improves, continual learning performance.
\end{example}

\myparagraph{The Proposed CL Model With Task Dependency} In order to address issues pertaining to \cref{remark:discrepancy-intro,example:ICL-intro} and to furthermore develop meaningful error bounds, we propose a CL framework with explicit modeling of task dependency. Specifically, we consider the problem of \textit{continual noisy nonlinear regression}: learn a function $f^*$ from a sequence of noisy nonlinear regression tasks with $f^*$ assumed to be the shared \textit{true predictor} of all these tasks. This is a setting that generalizes previous works on \textit{continual linear regression}  to the nonlinear, noisy case. Recognizing that directly analyzing such continual nonlinear regression model might lead to unsatisfactory bounds (\cref{remark:discrepancy-intro,example:ICL-intro}), we then arm this model with \textit{task dependency}, which posits that the data of the present task are obtained as a nonlinear yet unknown transformation of the data from previous tasks. This task dependency is motivated from several perspectives (\cref{subsection:motivation-context}). For example, it draws inspiration from the \textit{rotated MNIST} and \textit{permuted MNIST} datasets that have benchmarked many fundamental CL methods; there, the transformation is either some rotation or permutation. Also, it finds inspiration from \textit{dynamical systems}, where the current state is obtained as a transformation of the previous state.

\myparagraph{Implications of Our CL Model and Task Dependency} We now sketch how the issues in \cref{remark:discrepancy-intro,example:ICL-intro} are resolved with the proposed model and task dependency. For the two tasks in \cref{remark:discrepancy-intro}, the nonlinear transformation is simply a scaling function that multiplies its input by $s$; and $s$ is the \textit{Lipschitz constant} of this transformation. Crucially, we find that all our theorems exhibit only a \textit{logarithmic} dependency on $s$ in the context of \cref{remark:discrepancy-intro}, thereby allowing $s$ to grow polynomially with problem size (e.g., dimension, sample size), without drastically affecting our bounds. This is a significant improvement over the linear dependency of $s^2$ in \cref{remark:discrepancy-intro} \citep{Mansour-arXiv2009,Ye-NeurIPS2022,Friedman-arXiv2024v2}. 

Unlike \cref{example:ICL-intro}, 
our task dependency assumption makes the analysis tractable, though still non-trivial (cf. \cref{fig:3dependency}). In particular, it enables us to prove concentration bounds for dependent data, and it suffices to apply the bounds only once. By doing so, we eliminate the logarithmic dependency on $T$ as discussed in \cref{example:ICL-intro} where the concentration inequalities are invoked $T$ times.

More importantly, the proposed model and task dependency allow us to develop theoretical guarantees under a unifying framework and in a systematic fashion. Under basic assumptions, we bound the estimation errors of recovering $f^*$ for the aforementioned CL paradigms, including memory-based methods, regularization-based methods, and constrained optimization methods. In more detail: 

\begin{itemize}[wide]
    \item We begin by analyzing \textit{weighted experience replay with data-independent regularizers}, where the weights balance contributions of each task  (\cref{subsection:guarantee1}). With a non-uniform choice of weights, our bound on the weighted estimation error is inversely proportional to the total number of available data; such a bound is optimal. With uniform weights, our bound is also optimal if the replay buffer stores a constant fraction of the full data of each task.
    \item Then, we analyze regularization-based methods with a \textit{knowledge distillation} regularizer (\cref{subsection:guarantee2-regularization}). While such regularizer is \textit{data-dependent}, we can still prove  estimation error bounds via a non-trivial recursive argument, resulting in improvements over prior work in terms of generality \citep{Heckel-AISTATS2022,Li-CoLLAs2023,Zhao-ICML2024,Zhu-arXiv2025} and tightness \citep{Yin-arXiv2020v3} (cf. \cref{remark:regularization-compare}). 
    \item Lastly, we draw motivations from the constrained learning framework \citep{Chamon-TIT2022,Elenter-arXiv2023v2} and derive error bounds for \textit{replay with data-dependent weights}, where the weights might now be thought of as dual variables in primal dual optimization (\cref{subsection:guarantee3-weights}). We obtain error bounds of a similar flavor by extending our previous results to account for the fact that the weights are now random variables as well, depending on the data. 
\end{itemize}

\begin{figure*}
    \hspace{4em}%
    \begin{tikzpicture}
        \draw[black] (-1.5,0.4) rectangle (-3.5,-2.2);
    
        \node[draw, circle,inner sep=0.25mm,align=center] (x1) at (-2.2,0) {$x_1$}; 
        \node[draw, circle,inner sep=0.25mm,align=center] (x2) at (-2.2,-0.9) {$x_2$};
        \node[draw, circle,inner sep=0.25mm,align=center] (x3) at (-2.2,-1.8) {$x_3$};
        
        \draw[-latex, myred, -{Stealth[length=2mm]}] (x1) -- (x2) node [midway, left] {$g_1$};
        
        \draw[-latex, myblue, -{Stealth[length=2mm]}] (x2) -- (x3)  node [midway, left] {$g_2$};
        \draw[-latex, myblue, -{Stealth[length=2mm]}] (x1) to[out=200,in=160] (x3);
    \end{tikzpicture}
    \hspace{3em}%
    \begin{tikzpicture}[
		every matrix/.style={
			matrix of math nodes,
			column sep =0.5mm,
			row sep =0.5mm,
			left delimiter={[},
			right delimiter={]},
		}
		]
		\matrix (m) {
			x_{11} & x_{12} & x_{13} & x_{14} \\
			x_{21} & x_{22} & x_{23} & x_{24} \\
			x_{31} & x_{32} & x_{33} & x_{34} \\
		};
		\node[fit=(m-1-1)(m-1-4), draw, inner sep=0pt,thick, myblue,rounded corners=5pt] (iid) {};

        \node[above  = 0.4cm  of m] (iid2) {\footnotesize{\textit{\textcolor{myblue}{task $1$ has i.i.d. samples}}}};
		\draw[-latex, myblue, dashed, -{Stealth[length=2mm]}] (iid2) to[out=270,in=90] (iid);
    \end{tikzpicture}
    \hspace{3em}%
    \begin{tikzpicture}[
		every matrix/.style={
			matrix of math nodes,
			column sep =0.5mm,
			row sep =0.5mm,
			left delimiter={[},
			right delimiter={]},
		}
		]
		\matrix (m) {
			x_{11} &  & x_{13}  &  \\
			     & x_{22} & x_{23} & x_{24} \\
			x_{31} & x_{32} & x_{33} & x_{34} \\
		};






        \node[right  = 0.5cm  of m-1-4] (r1) {\footnotesize{\textit{\textcolor{myblue}{$\cR_1=\{1,3\}$}}}};
        \node[right  = 0.5cm  of m-2-4] (r2) {\footnotesize{\textit{\textcolor{myblue}{$\cR_2=\{2,3,4\}$}}}};
        \node[right  = 0.5cm  of m-3-4] (r3) {\footnotesize{\textit{\textcolor{myblue}{$\cR_3=\{1,2,3,4\}$}}}};
    \end{tikzpicture}

    \hspace{4em}%
    \makebox[0.18\textwidth][l]{\hspace{-0.5em}\footnotesize (a) data dependency }
    \makebox[0.2\textwidth]{\footnotesize (b) full data }
    \makebox[0.43\textwidth]{\footnotesize (c) training data available at task $3$ }

    \caption{Example setup  ($T=3,m=4$).  \cref{fig:data-assumption}a: task dependency \cref{eq:dependency}; \cref{fig:data-assumption}b: full data in a matrix, where each column is generated as per \cref{eq:dependency} and each row represents data of each task; \cref{fig:data-assumption}c: index sets $\cR_t$ and data available at task $3$.   \label{fig:data-assumption} }
\end{figure*}

\section{PROBLEM SETUP}
In \cref{subsection:data-model}, we introduce our data model with the task dependency specified in an autoregressive fashion, based upon which we will develop our theorems. In \cref{subsection:motivation-context} we provide justifications and motivations on our proposed task dependency model.





\subsection{Data Model, Task Dependency, and Samples}\label{subsection:data-model} 
\myparagraph{Data Model with Autoregressive Task Dependency} 
Let $T$ be the number of tasks seen thus far. Let $[T]:=\{1,\dots,T\}$. We consider a nonlinear regression setting, where each task $t$ shares a true predictor $f^*: \bbR^{d_x} \to \bbR^{d_y}$ that maps input $x_t$ to output $y_t$ up to some random noise $v_t$, similarly to prior work \citep{Evron-COLT2022,Peng-ICML2023,Li-CoLLAs2023,Zhao-ICML2024,Elenter-arXiv2023v2,Zhu-arXiv2025}. In addition, we specify the task dependency in an autoregressive manner through some deterministic, nonlinear transformation $g_t:\bbR^{(t-1)\times d_x}\to \bbR^{d_x}$. Concretely, we consider the model
\begin{align}
    y_t &= f^*(x_t) + v_t, \quad \quad \ \ \ \ \forall t\geq 1; \label{eq:regress} \\ 
    x_t &= g_t(x_1,\dots,x_{t-1}), \quad \forall t >1. \label{eq:dependency}
\end{align}
Here, $x_1\in\bbR^{d_x}$ is a random vector drawn from some probability distribution $\cD_1$. Thus, \cref{eq:regress} and \cref{eq:dependency} implicitly specify the distribution of $y_t$ (conditioned on $x_t$) and of $x_t$ (conditioned on $x_1,\dots,x_{t-1}$). By defining the relationship between data of different tasks, the transformation $g_t$ captures the dependency among these tasks (cf. \cref{fig:data-assumption}a). While  we interpret  \cref{eq:regress} and \cref{eq:dependency} as a task dependency model in the CL context, our main motivation is from the control literature:  \cref{eq:regress} and \cref{eq:dependency} define a nonlinear dynamical system where $x_t$'s are system states and $f^*$ consists of system parameters to be identified. Different from standard dynamical systems, our model has no control input and \cref{eq:dependency} is in the absence of noise. This is for the sake of simplicity, and it is not difficult to extend our results to the case of \textit{noisy} task dependency where $x_t$ is equal to $g_t(x_1,\dots,x_{t-1})$ up to additive random noise. One more difference is that  in \cref{eq:dependency} we require $x_t$ to depend on all the past $x_1,\dots,x_{t-1}$, while a common special case is of the Markov type $x_t=g_t(x_{t-1})$. We discuss more examples and motivations in \cref{subsection:motivation-context}.

\myparagraph{Samples and Memory} For task $1$, we sample $m$ i.i.d. input data points $\{x_{1,i} \}_{i=1}^m$ from distribution $\cD_1$, that is
\begin{align}\label{eq:iid-task1x}
    \{x_{1,i} \}_{i=1}^m  \overset{\text{i.i.d.}}{\sim} \cD_1.
\end{align}
We will omit the comma and write $x_{1i}$ for $x_{1,i}$ if it does not cause confusion. For task $t>1$, we generate input-output pairs ($x_{ti},y_{ti}$) as per \cref{eq:regress} and \cref{eq:dependency}, that is we have 
\begin{align}\label{eq:dependency-sample-level}
    x_{ti} = g_t( x_{1i},\dots,x_{t-1, i}), \quad \forall t>1, i\in[m],
\end{align}
and $y_{ti}=f^*(x_{ti}) + v_{ti}$ for some noise $v_{ti}$. \cref{fig:data-assumption}b visualizes the input samples  $\{x_{ti}\}_{t\in[T],i\in[m]}$ in a matrix form.

Motivated by memory-based methods in CL, we furthermore assume the availability of only part of the samples indexed by some fixed subset $\cM$ of $[T] \times [m]$. That is, $(x_{ti},y_{ti})$ is available for training if and only if $(t,i)\in \cM$. The available inputs $(x_{ti})_{(t,i)\in \cM}$ can be arranged into a $T\times m$ partial matrix (cf. \cref{fig:data-assumption}c). The $t$-th row of this matrix corresponds to the stored data $\{(x_{ti}, y_{ti})\}_{i\in \cR_t}$ of task $t$; we index them by $\cR_t$, a subset of $[m]$ (\cref{fig:data-assumption}c). Let $n_t:=|\cR_t|$. We have $n_t\leq m$. We have $|\cR_T|=m$, as we assume full access to the samples of the current task $T$. 

Given the above problem setup, our goal is to learn $f^*$ from the available samples indexed by $\cM$. 

\begin{remark}\label{remark:memory}
    Constructing $\cM$ is to select samples to store and is an interesting CL topic. This can be done via some information-theoretical criterion or optimization \citep{Borsos-NeurIPS2020,Sun-ICLR2022,Elenter-arXiv2023v2}. On the other hand, simple strategies such as random sampling or \textit{reservoir sampling} \citep{Vitter-1985} actually work very well \citep{Chaudhry-arXiv2019v4,Araujo-AACL2022}. We assume $\cM$ is fixed, while our results apply directly to the case where $\cM$ is constructed via random or reservoir sampling.
\end{remark}

\myparagraph{Comparison to Other Modeling Assumptions} 
To highlight the advantages of our modeling assumptions, we first consider the following alternatives:
\begin{itemize}[wide]
    \item A common setting is that the samples within each task are i.i.d., but the samples across tasks be dependent in an arbitrary way; namely,  $x_{ti}$ depends on $x_{\tau j}$ for any $t\neq \tau$ and any $i,j$ (cf.  \cref{fig:3dependency}a). This is a practical and general setup as considered in  \cref{example:ICL-intro}, but it is also a challenging situation for which prior work does not shed light on the benefits of learning from multiple dependent tasks.
    
    \item In another setting, one has $T$ distributions $\{\cD_t\}_{t\in[T]}$. For each $t$, one draws i.i.d. samples $x_{ti}$ from $\cD_t$. Thus, the samples of all tasks are independent, though not identically distributed (cf. \cref{fig:3dependency}b). This setting is intuitive but of less technical interest, as the proof of estimation error bounds is significantly simplified due to the full independence. 
\end{itemize}
In comparison, our proposed setting strikes a balance by assuming that the data across tasks are dependent in a specific way (cf. \cref{fig:3dependency}c). While the proof is still nontrivial under our setting, the intuitive benefit is clear: it reduces the potential dependency across tasks, making it possible to derive better bounds (\cref{fig:3dependency}a versus \cref{fig:3dependency}c).

\begin{figure}
    \centering
    \begin{tikzpicture}
    
        \node[draw, circle,inner sep=0.25mm,align=center] (x11) at (-0.5,0) {$x_{11}$}; 
        \node[draw, circle,inner sep=0.25mm,align=center] (x12) at (0.5,0) {$x_{12}$};
        \node[draw, circle,inner sep=0.25mm,align=center] (x13) at (1.5,0) {$x_{13}$};

        \node[draw, circle,inner sep=0.25mm,align=center] (x21) at (-0.5,-1) {$x_{21}$}; 
        \node[draw, circle,inner sep=0.25mm,align=center] (x22) at (0.5,-1) {$x_{22}$};
        \node[draw, circle,inner sep=0.25mm,align=center] (x23) at (1.5,-1) {$x_{23}$};

        \node[] at (-1.25, -0.5) (nodec) { (c) };
        
        \draw[myblue,thick] (x11) -- (x21)  node [midway, left] {};
        \draw[myblue,thick] (x12) -- (x22)  node [midway, left] {};
        \draw[myblue,thick] (x13) -- (x23)  node [midway, left] {};

        \draw[gray,thick,dashed] (-3.5,0.5) -- (4.5,0.5);

        \node[draw, circle,inner sep=0.25mm,align=center] (x11_) at (-3,2) {$x_{11}$}; 
        \node[draw, circle,inner sep=0.25mm,align=center] (x12_) at (-2,2) {$x_{12}$};
        \node[draw, circle,inner sep=0.25mm,align=center] (x13_) at (-1,2) {$x_{13}$};

        \node[draw, circle,inner sep=0.25mm,align=center] (x21_) at (-3,1) {$x_{21}$}; 
        \node[draw, circle,inner sep=0.25mm,align=center] (x22_) at (-2,1) {$x_{22}$};
        \node[draw, circle,inner sep=0.25mm,align=center] (x23_) at (-1,1) {$x_{23}$};

        \node[] at (-0.25, 1.5) (nodea) { (a) };
        
        \draw[myred,thick] (x11_) -- (x21_)  node [midway, left] {};
        \draw[myred,thick] (x12_) -- (x22_)  node [midway, left] {};
        \draw[myred,thick] (x13_) -- (x23_)  node [midway, left] {};

        \draw[myred,thick] (x11_) -- (x22_)  node [midway, left] {};
        \draw[myred,thick] (x12_) -- (x23_)  node [midway, left] {};
        \draw[myred,thick] (x13_) -- (x21_)  node [midway, left] {};
        \draw[myred,thick] (x11_) -- (x23_)  node [midway, left] {};
        \draw[myred,thick] (x12_) -- (x21_)  node [midway, left] {};
        \draw[myred,thick] (x13_) -- (x22_)  node [midway, left] {};

        \node[draw, circle,inner sep=0.25mm,align=center] (x11_) at (2,2) {$x_{11}$}; 
        \node[draw, circle,inner sep=0.25mm,align=center] (x12_) at (3,2) {$x_{12}$};
        \node[draw, circle,inner sep=0.25mm,align=center] (x13_) at (4,2) {$x_{13}$};

        \node[draw, circle,inner sep=0.25mm,align=center] (x21_) at (2,1) {$x_{21}$}; 
        \node[draw, circle,inner sep=0.25mm,align=center] (x22_) at (3,1) {$x_{22}$};
        \node[draw, circle,inner sep=0.25mm,align=center] (x23_) at (4,1) {$x_{23}$};

        \node[] at (1.25, 1.5) (nodeb) { (b) };
        

    \end{tikzpicture}

    \caption{Illustrating the sample-level dependency assumptions on 3 tasks, each with 3 samples. In (a, b, c), samples are i.i.d. within each task (i.e., no horizontal edges), and samples across tasks are distributed differently; (a): samples across different tasks are dependent; (b): samples across tasks are independent; (c) samples across tasks exhibit one-to-one dependency, as specified by  \cref{eq:dependency}. \label{fig:3dependency} }
\end{figure}
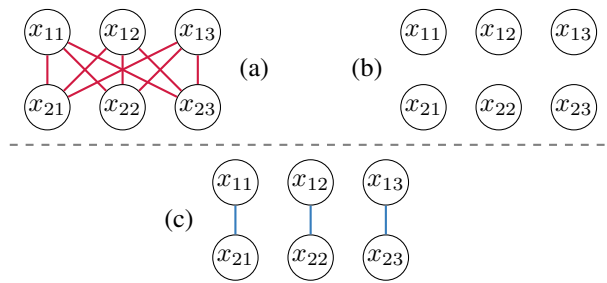

\myparagraph{Comparison to Task Dependency Assumptions}
Learning from the full data (or partial data) as shown in \cref{fig:data-assumption} can be viewed as a multitask learning problem where each task $t$ has $m$ (or $n_t$) samples \citep{Maurer-JMLR2016, Tripuraneni-NeurIPS2020, Tripuraneni-ICML2021, Du-ICLR2021}. The differences from recent multitask learning papers are twofold: Our samples across tasks are \textit{dependent}, as specified in \cref{eq:dependency}; our definition of task dependency is different. In fact, pursuing a suitable definition of task dependency has led to fruitful developments in \textit{multitask learning} and \textit{transfer learning}. \citet{Ben-ML2008} capture the notion of task dependency by an equivalence relation between functions in the hypothesis space. This allows empirical risk minimization to be conducted in a smaller (quotient) space, which translates to a tighter theoretical bound. But this benefit is compromised as optimization over such a space is difficult. A practical task dependency assumption, as considered by \citet{Maurer-JMLR2016, Tripuraneni-NeurIPS2020, Tripuraneni-ICML2021, Du-ICLR2021} and even in earlier work \citep{Caruana-1997,Baxter-JAIR2000,Argyriou-NeurIPS2006}, posits that each task $t$ admits a predictor of the form $f_t \circ \phi$, where $\phi$ is called the shared representation and $f_t$ a task-specific predictor often assumed to be linear. Applications of this setting include deep networks with a shared backbone representing $\phi$ and multiple task-specific heads representing $f_t$. However, such a network needs knowledge of the task corresponding to an input in order to select the corresponding head, which makes it inapplicable to some CL scenarios where no such task identity is available at test time (e.g., \textit{class-incremental} or \textit{domain-incremental} learning) \citep{Van-NMI2022,Ramesh-ICLR2022}.

\subsection{Motivation and Context}\label{subsection:motivation-context}
We now discuss several lines of prior work that  motivate our problem settings and task dependency assumptions.

\myparagraph{CL Context} Our nonlinear transformation $g_t$ in \cref{eq:dependency} covers as special cases the rotations or permutations that arise in CL datasets known respectively as \textit{rotated} MNIST and \textit{permuted} MNIST, where task $1$ is to classify the MNIST images $x_1$, and the images $x_t$ of subsequent tasks are obtained by rotating or permuting the pixels of $x_1$. While these datasets appear artificial, they have been used to benchmark many CL methods, which make themselves practically significant. The task dependency with rotations was theoretically analyzed recently by \citet{Goldfarb-AISTATS2023,Goldfarb-ICLR2024} in the case of two linear regression tasks. In these works, the rotation acts only on the input data and does not change the labels, similarly to the setting of \citet{Ben-ML2008}.

\myparagraph{Motivation from Machine Learning and Control} 
Note that $g_t$ defines a type of autoregressive models which have played major roles in machine learning for computing with sequential data \citep[Chapter 13]{Bishop-2006}. Classic such models include (hidden) Markov models and linear dynamical systems, and they find ample applications in speech recognition, language modeling, and control systems where $g_t$ is often to be learned. Furthermore, recent empirical research shows that deep generative models such as diffusion models can transform one distribution or dataset into another. In diffusion models the transformation between consecutive time steps is of the type \textit{identity-plus-noise}. Different from our deterministic $g_t$, this is a random transformation due to the use of random noise.

Our data generation protocol \cref{eq:regress} and \cref{eq:dependency} also find motivations in recent research on learning from \textit{many trajectories} or \textit{dependent data} \citep{Tu-JMLR2024,Ziemann-CDC2023,Ziemann-NeurIPS2022,Tadipatri-arXiv2025}. A major difference is that these works assume the availability of full data (\cref{fig:data-assumption}b). Instead, we are motivated by CL scenarios (memory-based methods), and consider a more general situation of learning from many \textit{incomplete} or \textit{partial} trajectories (\cref{fig:data-assumption}c). At a technical level, we build upon the proof ideas of \citet{Ziemann-CDC2023,Tadipatri-arXiv2025} but also make extensions for cases such as partial trajectories, weighted objectives, data-dependent regularization, and data-dependent weights (cf. \cref{section:theory}).




\myparagraph{Geometric Vision Context} In \textit{geometric vision} \citep{Hartley-2004}, $g_t$ also arises. There, a classical scenario is \textit{point cloud registration} \citep{Arun-TPAMI1987}: Given two point clouds $\{ x_{1i} \}_{i\in[m]}$ and $\{ x_{2i} \}_{i\in[m]}$ of the same scene captured by a moving camera, the goal is to find some Euclidean transformation $g_2$ that aligns them; ideally we have $x_{2i}=g_2(x_{1i})$ for all $i\in[m]$, while it is possible that not all points of the scene are captured and some points are missing (cf. \cref{fig:data-assumption}c). Here, $g_2$ models how the camera moves and is independent of the point clouds. Finally, a sequence of multiple point clouds arise naturally, as the camera moves and continually measures the scene.

\section{THEORETICAL CONTRIBUTIONS}\label{section:theory}
We first describe the notations and technical assumptions (\cref{subsection:assumption}). Then we introduce our theory for \textit{weighted experience replay with regularization} (\cref{subsection:guarantee1}),  \textit{data-dependent regularization} (\cref{subsection:guarantee2-regularization}), and \textit{data-dependent weights} (\cref{subsection:guarantee3-weights}).

\subsection{Notations and Technical Assumptions}\label{subsection:assumption}
\myparagraph{Notations} Let $\bbB_{s}(d)$ be the ball centered at $0$ in $\bbR^d$ of radius $s>0$, that is $\bbB_{s}(d):=\{ z\in \bbR^d: \| z \|_2 \leq s  \}$. Denote by $\poly(\cdot)$ some polynomial function of its input. The notation $\tilde{\cO}(\cdot)$ suppresses additive lower-order terms and multiplicative logarithmic terms of the form $\ln(\cdot)$. In any inequality of the form $a\lesssim b$, the symbol  $\lesssim$ means that the inequality holds up to a constant; define $\gtrsim$ similarly. We work with the class of functions $\cF:=\{f_{\theta}: \theta\in \Theta \}$ parameterized by elements of $\Theta$.  With $q\geq 1$ we assume $\cF$ is in the $L^q$ space with respect to the joint distribution of $x_1,\dots,x_T$. Let $\| \cdot \|_{\cF}$ be the $q$-norm of a function; namely, for $f\in \cF$ we define $\| f \|_{\cF}:= \frac{1}{T}\sum_{t\in[T]}\bbE_{x_t}[\| f(x_t) \|_q^q]^{1/q}$ (where $x_t$ is independent of $f$). Since the precise values of $q$ are not very relevant to our development, we hide the dependency on $q$ in the notation $\| \cdot \|_{\cF}$.

\begin{table*}[t]
\centering
\caption{Summary of assumptions that we describe in \cref{subsection:assumption} and use throughout the paper. }
\label{table:assumptions_summary}
\begin{tabular}{lc}
\toprule
\textbf{Assumption} & \textbf{Description} \\
\midrule
\cref{assumption:data} & Input and noise are sub-Gaussian with independent coordinates \\
\cref{assumption:Theta} & The parameter space is bounded and true predictor is realizable \\
\cref{assumption:finite-moment} & The direct difference map \cref{eq:def-G_ft} has bounded and nondegenerate moments  \\
\cref{assumption:G-Lipschitz} & The direct difference map and parametrized predictors satisfy Lipschitz-type conditions \\
\bottomrule
\end{tabular}
\end{table*}

\myparagraph{Data} We now describe our assumptions (see \cref{table:assumptions_summary} for a summary). First, we assume our input $x_1$ and noise $v_t$ are \textit{sub-Gaussian} (see the appendix):
\begin{assumption}\label{assumption:data}
    $x_1$ is a $d_x$-dimensional \textit{sub-Gaussian} vector with independent coordinates and \textit{proxy variance} $\sigma^2/d_x$.
    Furthermore, noise $v_t$, when conditioned on $x_1,\dots,x_t$, is a $d_y$-dimensional sub-Gaussian vector with independent coordinates and proxy variance $\nu^2$.
\end{assumption}
Bounded random variables and Gaussian random variables are all sub-Gaussian. It is harmless, though less general, to think of $x_1$ (\textit{resp.} $v_t$) as a vector with i.i.d. Gaussian entries, each with mean $0$ and variance $\sigma^2/d_x$ (\textit{resp.} $\nu^2$).

\myparagraph{Parameter and Function Space} We assume $f^*$ is realizable in a bounded parameter space $\Theta$:
\begin{assumption}\label{assumption:Theta}
    Parameter space $\Theta\subset \bbR^p$ is bounded with $\diam(\Theta):=\sup\{ \|\theta - \theta'\|_{2}: \theta,\theta'\in \Theta \}=\tilde{\cO}( \poly(p))$, and the true predictor $f^*$ is realizable in $\Theta$, that is $f^* \in \cF$. 
\end{assumption}

\myparagraph{Direct Difference Map} In \cref{eq:regress},  $y_t$ is a function of input $x_t$. In \cref{eq:dependency}, we have $x_t$ as a function of $x_1$. We now describe the function that maps $x_1$ \textit{directly} to $y_t$. With the identity mapping $g_1$, write $g_t \circ g_{t-1} \circ \dots \circ g_1$ for the function that maps $x_1$ to $x_t$; this is for convenience and should not cause confusion, even though the output dimension of $g_{i-1}$ does not match the input dimension of $g_i$. Then, the composition $f^* \circ g_t \circ g_{t-1} \circ \dots \circ g_1$ is the direct input-output map in our model. Since we are interested in finding some function $f$ that is close to $f^*$, we consider the following function, which we call \textit{direct difference map}:
\begin{align}\label{eq:def-G_ft}
    G_{f,t} := (f - f^*) \circ g_t \circ g_{t-1} \circ \cdots \circ g_1.
\end{align}
It maps $x_1$ to $x_t$ and then takes the difference $f(x_t) - f^* (x_t)$. Note that if all $g_i$'s are such that $x_1=\cdots=x_t$, then $G_{f,t}$ is simply the difference $f - f^*$. More generally, $G_{f,t}$ takes task relationship into account as it is defined with respect to $g_i$'s. This $G_{f,t}$ is crucial to our analysis, and we assume it has the following  properties:
\begin{assumption}\label{assumption:finite-moment}
     We have $\bbE\left[ \|G_{f,t}(x_1) \|_2^4 \right]<\infty$ and $\bbE\left[\nmm{G_{f,t}(x_1)}_2^2\right]>0$ ($\forall f\in \cF\ \backslash \{f^*\}, t\in [T]$).  
\end{assumption}
\begin{assumption}\label{assumption:G-Lipschitz}
    For constants $L_{\cF}>0$, $L_G>0$ we have
    \begin{align}
        &\| f_{\theta} - f_{\theta'} \|_{\cF} \leq L_{\cF} \cdot \| \theta - \theta' \|_{2}, \quad \forall \theta,\theta'\in \Theta; \label{eq:f-Lipschiz} \\
        &\|G_{f,t}(z) - G_{f,t}(z') \|_2 \leq L_G \cdot \| z - z'\|_2,\ \ \forall z,z'
        \in \bbR^{d_x}; \nonumber
    \end{align}
    There is a \textit{sufficiently large} number $r_x$ such that for every $f,f'\in \cF$ the following hold with  $K_G=\tilde{\cO}(\poly(r_x))$:
    \begin{align}
        &\sup_{z \in \bbB_{r_x}(d_x)} \nmm{G_{f,t}(z) - G_{f',t}(z)}_2 \leq K_G \cdot \nmm{f - f'}_{\cF}; \nonumber \\ 
        &\sup_{z \in \bbB_{r_x}(d_x)} \left| \|G_{f,t}(z) \|_2^2 - \| G_{f',t}(z) \|_2^2 \right| \leq K_G \nmm{f - f'}_{\cF}. \nonumber 
    \end{align} 
\end{assumption}
\cref{assumption:data,assumption:Theta,assumption:finite-moment} are standard. \cref{assumption:G-Lipschitz} imposes Lipschitz-type conditions on our parameterized function $f_{\theta}$ and direct difference map $G_{f,t}$ and it requires $r_x$ to be \textit{sufficiently large}; we will show the precise values of $r_x$ in the full version of our theorems in the appendix. Roughly speaking, \cref{eq:f-Lipschiz} holds for deep networks with bounded parameters and bounded inputs, while the rest inequalities hold true as soon as all $f$ and $g_t$'s are Lipschitz continuous. Note that rotations and permutations aforementioned are Lipschitz continuous.

\subsection{Recovery Guarantee 1: Weighted Replay With Data-Independent Regularization}\label{subsection:guarantee1}
In this and next two sections, we introduce our main theorems (see \cref{table:theorem_summary} for a summary).

Inspired by memory-based and regularization-based methods, 
we formulate the following  problem:
\begin{align}\label{eq:general-cl}
    \min_{\theta\in \Theta} \frac{1}{T}   \sum_{(t,i)\in \cM} \frac{w_t}{n_t } \cdot \cL \Big( y_{ti},  f_{\theta}(x_{ti}) \Big) + \lambda \cdot \Omega_T(\theta).
\end{align}
In \cref{eq:general-cl}, we use the squared loss $\cL(y, \hat{y}) = \|\hat{y} - y \|_2^2$, as it is empirically useful in CL practice \citep{Mcdonnell-NeurIPS2023,Peng-ICLR2025} and is a common objective of interest in CL theory. 
The factor $1/n_t$ normalizes the loss of each task. Each $w_t$ is some non-negative weight for task $t$, and we assume $w_T>0$, as we have full data for the current task $T$. Moreover, we assume that $w_t$'s are hyperparameters, chosen by a user and independent of data. The function $\Omega_T(\cdot)$ serves the purpose of regularization, and it is weighted by some non-negative number $\lambda$. Thus, \cref{eq:general-cl} amounts to minimizing a multitask loss with regularization $\Omega_T(\cdot)$. The generality of \cref{eq:general-cl} stems from its flexibility to choose weights $w_t$ and regularization $\Omega_T(\cdot)$. For example, \textit{experience replay} or \textit{rehearsal} amounts to solving \cref{eq:general-cl} with $\lambda=0$, while regularization-based methods often set $w_T=1$ and set all previous weights to $0$. 
Here, we assume $\Omega_T(\theta)$ is independent of data and noise, as is often the case in machine learning \citep[Chapter 7]{Goodfellow-book2016}, and as considered in some CL algorithms \citep{Kumar-arXIv2023v3,Lewandowski-ICLR2025}. For our results on data-dependent regularizers, see \cref{subsection:guarantee2-regularization}.


\begin{remark} To unify memory-based and regularization-based methods, \citet{Wang-ICLR2024} considers a general formulation similar to \cref{eq:general-cl}.  \citet{Wang-ICLR2024} has a specific algorithmic focus and its formulation is not leveraged in full generality for theoretical developments, while we develop statistical recovery results for \cref{eq:general-cl} and its variants.
\end{remark}

We are ready to state our main result of this section:
\begin{theorem}\label{theorem:general}
    Fix $\delta\in(0,1)$. Suppose \cref{assumption:data,assumption:Theta,assumption:G-Lipschitz,assumption:finite-moment} hold. Recall that noise $v_t$ is conditionally sub-Gaussian with proxy variance $\nu^2$. Let $\hat{\theta}_T\in \Theta\subset \bbR^{p}$ be a global minimizer of \cref{eq:general-cl} with  regularization parameter $\lambda$ satisfying $\lambda\lesssim   \frac{  1  }{ T  }  \max_{t\in [T]} \frac{ w_t }{n_t}$. Define
    \begin{align}\label{eq:def:kappa}
        \kappa := \sup_{f \in \cF\ \backslash \{f^*\}}\sup_{t \in [T],w_t>0}\frac{\bbE\left[\nmm{G_{f,t}(x_1)}_2^4\right]^{1/2}}{\bbE\left[\nmm{G_{f,t}(x_1)}_2^2\right]}.
    \end{align}
    Assume $n_t \geq \kappa^2 \cdot \tilde{\cO}\left(  p \ln (T) +  \ln (1/\delta) \right)$ for all $t\in[T]$. With probability at least $1-\delta$ the weighted estimation error $\frac{1}{T}\sum_{t \in [T]}w_t \cdot \bbE \nmm{f^*(x_t) - f_{\hat{\theta}_T}(x_t)}_2^2$  is bounded above by 
    \begin{align}\label{eq:stat_recovery}
        \tilde{\cO} \left(   \frac{  \nu^2  ( p   + \ln(1/\delta)) + \poly(\sigma) }{ T } \cdot \max_{t\in [T] } \frac{ w_t }{n_t}  \right) .
    \end{align}
\end{theorem}
\cref{theorem:general} is a finite sample guarantee with bound \cref{eq:stat_recovery} depending on variances of data ($\poly(\sigma)$) and noise ($\nu^2$), the number of tasks $T$, the number of of available samples $n_t$  and weight $w_t$ of task $t$ (note that the bounds hides dependency on the dimensions $d_x,d_y$ among other terms). We discuss these quantities next. 

\myparagraph{Proxy Variance of Data} Since both $x_{it}$'s and $v_{it}$'s are random, the estimation error depends on their proxy variances $\sigma^2$ and $\nu^2$. First, the dependency on $\sigma$ is $\poly(\sigma)$, which is because \cref{assumption:G-Lipschitz} introduces a polynomial dependency of $r_x$ whose precise values depend on $\sigma$. 

\myparagraph{Weights and Sample Complexity} While prior work sets uniform weights $w_1=\cdots=w_T>0$ to derive statistical bounds in their CL settings \citep{Lin-ICML2023,Friedman-arXiv2024v2}, we prove our bound in \cref{eq:stat_recovery} with arbitrary non-negative weights $w_t$ (independent of data). This allows us to set different weights and acquire different theoretical insights. For example, if we set all weights but $w_t$ to $0$, then \cref{eq:stat_recovery} becomes a single-task bound
\begin{equation*}
    \bbE\nmm{f^*(x_t) - f_{\hat{\theta}_T}(x_t)}_2^2 \leq \begin{cases}
    \tilde{\cO} \big( \frac{  1}{ m }  \big) & T=t, \\ 
    \tilde{\cO} \big( \frac{ 1 }{ n_t }  \big)  & T>t.
    \end{cases}
\end{equation*}
The case $T=t$ holds as we have $m$ samples at task $t$. We now see that when transiting from task $t$ to task $t+1$, forgetting occurs as the bound transits from $\tilde{\cO}(1/m)$ to $\tilde{\cO}(1/n_t)$. However, starting from task $t+1$, the bound remains the same, thus no significant forgetting is furthermore entailed.

If we set $w_1=\cdots=w_T>0$, then \cref{theorem:general} upper bounds the average error $\frac{1}{T}\sum_{t \in [T]} \bbE \nmm{f^*(x_t) - f_{\hat{\theta}_T}(x_t)}_2^2$ by $\tilde{\cO} \left(   \frac{  1 }{T  \cdot \min_{t\in[T]}  n_t  } \right)$. For this bound to be informative, we need a \textit{balanced} memory that stores approximately the same number of samples for each task, which is indeed the case when we construct the memory by random sampling or reservoir sampling (\cref{remark:memory}). Furthermore, this is in agreement with common CL practice, where balanced memory tends to perform better than the imbalanced ones \citep{Chaudhry-arXiv2019v4,Prabhu-ECCV2020,Araujo-AACL2022}. Then, given balanced memory and as $T$ increases, our result gets better, partly because the upper bound decreases, and partly because it provides bounds on the average error of more tasks. Finally, we note that, with non-uniform weights $w_t = \frac{T n_t}{n_1 + \cdots + n_T}$ that sum to $1$, \cref{theorem:general} gives the bound $\tilde{\cO} \left(   \frac{   1}{ n_1 + \cdots + n_T }  \right)$, which is with the optimal sample complexity $\tilde{\cO}(1/(n_1+\cdots+n_T))$.

Are the above bounds too good to be true? Indeed, we have at most $m$ samples starting  (the samples of all tasks are deterministic transformations of the $m$ samples from task $1$); at first glance one might conclude that the sample complexity lower bound would be $\tilde{\cO}(1/m)$, but \cref{theorem:general} gives a tighter result. This is not a contradiction, the reason being that \cref{theorem:general} assumes sufficiently many samples with $m\geq n_t \geq \kappa^2 \cdot \tilde{\cO}\left(  p \ln (T) +  \ln (1/\delta) \right)$, all tasks share a common predictor $f^*$, and we consider the regression formulation. For example, if $f^*$ is linear, then with $m\geq p$ linearly independent (random) samples the regression problem has a unique global minimizer $f^*$. In this case, we can recover the true $f^*$ exactly, which gives us the zero bound instead of $\tilde{\cO}(1/m)$.

\myparagraph{Condition on Sufficiently Many Samples} \cref{theorem:general} assumes $n_t \geq \kappa^2 \cdot \tilde{\cO}\left(  p \ln(T) +  \ln (1/\delta) \right)$, that is, every task stores sufficiently many samples. To understand what this assumption imposes, it suffices to analyze the role of $\kappa$. While computing the numerical values of $\kappa$ is difficult, from \cref{assumption:finite-moment} we know $\kappa$ is finite. More can be said if we make extra assumptions on how tasks are dependent (note that $\kappa$ depends on $G_{f,t}$ and $G_{f,t}$ encodes the information about task dependency):
\begin{example}\label{example:kappa}
    Assume the function class $\cF$ consists of linear functions and all entries of $x_1$ are i.i.d. Gaussian,  sampled from $\cN( 0, 1 )$. Consider two tasks with $g_2: x_1 \mapsto 10^{10} x_1$. With two functions $f,f^*\in \cF$ fixed and parameterized by $\theta,\theta^*$, respectively, we have $G_{f,t} (x_1)= (\theta- \theta^*)^\top x_1$. Since $G_{f,t} (x_1)$ is also a zero-mean Gaussian, a standard calculation of its fourth and second moments gives $\kappa^2=3$. Note that the scaling factor $10^{10}$ has no effect in $\kappa$. 
\end{example}
We have $\kappa^2\leq 3$ if Gaussianity of \cref{example:kappa} is replaced with \textit{strict sub-Gaussianity}. Finally, we note that $\kappa$ is upper bounded by some universal constant for sub-Gaussian data.

\myparagraph{Robustness to Distant Distributions} \cref{example:kappa} generalizes to multiple tasks with $g_t$'s being of the form $x_1\mapsto s_t x_t$ for some $s_t$. Such a simple example of $g_t$'s can in fact make a difficult case for analysis. Indeed, it affects the Lipschitz parameter $L_G$ and radius $r_x$ in \cref{assumption:G-Lipschitz}. That said, all of our bounds exhibit only a logarithmic dependency on $L_G$, $r_x$, so our theorems allow $s_t$ to be a polynomial function of problem parameters. The second difficulty \cref{example:kappa} brings manifests itself, not in our theorem, but in prior work. Indeed, for large $s_t$, the data distributions of  two tasks, $\cN( 0, I_{d_x} )$ and $\cN( 0, s_t^2 I_{d_x})$, could be very \textit{distant}; here $I_{d_x}$ is the $d_x\times d_x$ identity matrix. This distance has fundamental impacts on many statistical bounds that depend \textit{linearly} on this distance:
\begin{remark}\label{remark:discrepancy}
    The theory of \citet{Mansour-arXiv2009} on \textit{domain adaptation} bounds statistical errors of the second task (i.e., \textit{target task}) by the \textit{discrepancy distance} between distributions $\pi_1,\pi_2$ of the first task (\textit{source task}) and the second task. The discrepancy  $\dist(\pi_1,\pi_2)$, defined as $$\sup_{\theta_1,\theta_2\in \Theta } \left| \bbE_{\pi_1} \cL\left( f_{\theta_1}(z), f_{\theta_2}(z) \right) - \bbE_{\pi_2} \cL\left( f_{\theta_1}(z), f_{\theta_2}(z) \right) \right|,$$
    measures the similarities between  $\pi_1,\pi_2$. If $\dist(\pi_1,\pi_2)$ is small, we might expect some benefits of learning task 1 prior to task 2 (e.g., small errors on task 2). This basic intuition is extended to various settings, including online multitask learning, transfer learning, and continual learning \citep{Mohri-ICALT2012,Wang-JMLR2023,Ye-NeurIPS2022}. That said, the discrepancy can be large in \cref{example:kappa} with the squared loss $\cL(\cdot,\cdot)$. Indeed, set $\pi_1=\cN(0, I_{d_x})$ and $\pi_2=\cN( 0, s^2 I_{d_x})$ with $s\gg 1$, andone verifies $\dist(\pi_1,\pi_2)=(s^2-1)\cdot \diam(\Theta)^2$, that is, the distance grows unbounded as $s$ increases; their corresponding bound is vacuous in this scenario. In consequence, the proof based on such a discrepancy is unable to handle simple yet distant tasks. For a similar reason, the recent PAC-Bayes bound of  \citet{Friedman-arXiv2024v2} easily becomes vacuous due to the presence of a similar discrepancy term.
\end{remark}
Our final note on \cref{theorem:general} is a comparison to the experience replay theory of \citet{Peng-ICML2023}.
\begin{remark}\label{remark:compare-ICL}
    Theorem 3 of \citet{Peng-ICML2023} provides the bound $\tilde{\cO}( \frac{1}{n_t})$ on the excess risk of \cref{eq:general-cl} with $\lambda=0$ for each task $t$. This directly implies the bound $\tilde{\cO}(\frac{\log T}{T}\sum_{t\in[T]} \frac{1}{n_t})$ for the average joint loss. This bound becomes $\tilde{\cO}(\log T/n_t)$ when all $n_t$'s are of the same order. In contrast, we directly work with the average estimation error, and our bound has an $\tilde{\cO}(1/T)$ dependency on $T$.       
\end{remark}

\subsection{Recovery Guarantee 2: Data-Dependent Regularization and Knowledge Distillation}\label{subsection:guarantee2-regularization}
\cref{theorem:general} in \cref{subsection:guarantee1} assumes  the regularizer is independent of data. On the other hand, it is not uncommon to utilize data-dependent regularizers for CL. A typical data-dependent regularizer used in CL is \textit{knowledge distillation}, that is to match the (intermediate) outputs of the current network $f_{\theta}$ and the previously learned network $f_{\hat{\theta}_t}$ at stored samples. This motivates the following CL method: 

\begin{itemize}[wide]
    \item (\textit{Step 1}) Set hyperparameter $\beta_T>0$. Then solve 
     \begin{align}
        \hat{\theta}_T\in \argmin_{\theta\in \Theta} \beta_T \sum_{i\in [m]} \cL \Big( y_{Ti},  f_{\theta}(x_{Ti}) \Big) + \Omega_T(\theta), \label{eq:obj-distill} 
    \end{align}
    where $\Omega_T(\cdot)$ is either of the following two regularizers:
    \begin{align}
        \Omega_T(\theta) &= \sum_{t\in[T-1]} \sum_{i\in \cR_t}    \beta_t\cdot \| f_{\theta}(x_{ti}) - f_{\hat{\theta}_{T-1}}(x_{ti}) \|_2^2 \label{eq:obj-distill1} \\ 
        \Omega_T(\theta) &= \sum_{t\in[T-1]} \sum_{i\in \cR_t} \beta_t\cdot \| f_{\theta}(x_{ti}) - f_{\hat{\theta}_{t}}(x_{ti}) \|_2^2  \label{eq:obj-distill2}
    \end{align}
    \item (\textit{Step 2}) Choose indices $\cR_T\subset [m]$. Store $\{(x_{Ti}, y_{Ti}) \}_{i\in\cR_T}$. Increase $T$. Go back to \textit{Step 1}.
\end{itemize}
The above method are with regularizer \cref{eq:obj-distill1} or \cref{eq:obj-distill2} to compute $\hat{\theta}_1,\dots,\hat{\theta}_T$ continually. Regularizer \cref{eq:obj-distill1} aims to match $f_{\theta}$ and $f_{\hat{\theta}_{T-1}}$ on all previous tasks, while regularizer \cref{eq:obj-distill2} aims to match $f_{\theta}$ and $f_{\hat{\theta}_{t}}$ on each task $t$ ($\forall t\in[T-1]$). In both cases, it is understood that task $1$ is solved without regularization and that the algorithm consistently chooses either \cref{eq:obj-distill1} or \cref{eq:obj-distill2}. Computationally, for \cref{eq:obj-distill2} we can store $f_{\hat{\theta}_{t}}(x_{ti})$ after training task $t$, while for \cref{eq:obj-distill1} we need to compute $f_{\hat{\theta}_{T-1}}(x_{ti})$ after training task $T-1$. For both regularizers, we have the following theorem.
\begin{theorem}\label{theorem:regularization}
    Fix $\delta\in(0,1)$. Let $\kappa$ be as in \cref{eq:def:kappa}. Suppose \cref{assumption:data,assumption:Theta,assumption:G-Lipschitz,assumption:finite-moment} hold. Recall that noise $v_t$ is conditionally sub-Gaussian with proxy variance $\nu^2$. Assume $n_t \geq \kappa^2 \cdot \tilde{\cO}\left(  p \ln(T) +  \ln (1/\delta) \right)$ for all $t\in[T]$. If $\hat{\theta}_1, \dots, \hat{\theta}_T \in \Theta \subset \bbR^p$ are  global minimizers of \cref{eq:obj-distill1} with $\beta_t=1/4^{T-t}$. With probability at least $1-\delta$ the weighted error $\sum_{t \in [T]} \frac{ \beta_t n_t }{\sum_{t \in [T]} \beta_t n_t }  \cdot \bbE \nmm{f^*(x_t) - f_{\hat{\theta}_T}(x_t)}_2^2$ is bounded above by
    \begin{align}
        \tilde{\cO} \left(   \frac{ \nu^2 (pT + \ln(1/\delta))  + \poly(\sigma) }{  \sum_{t \in [T]}\beta_t n_t }  \right). \label{eq:distill1-bound2}
    \end{align}
    If $ \hat{\theta}_1, \dots, \hat{\theta}_T$ are solutions to \cref{eq:obj-distill2}, then with probability at least $1-\delta$ we have \cref{eq:distill1-bound2} holds as well.
\end{theorem}
In \cref{theorem:regularization}, the error term that we bound has its weights on distances $\nmm{f^*(x_t) - f_{\hat{\theta}_T}(x_t)}_2^2$ decay exponentially as $t$ decreases, meaning that \cref{eq:distill1-bound2} carries exponentially less control over the distances on past tasks. Therefore, \cref{theorem:regularization} is informative for a small number of the most recent tasks, and to some extent, it illuminates the limits of regularization-based methods in combating forgetting. In fact, this bound is worse than that of \cref{theorem:general} for experience replay. This aligns with empirical observations that experience replay typically outperforms regularization-based approaches; for example, see Table 4 of \citep{Prabhu-ECCV2020}, which compares \textit{LwF} \citep{Li-TPAMI2017} and \textit{GDumb} \citep{Prabhu-ECCV2020}. 


On a technical note, bound \cref{eq:distill1-bound2} itself could be sub-optimal for two reasons. First, since the algorithm analyzed involves all predictors $f_{\hat{\theta}_1},\dots,f_{\hat{\theta}_T}$,  all of which depend on data, we have to run an $\varepsilon$-net argument on the product space $\Theta^T$ in $\bbR^{pT}$, which brings the dependency $pT$. Second, the exponential forgetting phenomenon is due to our choice $\beta_t=1/4^{T-t}$. This choice is crucial, as it prevents the bound from getting exponentially large. 
Improving \cref{theorem:regularization} in these aspects is left for future work. 

We finish the section by remarking on existing theoretical contributions on regularization-based methods:
\begin{remark} \label{remark:regularization-compare} 
    The analysis of \citet{Heckel-AISTATS2022,Li-CoLLAs2023,Zhao-ICML2024,Levinstein-arXiv2025} relies on linear models or the kernel regime. The analysis of \citet{Zhu-arXiv2025} relies on either the assumption of shared global minimizers or a noisy linear regression model. The setting of \citet{Yin-arXiv2020v3} is general, but the error bound of their Theorem 4 contains the distance of the form $\| \hat{\theta}_T -\hat{\theta}_t\|_2$. This is a random variable whose dependency on problem parameters is unclear. In contrast, our bound in \cref{eq:distill1-bound2} contains only parameters related to the problem configuration. We achieve this via several key inequalities that give rise to a recurrence relation, which we unroll to eliminate all random variables related to $\hat{\theta}_t$.
\end{remark}

\subsection{Recovery Guarantee 3: Data-Dependent Weights and Constrained Optimization}\label{subsection:guarantee3-weights}
Our formulation \cref{eq:general-cl} in \cref{subsection:guarantee1} assumes the weights are independent of data. Here, we consider
\begin{align}\label{eq:general-cl-dependent-weight}
    \hat{\theta}_T\in \argmin_{\theta\in \Theta} \frac{1}{T}   \sum_{(t,i)\in \cM} \frac{\tilde{w}_t}{n_t } \cdot \cL \Big( y_{ti},  f_{\theta}(x_{ti}) \Big), 
\end{align}
where weight $\tilde{w}_t$ is a random variable depending on all data and noise. While we allowed zero weights in \cref{subsection:guarantee1} to study regularization-based methods without replay, here we assume $\tilde{w}_t\neq 0$ ($\forall t\in [T]$), or otherwise formulation \cref{eq:general-cl-dependent-weight} reduces to a case with fewer tasks. Then, dividing the smallest weight if necessary, we assume $\tilde{w}_t\in [1, W]$. 

There are multiple ways to choose $\hat{w}_t$ in a data-dependent fashion. For example, set $\hat{w}_t$ to be proportional (or inversely proportional) to the loss of task $t$ (in light of \textit{boosting} \citep{Schapire-ML1990,Ramesh-ICLR2022,Wang-JMLR2023} or \textit{iteratively reweighted least-squares} \citep{Daubechies-CPAM2010,Kummerle-NeurIPS2021,Peng-NeurIPS2022,Peng-CVPR2023}), or solve a bilevel program in weight variables that represent \textit{coresets} \citep{Borsos-NeurIPS2020}, or set the weights to the values of dual variables which arise in primal-dual methods for constrained learning \citep{Lopez-NeurIPS2017,Chamon-TIT2022,Peng-ICML2023,Elenter-arXiv2023v2,Li-arXiv2024v3model}.


We now give recovery guarantees for \cref{eq:general-cl-dependent-weight}: 
\begin{theorem}\label{theorem:dependent-weights}
    Fix $\delta\in(0,1)$. Let $\kappa$ be as in \cref{eq:def:kappa}. Suppose \cref{assumption:data,assumption:Theta,assumption:G-Lipschitz,assumption:finite-moment} hold. Recall that noise $v_t$ is conditionally sub-Gaussian with proxy variance $\nu^2$ and that $\tilde{w}_t\in[1,W]$. Let $\hat{\theta}_T\in \Theta\subset \bbR^{p}$ be a global minimizer of \cref{eq:general-cl-dependent-weight}. Assume $n_t \geq \kappa^2 \cdot \tilde{\cO}\left(  p \ln(T) +  \ln (1/\delta) \right)$ for all $t\in[T]$ and $W\leq 1+\frac{1}{Tn_t}$. With probability at least $1-\delta$, the average estimation error $\frac{1}{T} \sum_{t \in [T]} \bbE \nmm{f^*(x_t) - f_{\hat{\theta}_T}(x_t)}_2^2 $ is bounded above by 
    \begin{align}\label{eq:stat_recovery-dependent-weights}
        \tilde{\cO} \left( \frac{  \nu^2  ( p   + \ln(1/\delta)) + \max\{ \poly(\sigma), \ln(\sigma T) \} }{T\min_{t\in [T]} n_t }  \right). 
    \end{align}
\end{theorem}
The bound of \cref{theorem:dependent-weights} follows from that of \cref{theorem:general}, and in fact they are identical except assumption $W\leq 1+\frac{1}{Tn_t}$. This extra condition is what we pay for the case of data-dependent weights. Due to a number of differences in the settings, \cref{theorem:dependent-weights} is not directly comparable to  theoretical results of  \citet{Li-arXiv2024v3model,Peng-ICML2023,Chamon-TIT2022}. However, a distinguishing aspect is that, as $T$ gets large, our bound would in general get better, while their bounds deteriorate. This is because we explicitly model the dependency between tasks and take advantage of it in the proofs, while their proofs do not consider task dependency and need instead to apply $T$ concentration inequalities, one for each task (\cref{example:ICL-intro}).

\section{CONCLUSION}
Inspired by nonlinear dynamical systems and several other topics, we formalized a notion of task dependency for continual learning of nonlinear regression tasks. Building upon it, we developed some recovery guarantees for commonly used CL methods that involve replay, regularization (e.g., knowledge distillation), and data-dependent weighting. The estimation error bounds we derived are well-behaved, as they diminish as the number of samples tends to infinity or for sufficiently small noise. The key to proving our theorems is a careful balance we maintain between the generality of the problem formulation and the tightness of the resulting bounds; put differently, this can also be viewed as a limitation: our proof framework does not support deriving stronger guarantees for more general scenarios where different tasks are dependent arbitrarily (\cref{fig:3dependency}). The other limitation is this: Our results are based on the assumption that the number of stored samples per task exceeds the dimension of the parameter space. It will be of interest to relax this assumption and furthermore extend our results to the overparameterized regime.

In the future, we plan to improve the bounds for regularization-based methods (\cref{theorem:regularization}) by exploring more specific settings such as linear models (cf. \cref{remark:regularization-compare}), to extend our \cref{theorem:dependent-weights} with data-dependent weights within the constrained learning framework of \citet{Chamon-TIT2022} that guarantees the feasibility of constraints, and to extend the insights of the paper here for other CL paradigms (e.g., expansion-based methods). Also, note that our analysis is independent of the algorithmic choice, and in the future we plan to analyze the effect of the algorithm on the recovery errors as well.

\section*{Acknowledgement}
This work is supported by the National Science Foundation (grants 2031985), the Simons Foundation (grant 814201), and the Office of Naval Research (grant 503405-78051).

\newpage
{
\small 

\bibliographystyle{apalike}
\bibliography{Liangzu, Uday, Ziqing}
}






\section*{Checklist}
\begin{enumerate}

  \item For all models and algorithms presented, check if you include:
  \begin{enumerate}
    \item A clear description of the mathematical setting, assumptions, algorithm, and/or model. [Yes]
    \item An analysis of the properties and complexity (time, space, sample size) of any algorithm. [Not Applicable]
    \item (Optional) Anonymized source code, with specification of all dependencies, including external libraries. [Not Applicable]
  \end{enumerate}

  \item For any theoretical claim, check if you include:
  \begin{enumerate}
    \item Statements of the full set of assumptions of all theoretical results. [Yes]
    \item Complete proofs of all theoretical results. [Yes]
    \item Clear explanations of any assumptions. [Yes]     
  \end{enumerate}

  \item For all figures and tables that present empirical results, check if you include:
  \begin{enumerate}
    \item The code, data, and instructions needed to reproduce the main experimental results (either in the supplemental material or as a URL). [Not Applicable]
    \item All the training details (e.g., data splits, hyperparameters, how they were chosen). [Not Applicable]
    \item A clear definition of the specific measure or statistics and error bars (e.g., with respect to the random seed after running experiments multiple times). [Not Applicable]
    \item A description of the computing infrastructure used. (e.g., type of GPUs, internal cluster, or cloud provider). [Not Applicable]
  \end{enumerate}

  \item If you are using existing assets (e.g., code, data, models) or curating/releasing new assets, check if you include:
  \begin{enumerate}
    \item Citations of the creator If your work uses existing assets. [Not Applicable]
    \item The license information of the assets, if applicable. [Not Applicable]
    \item New assets either in the supplemental material or as a URL, if applicable. [Not Applicable]
    \item Information about consent from data providers/curators. [Not Applicable]
    \item Discussion of sensible content if applicable, e.g., personally identifiable information or offensive content. [Not Applicable]
  \end{enumerate}

  \item If you used crowdsourcing or conducted research with human subjects, check if you include:
  \begin{enumerate}
    \item The full text of instructions given to participants and screenshots. [Not Applicable]
    \item Descriptions of potential participant risks, with links to Institutional Review Board (IRB) approvals if applicable. [Not Applicable]
    \item The estimated hourly wage paid to participants and the total amount spent on participant compensation. [Not Applicable]
  \end{enumerate}

\end{enumerate}

\clearpage
\appendix
\thispagestyle{empty}

\onecolumn

\aistatstitle{Recovery Guarantees for Continual Learning of Dependent Tasks: Memory, Data-Dependent Regularization, and Data-Dependent Weights:\\ Supplementary Materials}

In this material, we will present proofs of results, examples, related work, 
extra discussion on extensions, and preliminaries for the main draft. 
The below is the table of contents of this material:
\appendixtableofcontents

\begin{table*}[t]
    \centering
    \caption{Summary of theorems that we describe in \cref{subsection:guarantee1,subsection:guarantee2-regularization,subsection:guarantee3-weights}. }
    \label{table:theorem_summary}
    \begin{tabular}{lc}
    \toprule
    \textbf{Theorem} & \textbf{Description} \\
    \midrule
    \cref{theorem:general} & Recovery error bound for weighted replay with data-independent regularization\\
    \cref{theorem:regularization} & Recovery error bound for data-dependent regularization \\
    \cref{theorem:dependent-weights} & Recovery error bound for data-dependent weights \\
    \bottomrule
    \end{tabular}
\end{table*}

\section{Extra Details on Main Paper}

\subsection{Extra Notations and Figures}
In our proof, we will need to index our data by columns, thus we extend \cref{fig:data-assumption} and include examples of the index sets $\cM_i\subset [T]$ for each column $i$. Specifically, the $i$-th column of this partial matrix in \cref{fig:data-assumption-full}c corresponds to a partial \textit{trajectory} of samples $\{(x_{ti}, y_{ti})\}_{t\in \cM_i}$.



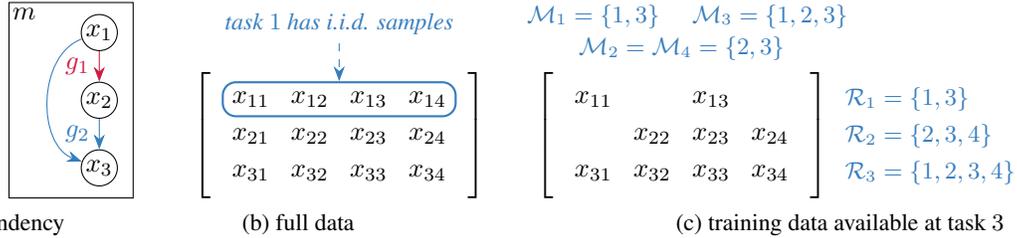
\begin{figure}
    \centering
    \begin{tikzpicture}
        \draw[black] (0.45,0.4) rectangle (-1.2,-2.2);

        \node[] at (-1,0.25) {$m$};
    
        \node[draw, circle,inner sep=0.25mm,align=center] (x1) at (0,0) {$x_1$}; 
        \node[draw, circle,inner sep=0.25mm,align=center] (x2) at (0,-0.9) {$x_2$};
        \node[draw, circle,inner sep=0.25mm,align=center] (x3) at (0,-1.8) {$x_3$};
        
        \draw[-latex, myred, -{Stealth[length=2mm]}] (x1) -- (x2) node [midway, left] {$g_1$};
        
        \draw[-latex, myblue, -{Stealth[length=2mm]}] (x2) -- (x3)  node [midway, left] {$g_2$};
        \draw[-latex, myblue, -{Stealth[length=2mm]}] (x1) to[out=200,in=160] (x3);
    \end{tikzpicture}
    \quad \quad 
    \begin{tikzpicture}[
		every matrix/.style={
			matrix of math nodes,
			column sep =0.5mm,
			row sep =0.5mm,
			left delimiter={[},
			right delimiter={]},
		}
		]
		\matrix (m) {
			x_{11} & x_{12} & x_{13} & x_{14} \\
			x_{21} & x_{22} & x_{23} & x_{24} \\
			x_{31} & x_{32} & x_{33} & x_{34} \\
		};
		\node[fit=(m-1-1)(m-1-4), draw, inner sep=0pt,thick, myblue,rounded corners=5pt] (iid) {};

        \node[above  = 0.4cm  of m] (iid2) {\footnotesize{\textit{\textcolor{myblue}{task $1$ has i.i.d. samples}}}};
		\draw[-latex, myblue, dashed, -{Stealth[length=2mm]}] (iid2) to[out=270,in=90] (iid);
    \end{tikzpicture}
    \quad 
    \begin{tikzpicture}[
		every matrix/.style={
			matrix of math nodes,
			column sep =0.5mm,
			row sep =0.5mm,
			left delimiter={[},
			right delimiter={]},
		}
		]
		\matrix (m) {
			x_{11} &  & x_{13}  &  \\
			     & x_{22} & x_{23} & x_{24} \\
			x_{31} & x_{32} & x_{33} & x_{34} \\
		};


        \node[above = 0.6cm of m-1-1] (m1) {\footnotesize{\textit{\textcolor{myblue}{$\cM_1=\{1,3\}$}}}};


        \node[above = 0.6cm of m-1-4] (m3) {\footnotesize{\textit{\textcolor{myblue}{$\cM_3=\{1, 2,3\}$}}}};

        \node[above  = 0.05cm  of m] (m4) {\footnotesize{\textit{\textcolor{myblue}{$\cM_2=\cM_4=\{2,3\}$}}}};

        \node[right  = 0.5cm  of m-1-4] (r1) {\footnotesize{\textit{\textcolor{myblue}{$\cR_1=\{1,3\}$}}}};
        \node[right  = 0.5cm  of m-2-4] (r2) {\footnotesize{\textit{\textcolor{myblue}{$\cR_2=\{2,3,4\}$}}}};
        \node[right  = 0.5cm  of m-3-4] (r3) {\footnotesize{\textit{\textcolor{myblue}{$\cR_3=\{1,2,3,4\}$}}}};
    \end{tikzpicture}

    \makebox[0.15\textwidth]{\footnotesize  (a) data dependency }
    \makebox[0.35\textwidth]{\footnotesize \centering (b) full data  }
    \makebox[0.48\textwidth]{\footnotesize \centering (c) training data available at task $3$ }

    \caption{Example of our setup  ($T=3,m=4$): \ref{fig:data-assumption}a shows $m$ i.i.d. trajectories generated by $x_t=g_t(x_1,\dots,x_{t-1})$; \ref{fig:data-assumption}b shows full data in a matrix, where columns represent trajectories and each row represents data of each task; \ref{fig:data-assumption}c shows the index sets $\cR_t,\cM_i$ and data available at task $3$. \label{fig:data-assumption-full} }
\end{figure}

\subsection{On Task-Specific Predictors}

At first glance, \cref{eq:regress} seems to ask all tasks to share a common true predictor $f^*$, similarly to prior work \citep{Evron-COLT2022,Pengb-NeurIPS2022,Peng-ICML2023,Li-CoLLAs2023,Zhao-ICML2024,Elenter-arXiv2023v2,Zhu-arXiv2025}. But the extra presence of \cref{eq:dependency} on top of \cref{eq:regress} recovers a case where each task $t$ possesses its own predictor $f_t^*$:
\begin{example}\label{example:ft}
    Consider the following model:
    \begin{equation}\label{eq:tmp-dependency-fh}
        \begin{split}
            y_t &= f_t^*(\overline{x}_t) + v_t,\ \  \overline{x}_t = \overline{g}_t(\overline{x}_1,\dots,\overline{x}_{t-1}),   \\ 
        f_t^* &= f_{t-1}^* \circ h_t. 
        \end{split}
    \end{equation}
    In \cref{eq:tmp-dependency-fh}, the first two equations are identical to  \cref{eq:regress} and \cref{eq:dependency}, despite the different notations $\overline{x}_t$ and $\overline{g}_t$. Moreover, \cref{eq:tmp-dependency-fh} contains an extra equation $f_t^* = f_{t-1}^* \circ h_t$, where $h_t$ is some \textit{known} invertible mapping capturing a relationship between the predictors of two consecutive tasks. 
    Let $h_1$ be the identity mapping, and write $h_{1:t}:=h_1 \circ \cdots \circ h_t$ and $x_t:=h_{1:t} (\overline{x}_t)$. By basic algebra we eliminate the presence of $f_t^*, \overline{x}_t$ in \cref{eq:tmp-dependency-fh}, and obtain $y_t = f_1^*  (x_t) + v_t$ and $ x_t = h_{1:t} (\overline{g}_t (h_{1}^{-1}(x_1),\dots, h_{1:t-1}^{-1} (x_{t-1}) ))$. The former equation is identical to \cref{eq:regress} with $f^*=f_1^*$, and the latter equation on $x_t$ defines $g_t$ in \cref{eq:dependency}. Thus, estimating $f_1^*$ in \cref{eq:tmp-dependency-fh} reduces to estimating $f^*$ in \cref{eq:regress} and \cref{eq:dependency} with a particular choice of $g_t$. Moreover, once $f_1^*$ is estimated, we estimate all other $f_t$'s via $f_t^* = f_{t-1}^* \circ h_t$. Conversely, our model reduces to \cref{eq:tmp-dependency-fh} when $h_t$'s are all identity mappings. Thus, when $h_t$ is known and invertible, imposing task dependency on predictors \cref{eq:tmp-dependency-fh}  is equivalent to imposing such dependency on input samples \cref{eq:dependency}. 
\end{example}

\begin{remark}\label{remark:linear-task} 
    Consider two tasks with linear predictors $f_1^*,f_2^*$ parameterized by $\theta_1^*,\theta_2^*$ respectively, and assume they are related by some invertible matrix $H_2$, that is $\theta_2^*=H_2 \theta_1^*$. Since such $H_2$ always exists, if $H_2$ is unknown, then the relationship $\theta_2^*=H_2 \theta_1^*$ gives us no extra information for learning either $\theta_1^*$ or $\theta_2^*$. Thus, to understand the benefit of having $\theta_2^*=H_2 \theta_1^*$, we consider $H_2$ is known. Furthermore, in \citet{Peng-MoCL2025,Dar-SIAM-J-MDS2024}, the relationship $\theta_2^*=H_2 \theta_1^*$ is  assumed to hold up to some additive Gaussian noise. In \citet{Peng-MoCL2025}, it is shown that Kalman filtering and smoothing improve the performance on task 1 after learning task 2, an example of  \textit{positive backward transfer}. In \citet{Dar-SIAM-J-MDS2024}, a transfer learning setup is considered, and it is shown that learning task 1 would benefit learning task 2 in the underparameterized case but might hurt otherwise \cite[Fig. 1]{Dar-SIAM-J-MDS2024}. 
\end{remark}


\subsection{On \cref{assumption:G-Lipschitz}}\label{section:proof-assumption-G-Lipschitz}
We claimed in the main paper that \cref{assumption:G-Lipschitz} holds true as soon as all $f$ and $g_t$'s are Lipschitz continuous. A more precise statement of it is shown below in \cref{lemma:assumption_example-full}.
\begin{lemma}\label{lemma:assumption_example-full}
    Suppose that each transformation \(g_i\) is Lipschitz continuous with constant \(L_{g,i}\;(i\!\in[T])\). Assume $\cF$ consist of functions $f:\bbR^{d_x} \to \bbR^{d_y}$ such that $f(z)<\infty$ for any finite input $z$. Consider the ball $\bbB_{r'}(d_x)$ of radius  $r':= r\cdot \max_{t\in[T]} \prod_{i=1}^t L_{g,i}$. 
    Suppose \cref{assumption:Theta}  and \cref{eq:f-Lipschiz} hold with Lipschitz parameter $L_{\cF}$ and with the following norm $\| \cdot \|_{\cF}$:
    \begin{align*}
        \|f\|_{\mathcal F}:=\sup_{z\in\bbB_{r'}(d_x) }\|f(z)\|_2.
    \end{align*}
    Furthermore, assume $f$ is Lipschitz continuous with respect to its input with also constant $L_{\cF}$. Then rest conditions in \cref{assumption:G-Lipschitz} hold with
    \begin{align*}
        L_G=2L_{\cF} \max_{t\in[T]} \prod_{i=1}^t L_{g,i}, \quad \quad K_G=\max\left(1,  2rL_G  + 2C_{\max} \right).
    \end{align*}
\end{lemma}

\begin{proof}
First note that we have (for all $z, z'\in\bbR^{d_x}$, for all $f\in\cF$, and for all $t\in[T]$)
\begin{equation}\label{tmp:assump-G-Lipschitz}
    \begin{split}
        \|G_{f,t}(z) - G_{f,t}(z') \|_2 &= \| (f-f^*)\circ g_t\circ\cdots\circ g_1(z)-(f-f^*)\circ g_t\circ\cdots\circ g_1(z')\|_2\\
    &\leq 2L_{\cF} \cdot \| g_t\circ\cdots\circ g_1(z)-g_t\circ\cdots\circ g_1(z')\|_2\\
    &\leq 2L_{\cF} L_{g,1} \cdot \| g_{t-1}\circ\cdots\circ g_1(z)-g_{t-1}\circ\cdots\circ g_1(z')\|_2 \\
    &\leq 2L_{\cF}\prod_{i=1}^tL_{g,i} \cdot \|z-z'\|_2\,.
    \end{split}
\end{equation}
Therefore, $L_G=2L_{\cF}\max_{t\in[T]}\prod_{i=1}^t L_{g,i}$ in this case. 

On the other hand, for any $f,f'\in \cF$ parametrized by $\theta,\theta'$ respectively, and for any $z\in \bbB_r(d_x)$, we have $g_t\circ\cdots\circ g_1(z)\in \bbB_{r'}(d_x)$ and thus
\begin{align*}
    \sup_{z \in \bbB_r(d_x)} \nmm{G_{f,t}(z) - G_{f',t}(z)}_2 &= \sup_{z \in \bbB_r(d_x)} \| f_{\theta}\circ g_t\circ\cdots\circ g_1(z)-f_{\theta'}\circ g_t\circ\cdots\circ g_1(z)\|_2\\
    &\leq \|f_{\theta}-f_{\theta'}\|_{\cF} 
\end{align*} 

Finally, to bound $\left| \|G_{f,t}(z) \|_2^2 - \| G_{f',t}(z) \|_2^2 \right| $, let us define
\begin{align*}
    C_{\max}=\max\biggl\{\sup_{\theta\in\Theta} \|f_{\theta}\circ g_t\circ\cdots\circ g_1(0)\|_2, \|f^*\circ g_t\circ\cdots\circ g_1(0)\|_2\biggl\}.
\end{align*}
Since $\Theta$ is bounded, $\sup_{\theta\in\Theta} \|f_{\theta}\!\circ \!g_t\!\circ\!\cdots\!\circ\! g_1(0)\|_2$ is finite, and thus $C_{\max}$ is a finite constant. It then follows from the triangular inequality and \cref{tmp:assump-G-Lipschitz} that 
\begin{align*}
    \|G_{f,t}(z) \|_2 \leq \|G_{f,t}(z) - G_{f,t}(0) \|_2 + \|G_{f,t}(z) \|_2 \leq L_G \cdot \|z\|_2 + C_{\max}.
\end{align*}
We then obtain
\begin{align*}
    &\sup_{z \in \bbB_r(d_x)} \left| \|G_{f,t}(z) \|_2^2 - \| G_{f',t}(z) \|_2^2 \right|\\
    \leq& \sup_{z \in \bbB_r(d_x)}\nmm{G_{f,t}(z) - G_{f',t}(z)}_2\cdot \nmm{G_{f,t}(z) + G_{f',t}(z)}_2\\
    \leq& \sup_{z \in \bbB_r(d_x)} \|f-f'\|_{\cF}\cdot \left( 2L_G \cdot \|z\|_2 + 2C_{\max} \right) \\ 
    =& \ \|f-f'\|_{\cF}\cdot \left( 2rL_G  + 2C_{\max} \right)
\end{align*}
Set $K_G=\max\left(1,  2rL_G  + 2C_{\max} \right)$ and we finish the proof.
\end{proof}

\section{Full Statement of \cref{theorem:general} and Its Proof}\label{section:main-theorem-proof}
Note that the values of $K_G$ depend on $r_x$ in \cref{assumption:G-Lipschitz}. Sometimes we make this dependency explicit by writing $K_G(r)$ for any radius $r>0$.
\begin{theorem}[Full Version of \cref{theorem:general}]\label{theorem:general-appendix_n}
    Recall the definitions of $M_2$ and $\kappa$ in \cref{eq:def-M2} and \cref{eq:def:kappa}: 
    \begin{align*}
        M_2=\sup_{f\in \cF} \sup_{t\in [T], w_t>0} \bbE\left[ \|G_{f,t}(x_1) \|_2^2 \right], \quad \kappa = \sup_{f \in \cF\ \backslash \{f^*\}}\sup_{t \in [T],w_t>0}\frac{\bbE\left[\nmm{G_{f,t}(x_1)}_2^4\right]^{1/2}}{\bbE\left[\nmm{G_{f,t}(x_1)}_2^2\right]}. 
    \end{align*}
    Let $\delta\in(0,1]$.  Let $w_{\text{avg}}:=\frac{1}{T}\sum_{t\in[T]} w_t$. Let $C$ be some constant with $C>1$. 
    Define 
    \begin{align*}
    \underbar{r}_x := \sigma \left[3 + 16\sqrt{\frac{\ln(4m/\delta)}{d_x}}\right],
    \quad 
    \underbar{r}_v := 2\nu \left[ \sqrt{d_y} + 8\sqrt{2\ln(4m/\delta)}\right].
    \end{align*}
    
    Define $n' := \min_{t \in [T]}n_t$,
    and $n'' := \min_{t: w_t > 0}(n_t / w_t)$.
    Suppose \cref{assumption:data,assumption:Theta,assumption:G-Lipschitz,assumption:finite-moment} hold and let $\sigma,\nu, r_x, K_G, L_G$ be defined therein. Let $K_G(r)\leq k_G r^\alpha$ for some $\alpha>0$ and $k_G>0$.
    Assume that 
    \begin{align}
    \begin{split}
    n'\geq\frac{2\kappa^2C^2}{C-1}&\left[\ln(4/\delta) + p\ln\left(1+\frac{2B(C+1)}{C(1+\underbar{r}_v)}Tn''\right)\right],\\
    &m \geq \frac{\sqrt{2}L_G\sigma \delta \sqrt{(d_x + 64)}}{e^{d_x/256}\sqrt{M_2}}.
    \end{split}
    \end{align}

    Let $\hat{\theta}_T\in \Theta\subset \bbR^{p}$ be a global minimizer of \cref{eq:general-cl} with the regularization parameter $\lambda \leq {  4\nu^2}/{ T  n''}$.
    Then, with probability at least $1-\delta$ the quantity $\bbE\left[\frac{1}{T}\sum_{t \in [T]}w_t \cdot \nmm{f^*(x_t) - f_{\hat{\theta}_T}(x_t)}_2^2\right]$ is upper bounded by
\begin{align}
    \begin{split}
    &\Big\{2\alpha C w_{\text{avg}} k_G\left(1 + 8\nu\left[1 + 8\sqrt{2\ln(4m/\delta)}\right]\right)
    \Big[\left(\left[3 + 16\sqrt{\frac{\ln(4m/\delta)}{d_x}}\right]^{\alpha} + \frac{1}{16^{\alpha}d_x^{\alpha/2}}\right) \\
    &\ \ + \left({\max\left\{\frac{256}{d_x}\ln\left(\frac{32L_G\sqrt{2M_2}}{\alpha 16^{\alpha}(C+1)w_{\text{avg}}k_G}\frac{d_x^{\alpha/2 - 1}\sqrt{d_x+64}}{\sigma^{\alpha-1}}m^2\right), 0\right\}}\right)^{\alpha/2}\Big]\Big\}\frac{\sigma^{\alpha}}{Tn''}\\
    &\ \ + 8C\nu^2\frac{p\ln\left(1+\frac{2B(C+1)}{C(1+8\nu\left[1 + 8\sqrt{2\ln(4m/\delta)}\right])}Tn''\right) + \Omega_T(f^*)}{Tn''} + 8C\nu^2\frac{\ln(4/\delta)}{Tn''}.
    \end{split}
\end{align}
\end{theorem}

\begin{corollary}
Under the setting of \cref{theorem:general-appendix_n}, if $\nu = 0$ and
\begin{align}
    \begin{split}
    n'
    > \frac{2\kappa^2 C}{C-1} &\left[\ln(4/\delta) + p\ln\left(1+\frac{2B(C+1)}{C(1+8\nu\left[1 + 8\sqrt{2\ln(4m/\delta)}\right])}Tn''\right)\right],\\
    &m \geq \frac{\sqrt{2}L_G\sigma \delta \sqrt{(d_x + 64)}}{e^{d_x/256}\sqrt{M_2}},
    \end{split}
    \end{align}
then with probability at least $1-\delta$ the quantity $\bbE\left[\frac{1}{T}\sum_{t \in [T]}w_t \cdot \nmm{f^*(x_t) - f_{\hat{\theta}_T}(x_t)}_2^2\right]$ is upper bounded by
\begin{align}
    \begin{split}
    &\Big\{2\alpha C w_{\text{avg}} k_G
    \Big[\left(\left[3 + 16\sqrt{\frac{\ln(4m/\delta)}{d_x}}\right]^{\alpha} + \frac{1}{16^{\alpha}d_x^{\alpha/2}}\right) \\
    &\ \ + \left({\max\left\{\frac{256}{d_x}\ln\left(\frac{32L_G\sqrt{2M_2}}{\alpha 16^{\alpha}(C+1)w_{\text{avg}}k_G}\frac{d_x^{\alpha/2 - 1}\sqrt{d_x+64}}{\sigma^{\alpha-1}}m^2\right), 0\right\}}\right)^{\alpha/2}\Big]\Big\}\frac{\sigma^{\alpha}}{Tn''}.
    \end{split}
\end{align}
\end{corollary}

\begin{corollary}
Under the setting of \cref{theorem:general-appendix_n}, if  $\sigma = 0$ and
\begin{align}
    \begin{split}
    n'
    > \frac{2\kappa^2C^2}{C-1} &\left[\ln(4/\delta) + p\ln\left(1+\frac{2B(C+1)}{C(1+8\nu\left[1 + 8\sqrt{2\ln(4m/\delta)}\right])}Tn''\right)\right],
    \end{split}
    \end{align}
then with probability at least $1-\delta$ the quantity $\bbE\left[\frac{1}{T}\sum_{t \in [T]}w_t \cdot \nmm{f^*(x_t) - f_{\hat{\theta}_T}(x_t)}_2^2\right]$ is upper bounded by
\begin{align}
    \begin{split}
    8C\frac{\nu^2}{Tn''}\left[{p\ln\left(1+\frac{2B(C+1)}{C(1+8\nu\left[1 + 8\sqrt{2\ln(4m/\delta)}\right])}Tn''\right) + \Omega_T(f^*)} + {\ln(4/\delta)}\right].
    \end{split}
\end{align}
\end{corollary}

\begin{corollary}
Under the setting of \cref{theorem:general-appendix_n}, and if $T$ or $n_t$ is large enough
such that $n'> \kappa^2C p\ln(Tn'') /(C-1)$, then with probability at least $1-\delta$ we have
\begin{align}
    \begin{split}
    \bbE\left[\frac{1}{T}\sum_{t \in [T]}w_t \cdot \nmm{f^*(x_t) - f_{\hat{\theta}_T}(x_t)}_2^2\right]\leq \tilde{\mathcal{O}}\left(\nu^2\frac{p+ \Omega_T(f^*)+\ln(4/\delta)}{Tn''} + \sigma^{\alpha}\frac{1}{Tn''}\right).
    \end{split}
\end{align}
\end{corollary}

\begin{proof}[Proof of \cref{theorem:general-appendix_n}]
First note that $\kappa$ and $M_2$ are finite under \cref{assumption:finite-moment} (see \cref{lemma:finite-kappa-M2}). We then derive some basic inequalities. Suppose $\hat{f}:=f_{\hat{\theta}_{T}}$ is a global minimizer for \cref{eq:general-cl}. Since $f^*$ is realizable by \cref{assumption:Theta}, the optimality of $\hat{f}$ implies
\begin{align*}
    &\ \frac{1}{T}   \sum_{(t,i)\in \cM} \frac{w_t}{n_t } \cdot \cL \Big( y_{ti},  \hat{f}(x_{ti}) \Big)
    + \lambda\Omega_T(\hat{f})
    \leq \frac{1}{T}   \sum_{(t,i)\in \cM} \frac{w_t}{n_t } \cdot \cL \Big( y_{ti},  f^*(x_{ti}) \Big)
    + \lambda\Omega_T(f^*)\\ 
    \Leftrightarrow &\ \frac{1}{T}   \sum_{(t,i)\in \cM} \frac{w_t}{n_t } \cdot \nmm{ y_{ti} - \hat{f}(x_{ti}) }_2^2 + \lambda\Omega_T(\hat{f}) \leq  \frac{1}{T}   \sum_{(t,i)\in \cM} \frac{w_t}{n_t } \cdot \nmm{ y_{ti} - f^*(x_{ti}) }_2^2 + \lambda\Omega_T({f^*}). 
\end{align*}
Using the fact $y_{ti} = f^*(x_{ti}) + v_{ti}$, expanding the terms, and simplifying, we obtain
\begin{align*}
    &\ \frac{1}{T}   \sum_{(t,i)\in \cM} \frac{w_t}{n_t } \cdot \nmm{ f^*(x_{ti}) + v_{ti} - \hat{f}(x_{ti}) }_2^2 + \lambda\Omega_T(\hat{f}) \leq  \frac{1}{T}   \sum_{(t,i)\in \cM} \frac{w_t}{n_t } \cdot \nmm{ v_{ti} }_2^2 + \lambda\Omega_T(f^*), 
\end{align*}
which is equivalent to 
\begin{align*}
     \sum_{(t,i)\in \cM} \frac{w_t}{Tn_t } \cdot \nmm{ f^*(x_{ti}) - \hat{f}(x_{ti}) }_2^2 \leq  \sum_{(t,i)\in \cM} \frac{2w_t}{Tn_t } \cdot \IP{v_{ti}}{\hat{f}(x_{ti}) -  f^*(x_{ti}) } + \lambda\left[\Omega_T(f^*) - \Omega_T(\hat{f})\right].
\end{align*}

Multiplying both sides by $2$, rearranging the terms, and with the  definition 
\begin{equation}\label{eq:residual-term-Mf}
M(f) := \frac{1}{T}   \sum_{(t,i)\in \cM} \frac{w_t}{n_t } \left[ 4 \cdot \IP{v_{ti}}{f(x_{ti}) -  f^*(x_{ti}) } - \nmm{ f^*(x_{ti}) - f(x_{ti}) }_2^2\right],
\end{equation}
we obtain
\begin{align*}
\frac{1}{T}   \sum_{(t,i)\in \cM} \frac{w_t}{n_t } \cdot \nmm{ f^*(x_{ti}) - \hat{f}(x_{ti}) }_2^2 
\leq  M(\hat{f}) + 2\lambda\left[\Omega_T(f^*) - \Omega_T(\hat{f})\right].
\end{align*}

Multiplying both sides by $C$ (recall $C>1$) yields
\begin{align*}
C \cdot \frac{1}{T}   \sum_{(t,i)\in \cM} \frac{w_t}{n_t } \cdot \nmm{ f^*(x_{ti}) - \hat{f}(x_{ti}) }_2^2 
\leq  C \cdot M(\hat{f}) + 2C \lambda\left[\Omega_T(f^*) - \Omega_T(\hat{f})\right].
\end{align*}
We can furthermore rewrite the above inequality with the definition
\begin{equation}\label{eq:quadratic-term-Q}
    Q(f, C) := \frac{1}{T}   \sum_{(t,i)\in \cM} \frac{w_t}{n_t }  \left(\bbE\left[\nmm{f(x_t) - f^*(x_t)}_2^2\right] - C \cdot \nmm{f(x_{ti}) - f^*(x_{ti})  }_2^2 \right),
\end{equation}
arriving at
\begin{align}\label{eq:check_point}
\bbE\left[\frac{1}{T}\sum_{t \in [T]}w_t \cdot \nmm{f^*(x_t) - \hat{f}(x_t)}_2^2\right] 
&\leq Q(\hat{f},C) +   C \cdot M(\hat{f}) + 2C \lambda\left[\Omega_T(f^*) - \Omega_T(\hat{f})\right]. 
\end{align}

Taking any $r_x, r_v, \varepsilon, u>0$, and combining \cref{prop:final_e_1_n,prop:final_e_2_n}, we see that
\begin{align*}
    \bbP\left( Q(f,C) >  2 B_{\text{sq}}  + (C+1) w_{\text{avg}} K_G\varepsilon^2,  \forall f\in \cF \right) &\leq \left|\cF(\varepsilon) \right| \cdot \exp\left( - \frac{(C-1)\min_{t\in[T]} n_t}{2C^2\kappa^2} \right) \\
    &\ \ + m \cdot \exp\left( - \frac{d_x(r_x-2\sigma)^2}{128\sigma^2} \right),
\end{align*}
\begin{align*}
    \bbP\left( M(f) > w_{\text{avg}} K_G \varepsilon \left( 1 +4 r_v  \right) + u/2C,\ \forall f\in \cF \right) 
    &\leq |\cF(\varepsilon)|\cdot \exp \left(- u \cdot \frac{T}{16C\nu^2} \min_{t\in [T], w_t> 0} \frac{ n_t }{w_t} \right)\\
    &\ \ \ + m \cdot \exp\left( - \frac{d_x(r_x-2\sigma)^2}{128\sigma^2} \right)+ m \cdot \exp\left( - \frac{(r_v-2\nu \sqrt{d_y})^2}{128\nu^2} \right). 
\end{align*}
Since the common term $m \cdot \exp\left( - \frac{d_x(r_x-2\sigma)^2}{128\sigma^2} \right)$ in the above two probability bounds arise from bounding the norm of $x_{1i}$'s, we apply the union bound and obtain that 
\begin{align*}
    \bbE\left[\frac{1}{T}\sum_{t \in [T]}w_t \cdot \nmm{f^*(x_t) - \hat{f}(x_t)}_2^2\right] > (C+1) w_{\text{avg}} K_G(r_x)\varepsilon^2 &+ C w_{\text{avg}} K_G(r_x) \left( 1 +4 r_v  \right)\varepsilon + \frac{u}{2} + 2 B_{\text{sq}}(r_x) \\
    &\ +  2C \lambda\left[\Omega_T(f^*) - \Omega_T(\hat{f})\right],
\end{align*}
holds with probability at most
\begin{equation}\label{eq:tmp-probability-final_n}
\begin{split}
    \left|\cF(\varepsilon) \right| \cdot
    &\left(\exp\left( - \frac{(C-1)\min_{t\in[T]} n_t}{2C^2\kappa^2} \right) + 
    \exp \left(- u \cdot \frac{T}{16C\nu^2} \min_{t\in [T], w_t> 0} \frac{ n_t }{w_t} \right)\right) \\
    &\ + m \cdot \exp\left( - \frac{d_x(r_x-2\sigma)^2}{128\sigma^2} \right) + m \cdot \exp\left( - \frac{(r_v-2\nu \sqrt{d_y})^2}{128\nu^2} \right),
\end{split}
\end{equation}

Recall $n' = \min_{t \in [T]}n_t$, $n'' = \min_{t\in [T], w_t> 0} \frac{ n_t }{w_t}$. For any $\delta \in (0, 1]$, consider the following problem (which essentially minimizes the bound while maintaining high probability $1-\delta$):
\begin{align}\label{eq:opt_bound}
u^* = &\min_{r_x, r_v, \varepsilon, u \geq 0} \quad u \tag{O}\\
&\st \quad (C+1) w_{\text{avg}} K_G(r_x)\varepsilon^2 + C w_{\text{avg}} K_G(r_x) \left( 1 +4 r_v  \right)\varepsilon  + 2 B_{\text{sq}}(r_x) - \frac{u}{2} \leq 0, \tag{C1} \label{eq:C1}\\
&\quad 
\ln(4/\delta) + \ln(\left|\cF(\varepsilon) \right|) - \frac{(C-1)n'}{2C^2\kappa^2} \leq 0, \tag{C2} \label{eq:C2}\\
&\quad \ln(4/\delta) + \ln(\left|\cF(\varepsilon) \right|)
- u \cdot \frac{Tn''}{16C\nu^2} \leq 0, \tag{C3} \label{eq:C3}\\
&\quad \ln(4m/\delta) - \frac{d_x(r_x-2\sigma)^2}{128\sigma^2} \leq 0, \tag{C4} \label{eq:C4}\\
&\quad \ln(4m/\delta) - \frac{(r_v-2\nu \sqrt{d_y})^2}{128\nu^2} \leq 0. \tag{C5} \label{eq:C5}
\end{align}

\textbf{Choosing $r_x, r_v, \epsilon, u$}:

Recall  
\begin{align*}
\underbar{r}_x := \sigma \left[3 + 16\sqrt{\frac{\ln(4m/\delta)}{d_x}}\right],
\quad 
\underbar{r}_v := 2\nu \left[ \sqrt{d_y} + 8\sqrt{2\ln(4m/\delta)}\right].
\end{align*}
Then from \eqref{eq:C4}, and \eqref{eq:C5} we have that $r_x \geq \underbar{r}_x$,
and $r_v \geq \underbar{r}_v$. From \eqref{eq:C1}, and \eqref{eq:C3} we
have
\begin{align*}
\frac{u}{2} &\geq (C+1) w_{\text{avg}} K_G(r_x)\varepsilon^2 + C w_{\text{avg}} K_G(r_x) \left( 1 +4 r_v  \right)\varepsilon  + 2 B_{\text{sq}}(r_x) \geq 2 B_{\text{sq}}(r_x),\\
\frac{u}{2} &\geq 8C\nu^2\frac{\ln(4/\delta) + \ln(\left|\cF(\varepsilon) \right|)}{Tn''}
\geq 0.
\end{align*}
We can upper bound the optimization program by 
setting $u$ to be sum of the lower bounds obtained earlier 
(by non-negativity), this gives us
\begin{align*}
u^* &\leq \min_{r_x \geq \underbar{r}_x, r_v \geq \underbar{r}_v, \varepsilon \geq 0}
\left[(C+1) w_{\text{avg}} K_G(r_x)\varepsilon^2 + C w_{\text{avg}} K_G(r_x) \left( 1 +4 r_v  \right)\varepsilon  + 2 B_{\text{sq}}(r_x) + 8C\nu^2\frac{\ln(4/\delta) + \ln(\left|\cF(\varepsilon) \right|)}{Tn''}\right],\\
&= 8C\nu^2\frac{\ln(4/\delta)}{Tn''} + \min_{r_x \geq \underbar{r}_x, \varepsilon \geq 0}
\left[\left\{(C+1) w_{\text{avg}} \varepsilon^2 + C w_{\text{avg}} \left( 1 +4 \underbar{r}_v  \right)\varepsilon \right\}K_G(r_x)  + 2 B_{\text{sq}}(r_x) + 8C\nu^2\frac{\ln(\left|\cF(\varepsilon) \right|)}{Tn''}\right],\\
\end{align*}
Rewrite the right-hand side of the above and we have: 
\begin{align}\label{eq:temp_u_1}
\begin{split}
u^* \leq
\min_{\epsilon > 0}\Big[ 8C\nu^2\frac{\ln(4/\delta)}{Tn''} + 
\Big[&8C\nu^2\frac{\ln(\left|\cF(\varepsilon) \right|)}{Tn''} 
+ \left\{(C+1) w_{\text{avg}} \varepsilon^2 + C w_{\text{avg}} \left( 1 +4 \underbar{r}_v  \right)\varepsilon \right\}\\
& \times
\min_{r_x \geq \underbar{r}_x}\left(K_{G}(r_x)
+ \frac{2}{\left\{(C+1) w_{\text{avg}} \varepsilon^2 + C w_{\text{avg}} \left( 1 +4 \underbar{r}_v  \right)\varepsilon \right\}}B_{\text{sq}}(r_x)\right)\Big].
\end{split}
\end{align}

We recall \cref{lemma:b_sq_n} which states that when $r_x > 3\sigma$ we have
\begin{align*}
    \begin{split}
        \frac{B_{\text{sq}}(r_x)}{128w_{\text{avg}}}&\leq 
        \left[\frac{L_G\sigma}{d_x}\sqrt{(d_x + 64)}\exp\left( - \frac{d_x (r_x-2\sigma)^2}{256\sigma^2}  \right)\right]^2
        + \sqrt{\frac{M_2}{32}}\frac{L_G\sigma}{d_x}\sqrt{(d_x + 64)}\exp\left( - \frac{d_x (r_x-2\sigma)^2}{256\sigma^2}  \right).
    \end{split}
\end{align*}
We now use the inequality
\begin{align*}
(r_x - 2\sigma)^2 \geq (\underbar{r}_x - 2\sigma)^2 + (r_x - \underbar{r}_x)^2 
> (r_x - \underbar{r}_x)^2 + \sigma^2 + 256\sigma^2\frac{\ln(4m/\delta)}{d_x},
\end{align*}
to obtain
\begin{align*}
\frac{L_G\sigma}{d_x}\sqrt{(d_x + 64)}\exp\left( - \frac{d_x (r_x-2\sigma)^2}{256\sigma^2}  \right)
\leq \frac{L_G\sigma}{4d_xe^{d_x/256}}\sqrt{(d_x + 64)}\frac{\delta}{m}\exp\left( - \frac{d_x (r_x-\underbar{r}_x)^2}{256\sigma^2}  \right).
\end{align*}
Therefore, when $m \geq \frac{\sqrt{2} L_G\sigma}{\sqrt{M_2}}\delta e^{-d_x/256}\frac{\sqrt{d_x+64}}{d_x}$, it is true that
\begin{align}
\begin{split}
\frac{L_G\sigma}{d_x}\sqrt{(d_x + 64)}\exp\left( - \frac{d_x (r_x-2\sigma)^2}{256\sigma^2}  \right)
&\leq \frac{L_G\sigma}{4d_xe^{d_x/256}}\sqrt{(d_x + 64)}\frac{\delta}{m}\exp\left( - \frac{d_x (r_x-\underbar{r}_x)^2}{256\sigma^2}  \right),\\
&\leq \sqrt{\frac{M_2}{32}}\exp\left( - \frac{d_x (r_x-\underbar{r}_x)^2}{256\sigma^2}  \right) \leq \sqrt{\frac{M_2}{32}}.
\end{split}
\end{align}
Based on the inequality $x^2+ax\leq 2ax$ when $0<x\leq a$, this gives the bound
\begin{align}\label{eq:upper_b_sq}
B_{\text{sq}}(r_x)
\leq 32\frac{w_{\mathrm{avg}}L_G\sigma}{d_x}\sqrt{2M_2(d_x + 64)}\exp\left( - \frac{d_x (r_x-2\sigma)^2}{256\sigma^2}  \right).
\end{align}
We substitute the upper bound \cref{eq:upper_b_sq} to \cref{eq:temp_u_1} and obtain
\begin{align*} 
\begin{split}
u^* \leq 8C\nu^2\frac{\ln(4/\delta)}{Tn''} + 
&\min_{\epsilon > 0}
\Big[8C\nu^2\frac{\ln(\left|\cF(\varepsilon) \right|)}{Tn''} + \left\{(C+1) w_{\text{avg}} \varepsilon^2 + C w_{\text{avg}} \left( 1 +4 \underbar{r}_v  \right)\varepsilon \right\} \\
& \ \ \times \min_{r_x \geq \underbar{r}_x}\left(K_{G}(r_x)
+ \frac{64w_{\text{avg}}L_G\sigma\sqrt{2M_2(d_x + 64)}}{d_x\left\{(C+1) w_{\text{avg}} \varepsilon^2 + C w_{\text{avg}} \left( 1 +4 \underbar{r}_v  \right)\varepsilon \right\}}\exp\left( - \frac{d_x (r_x-2\sigma)^2}{256\sigma^2}  \right)\right)\Big].
\end{split}
\end{align*}
Let $K_G(r) \leq k_Gr^{\alpha}$, then we have the bound $K_G(r_x-\underbar{r}_x + \underbar{r}_x)
\leq \alpha k_G(r_x-\underbar{r}_x)^{\alpha} + \alpha k_G( \underbar{r}_x)^{\alpha}$.
Now define $r := r_x - \underbar{r}_x$
this gives us

\begin{align}\label{eq:temp_u_2}
\begin{split}
u^* \leq
&8C\nu^2\frac{\ln(4/\delta)}{Tn''} + \min_{\epsilon > 0}
\Big[\alpha k_G\left\{(C+1) w_{\text{avg}} \varepsilon^2 + C w_{\text{avg}} \left( 1 +4 \underbar{r}_v  \right)\varepsilon \right\}\\
&\ \ \times \min_{r \geq 0}\left(( \underbar{r}_x)^{\alpha} + r^{\alpha}
+ \frac{64w_{\text{avg}}L_G\sigma\sqrt{2M_2(d_x + 64)}}{\alpha k_G d_x\left\{(C+1) w_{\text{avg}} \varepsilon^2 + C w_{\text{avg}} \left( 1 +4 \underbar{r}_v  \right)\varepsilon \right\}}\exp\left( - \frac{d_x (r - (2\sigma-\underbar{r}_x))^2}{256\sigma^2}  \right)\right)\\
&\ \  + 8C\nu^2\frac{\ln(\left|\cF(\varepsilon) \right|)}{Tn''}\Big].
\end{split}
\end{align}

Moreover, define the quantities:
\begin{align*}
&\tau := 64w_{\text{avg}}L_G\sigma\sqrt{2M_2(d_x+64)}/d_x, \beta := \tau \left(\alpha k_G\left\{(C+1) w_{\text{avg}} \varepsilon^2 + C w_{\text{avg}} \left( 1 +4 \underbar{r}_v  \right)\varepsilon \right\}\right)^{-1},\\
&\gamma := d_x/(256 \sigma^2),
\text{ and }\zeta := 2\sigma - \underbar{r}_x < 0,
\end{align*}

Rewriting \eqref{eq:temp_u_2} gives us
\begin{align}\label{eq:temp_u_4}
u^* \leq
\min_{\epsilon > 0}
\left[\frac{\tau}{\beta}\times \min_{r \geq 0}\left(( \underbar{r}_x)^{\alpha} + r^{\alpha}
+ \beta\exp\left( - \gamma(r - \zeta)^2 \right)\right)+ 8C\nu^2\frac{\ln(\left|\cF(\varepsilon) \right|)}{Tn''}\right] + 8C\nu^2\frac{\ln(4/\delta)}{Tn''}.
\end{align}

We can further upper bound this via setting 
$r = \sqrt{\max\{\ln(\beta \gamma^{\alpha/2})/\gamma, 0\}}$, 
and from \cref{lemma:opt_r} we obtain:
\begin{align}\label{eq:temp_u_4}
u^* \leq
\min_{\epsilon > 0}
\left[\frac{\tau}{\beta}\times \left[\underbar{r}_x^{\alpha} + \frac{1}{\gamma^{\alpha/2}}
+ \left({\max\{\ln(\beta \gamma^{\alpha/2})/\gamma, 0\}}\right)^{\alpha/2}\right] + 8C\nu^2\frac{\ln(\left|\cF(\varepsilon) \right|)}{Tn''}\right] + 8C\nu^2\frac{\ln(4/\delta)}{Tn''}.
\end{align}

When $0 < \epsilon < C(1+4\underbar{r}_v)/(C+1)$ we have 
\begin{align*}
\frac{\tau}{2Cw_{\text{avg}}\alpha k_G(1+4\underbar{r}_v)\epsilon} \leq \beta 
\leq \frac{\tau}{2(C+1)w_{\text{avg}}\alpha k_G\epsilon^2}.
\end{align*}
Restricting to this domain gives us,
\begin{align}\label{eq:temp_u_4}
\begin{split}
u^* \leq
\min_{\epsilon' \in (0, 1]}&
\Big[2\alpha C w_{\text{avg}} k_G(1 + 4\underbar{r}_v) \epsilon' \times \left[\underline{r}_x^{\alpha} + \frac{1}{\gamma^{\alpha/2}}
+ \left({\max\left\{\ln\left(\frac{\tau \gamma^{\alpha/2}}{2(C+1)w_{\text{avg}}\alpha k_G{\epsilon}'^2}\right)/\gamma, 0\right\}}\right)^{\alpha/2}\right] \\
&\ \ + 8C\nu^2\frac{\ln(\left|\cF(\varepsilon' C(1+4\underbar{r}_v)/(C+1)) \right|)}{Tn''}\Big] + 8C\nu^2\frac{\ln(4/\delta)}{Tn''},
\end{split}
\end{align}
where $\epsilon'=\frac{(C+1)\epsilon}{C(1+4\underbar{r}_v)}$.

Now set $\epsilon' = 1/Tn''$ this gives us the bound
\begin{align}\label{eq:temp_u_4}
u^* \leq
2\alpha C w_{\text{avg}} k_G(1 + 4\underbar{r}_v)&\left[\underline{r}_x^{\alpha} + \frac{1}{\gamma^{\alpha/2}}
+ \left({\max\left\{\ln\left(\frac{\tau \gamma^{\alpha/2}}{2(C+1)w_{\text{avg}}\alpha k_G}T^2{n''}^2\right)/\gamma, 0\right\}}\right)^{\alpha/2}\right]\frac{1}{Tn''} \\
&\ \ + 8C\nu^2\frac{p\ln\left(1+\frac{2B(C+1)}{C(1+\underbar{r}_v)}Tn''\right)}{Tn''} + 8C\nu^2\frac{\ln(4/\delta)}{Tn''},
\end{align}
plugging back the values and retaining $\sigma, \nu, d_x, n', n'', m$ and $T$ gives us
\begin{align}\label{eq:laast_eq}
\begin{split}
u^* \leq
2\alpha C w_{\text{avg}} &k_G\left(1 + 8\nu\left[1 + 8\sqrt{2\ln(4m/\delta)}\right]\right) \times 
\Big[\sigma^{\alpha}\left(\left[3 + 16\sqrt{\frac{\ln(4m/\delta)}{d_x}}\right]^{\alpha} + \frac{16^{\alpha}}{d_x^{\alpha/2}}\right) \\
&\ \ + \left({\max\left\{256\frac{\sigma^2}{d_x}\ln\left(\frac{32T^2(n'')^2L_G\sqrt{2M_2}}{\alpha 16^{\alpha}(C+1)k_G}\frac{d_x^{\alpha/2 - 1}\sqrt{d_x+64}}{\sigma^{\alpha-1}}\right), 0\right\}}\right)^{\alpha/2}\Big]\frac{1}{Tn''}\\
&\ \ + 8C\nu^2\frac{p\ln\left(1+\frac{2B(C+1)}{C(1+\underbar{r}_v)}Tn''\right)}{Tn''} + 8C\nu^2\frac{\ln(4/\delta)}{Tn''},
\end{split}
\end{align}
when $m \geq \frac{L_G\sigma}{\sqrt{2M_2}}\delta e^{-d_x/256}\frac{\sqrt{d_x+64}}{d_x}$ and $n'\geq\frac{2\kappa^2C^2}{C-1}\left[\ln(4/\delta) + p\ln\left(1+\frac{2B(C+1)}{C(1+\underbar{r}_v)}Tn''\right)\right]$. Finally we
upper bound the term $ 2C \lambda\left[\Omega_T(f^*) - \Omega_T(\hat{f})\right]$
by $8C\nu^2\Omega_T(f^*)/Tn''$. Combining this upper bound,
\eqref{eq:laast_eq}, and \eqref{eq:check_point}, finishes the proof.

\end{proof}

\subsection{\cref{prop:final_e_1_n} and Its Proof}

\begin{proposition}\label{prop:final_e_1_n}
Fix $u>0$. Let $C$ be some constant larger than $1$, and write $w_{\text{avg}}:=\frac{1}{T}\sum_{t\in[T]} w_t$. Recall the quadratic term $Q(f,C)$ defined in \cref{eq:quadratic-term-Q}:
\begin{align*}
    Q(f,C):= \frac{1}{T}   \sum_{(t,i)\in \cM} \frac{w_t}{n_t }  \left(\bbE\left[\nmm{ G_{f,t}(x_1) }_2^2\right] - C \cdot \nmm{G_{f,t}(x_{1i})   }_2^2  \right).
\end{align*}
Suppose \cref{assumption:G-Lipschitz,assumption:finite-moment} hold, and let $K_G, L_G$ be defined therein. Let $M_2$ and $\kappa$ be defined in \cref{eq:def-M2} and \cref{eq:def:kappa} respectively.
For any $r_x > 2\sigma$, 
define $\tilde{G}_{f,t}:= G_{f,t} \circ \cP_{\bbB_{r_x}(d_x)}$ and
\begin{equation}\label{eq:def_Bsq}
B_{\text{sq}} := \sup_{f \in \cF}\left|\bbE\left[\frac{1}{T}\sum_{t \in [T]}
w_t\left(\nmm{\tilde{G}_{f,t}(x_1)}_2^2
-\nmm{G_{f,t}(x_1)}_2^2\right)\right]\right|.
\end{equation}
Let $\cF(\varepsilon)$ be the smallest $\varepsilon$-net of $\cF\ \backslash \{f^*\}$,
then we have 
\begin{align*}
    \bbP\left( Q(f,C) >  2 B_{\text{sq}}  + (C+1) w_{\text{avg}} K_G\varepsilon^2,  \forall f\in \cF \right) &\leq \left|\cF(\varepsilon) \right| \cdot \exp\left( - \frac{(C-1)\min_{t\in[T]} n_t}{2C^2\kappa^2} \right) \\
    &\ \ + m \cdot \exp\left( - \frac{d_x(r_x-2\sigma)^2}{128\sigma^2} \right).
\end{align*}
\end{proposition}
\begin{proof}
For any $f\in\cF$, let $f'\in \cF(\varepsilon)$ be such that $\| f' - f \|_{\cF}\leq \varepsilon$. It follows from \cref{assumption:G-Lipschitz} that
\begin{equation}\label{eq:tmp-G-square-diff}
    \begin{split}
        \bbE\left[\frac{1}{T}\sum_{t \in [T]}w_t \cdot \left( 
        \nmm{\tilde{G}_{f,t}(x_1)}_{2}^2-\nmm{\tilde{G}_{f',t}(x_1)}_{2}^2 \right)\right]  &\leq w_{\text{avg}} K_G\nmm{f - f'}_{\cF}^2; \\ 
    \frac{C}{T}\sum_{(t, i) \in \M}\frac{w_t}{n_t} \cdot \left( 
    \nmm{\tilde{G}_{f',t}(x_1)}_{2}^2-\nmm{\tilde{G}_{f,t}(x_1)}_{2}^2 \right) &\leq C w_{\text{avg}} K_G\nmm{f - f'}_{\cF}^2.
    \end{split}
\end{equation}

We consider the following two events that hold with high probability:
\begin{itemize}
    \item For any $r_x>2\sigma$, using \cref{lemma:delta_X-radius}, we have
    \begin{align*}
        \bbP\left( \|x_{1i} \|_2 \leq r_x,\ \forall i\in[m] \right) \geq 1 - m \cdot \exp\left( - \frac{d_x(r_x-2\sigma)^2}{128\sigma^2} \right).
    \end{align*}
    
    \item In \cref{lemma:er-point-wise-lhs} we have established for a fixed $f \in \cF$ and $C > 1$ that
    \begin{equation*}
        \bbE\left[\frac{1}{T}\sum_{t \in [T]}w_t \cdot \nmm{G_{f,t}(x_1)}_{2}^2\right] \leq \frac{C}{T} \sum_{(t,i)\in \cM} \frac{w_t}{n_t } \cdot \nmm{G_{f,t}(x_{1i})}_2^2
    \end{equation*}
    holds with probability at least $1 - \exp\left( - \frac{(C-1)\min_{t\in[T]} n_t}{2C^2\kappa^2} \right)$. Applying the union bound with the $\varepsilon$-net $\cF(\varepsilon)$, we see that the following holds with probability at least $1-|\cF(\varepsilon)| \cdot \exp\left( - \frac{(C-1)\min_{t\in[T]} n_t}{2C^2\kappa^2} \right)$:
    \begin{align}\label{eq:tmp-union-bound-E1}
        \bbE\left[\frac{1}{T}\sum_{t \in [T]}w_t \cdot \nmm{G_{f',t}(x_1)}_{2}^2\right] \leq \frac{C}{T} \sum_{(t,i)\in \cM} \frac{w_t}{n_t } \cdot \nmm{G_{f',t}(x_{1i})}_2^2, \quad \forall f'\in \cF(\varepsilon).
    \end{align}
\end{itemize}
Under these two events, we seek to establish
the upper bound $Q(f,C)$ for general $f$:
\begin{align*}
    Q(f,C)&= \frac{1}{T}   \sum_{(t,i)\in \cM} \frac{w_t}{n_t }  \left(\bbE\left[\nmm{ G_{f,t}(x_1) }_2^2\right] - C \cdot \nmm{G_{f,t}(x_{1i})   }_2^2  \right) \\ 
    &\overset{\text{(i)}}{\leq} B_{\text{sq}} + \frac{1}{T}   \sum_{(t,i)\in \cM} \frac{w_t}{n_t }  \left(\bbE\left[\nmm{ \tilde{G}_{f,t}(x_1) }_2^2\right] - C \cdot \nmm{G_{f,t}(x_{1i})   }_2^2  \right) \\
    &\overset{\text{(ii)}}{\leq} B_{\text{sq}} + \frac{1}{T}   \sum_{(t,i)\in \cM} \frac{w_t}{n_t }  \left(\bbE\left[\nmm{ \tilde{G}_{f,t}(x_1) }_2^2\right] - C \cdot \nmm{\tilde{G}_{f,t}(x_{1i})   }_2^2  \right) \\
    &\overset{\text{(iii)}}{\leq} B_{\text{sq}} + \frac{1}{T}   \sum_{(t,i)\in \cM} \frac{w_t}{n_t }  \left(\bbE\left[\nmm{ \tilde{G}_{f',t}(x_1) }_2^2\right] - C \cdot \nmm{\tilde{G}_{f',t}(x_{1i})   }_2^2  \right) \\
    & \quad + (C+1) w_{\text{avg}} K_G\nmm{f - f'}_{\cF}^2 \\
    &\overset{\text{(iv)}}{\leq} B_{\text{sq}} + \frac{1}{T}   \sum_{(t,i)\in \cM} \frac{w_t}{n_t }  \left(\bbE\left[\nmm{ \tilde{G}_{f',t}(x_1) }_2^2\right] - C \cdot \nmm{G_{f',t}(x_{1i})   }_2^2  \right) \\
    & \quad + (C+1) w_{\text{avg}} K_G\nmm{f - f'}_{\cF}^2 \\
    &\overset{\text{(v)}}{\leq} 2 B_{\text{sq}} + \frac{1}{T}   \sum_{(t,i)\in \cM} \frac{w_t}{n_t }  \left(\bbE\left[\nmm{ G_{f',t}(x_1) }_2^2\right] - C \cdot \nmm{G_{f',t}(x_{1i})   }_2^2  \right) \\
    & \quad + (C+1) w_{\text{avg}} K_G\nmm{f - f'}_{\cF}^2 \\ 
    &\overset{\text{(vi)}}{\leq} 2 B_{\text{sq}}  + (C+1) w_{\text{avg}} K_G\nmm{f - f'}_{\cF}^2 \\
    &\overset{\text{(vii)}}{\leq} 2 B_{\text{sq}}  + (C+1) w_{\text{avg}} K_G\varepsilon^2. 
\end{align*}
Here, (i) follows from the definitions of $B_{\text{sq}}$ and $Q(f,C)$, (iii) follows by summing the two inequalities in \cref{eq:tmp-G-square-diff} and rearranging, (ii) and (iv) follow from the event $x_{1i} \leq r_x$ ($\forall i\in [m]$), which implies $G_{f,t}(x_{1i})=\tilde{G}_{f,t}(x_{1i})$ and $G_{f',t}(x_{1i})=\tilde{G}_{f',t}(x_{1i})$, (v) follows again from the definition of $B_{\text{sq}}$,  (vi) follows from the event that \cref{eq:tmp-union-bound-E1} holds, and (vii) follows from the definition of the $\varepsilon$-net $\cF(\varepsilon)$. 

By an application of the union bound, we arrive at
\begin{align*}
    \bbP\left( Q(f,C) >  2 B_{\text{sq}}  + (C+1) w_{\text{avg}} K_G\varepsilon^2,  \forall f\in \cF \right) &\leq \left|\cF(\varepsilon) \right| \cdot \exp\left( - \frac{(C-1)\min_{t\in[T]} n_t}{2C^2\kappa^2} \right) \\
    &\ \ + m \cdot \exp\left( - \frac{d_x(r_x-2\sigma)^2}{128\sigma^2} \right).
\end{align*}
The proof is now complete.
\end{proof}

\begin{lemma}\label{lemma:er-point-wise-lhs}
    Recall the definition of $\kappa$ in \cref{eq:def:kappa}:  
    $$\kappa = \sup_{f \in \cF\ \backslash \{f^*\}}\sup_{t \in [T],w_t>0}\frac{\bbE\left[\nmm{G_{f,t}(x_1)}_2^4\right]^{1/2}}{\bbE\left[\nmm{G_{f,t}(x_1)}_2^2\right]}.$$
    For any fixed $f \in \cF\ \backslash \{f^*\}$ and $C > 1$ we have 
    \begin{align*}
        \bbE\left[\frac{1}{T}\sum_{t \in [T]}w_t \cdot \nmm{f(x_t) - f^*(x_t)}_{2}^2\right] \leq \frac{C}{T} \sum_{(t,i)\in \cM} \frac{w_t}{n_t } \cdot \nmm{f^*(x_{ti}) - f(x_{ti})}_2^2
    \end{align*}
    with probability at least 
    \begin{align*}
        1 - \exp\left( - \frac{(C-1)\min_{t\in[T], w_t>0} n_t}{2C^2\kappa^2} \right).
    \end{align*}
\end{lemma}
\begin{proof}
    In the proof we assume $w_t> 0$ for all $t\in [T]$, as the case with some of the $w_t$'s being zero follows immediately. Define $F:= f - f^*$. Then we need to compute the probability of the event
    \begin{align}\label{eq:ER-eq-tmp}
        \left\{\bbE\left[\frac{1}{T}\sum_{t \in [T]}w_t \cdot \nmm{F(x_t)}_{2}^2\right] \leq \frac{C}{T} \sum_{(t,i)\in \cM} \frac{w_t \cdot \nmm{ F (x_{ti}) }_2^2}{n_t } \right\},
    \end{align}
    or equivalently
    \begin{align}
        \left\{\bbE\left[\sum_{t \in [T]}w_t \cdot \nmm{F(x_t)}_{2}^2\right] \leq C \cdot \sum_{(t,i)\in \cM} \frac{w_t \cdot \nmm{ F (x_{ti}) }_2^2}{n_t } \right\}.
    \end{align}

    Applying the Chernoff-Cram{\'e}r bound gives
    \begin{align*}
        &\ \bbP \left( C \sum_{(t,i)\in \cM} \frac{w_t \cdot \nmm{ F (x_{ti}) }_2^2}{n_t } \leq \sum_{t\in [T]}\bbE\Big[ w_t \cdot \| F (x_t) \|_2^2 \Big] \right) \\
        \leq&\ \inf_{\alpha > 0}\exp\left(\alpha \cdot \sum_{t\in [T]}\bbE\Big[ w_t \cdot \| F (x_t) \|_2^2 \Big] \right) \cdot \bbE\left[ \exp\left(-\alpha \cdot C \sum_{(t,i)\in \cM} \frac{w_t \cdot \nmm{ F (x_{ti}) }_2^2}{n_t } \right) \right].
    \end{align*}
    We proceed by bounding the rightmost term:
    \begin{align*}
        &\ \bbE\left[ \exp\left(-\alpha \cdot C \sum_{(t,i)\in \cM} \frac{w_t \cdot \nmm{ F (x_{ti}) }_2^2}{n_t } \right) \right] \\
        =&\ \bbE\left[ \exp\left( -C\alpha \sum_{i\in [m]} \sum_{ t \in \cM_i} \frac{w_t \cdot \nmm{ F (x_{ti}) }_2^2}{n_t }  \right) \right]  \\ 
        \overset{\text{(i)}}{=} &\ \prod_{i\in [m]}  \bbE\left[ \exp\left( -C\alpha  \sum_{ t \in \cM_i} \frac{w_t \cdot \nmm{ F (x_{ti}) }_2^2}{n_t }  \right) \right] \\
        \overset{\text{(ii)}}{\leq} &\   \prod_{i\in [m]} 
        \exp\left( - C\alpha  \sum_{ t \in \cM_i} \frac{ \bbE \left[w_t \cdot \nmm{ F (x_{ti}) }_2^2\right]}{n_t } + \frac{C^2\alpha^2}{2} \kappa^2  \left( \sum_{ t \in \cM_i} \frac{ \bbE \left[w_t \cdot \nmm{ F (x_{ti}) }_2^2\right]}{n_t } \right)^2  \right) \\
        = &\   
        \exp\left( - C\alpha \sum_{t\in [T]}\bbE\Big[ w_t \cdot \| F (x_t) \|_2^2 \Big] + \frac{C^2\alpha^2}{2} \kappa^2  \sum_{i\in [m]}  \left( \sum_{ t \in \cM_i} \frac{ \bbE \left[w_t \cdot \nmm{ F (x_{t}) }_2^2\right]}{n_t } \right)^2  \right) \\ 
        \overset{\text{(iii)}}{\leq} &\  \exp\left( - C\alpha \sum_{t\in [T]}\bbE\Big[ w_t \cdot \| F (x_t) \|_2^2 \Big] + \frac{C^2\alpha^2 \kappa^2 }{2\min_{t\in[T]} n_t } \left(  \sum_{t\in [T]} \bbE\Big[ w_t \cdot \| F (x_t) \|_2^2 \Big] \right)^2  \right)
    \end{align*}
    where (i) follows from the assumption that the data indexed by $\cR_1$ are i.i.d., and therefore the trajectories are independent (\cref{lemma:independent_trajectory}), (ii) follows from \cref{lemma:single-trajectory}, and (iii) follows from \cref{lemma:quad<quad}. Combining the above gives
    \begin{align*}
        &\ \bbP \left( C \sum_{(t,i)\in \cM} \frac{w_t \cdot \nmm{ F (x_{ti}) }_2^2}{n_t } \leq \sum_{t\in [T]}\bbE\Big[ w_t \cdot \| F (x_t) \|_2^2 \Big] \right) \\
        \leq &\ \exp\left( - (C-1)\alpha \sum_{t\in [T]}\bbE\Big[ w_t \cdot \| F (x_t) \|_2^2 \Big] + \frac{C^2\alpha^2 \kappa^2 }{2\min_{t\in[T]} n_t } \left(  \sum_{t\in [T]} \bbE\Big[ w_t \cdot \| F (x_t) \|_2^2 \Big] \right)^2  \right).
    \end{align*}
    Since the above holds for any $\alpha>0$, we take 
    \begin{align*}
        \alpha = \frac{C-1}{C^2\sum_{t\in [T]}\bbE\Big[ w_t \cdot \| F (x_t) \|_2^2 \Big]} \cdot \frac{\min_{t\in[T]} n_t}{\kappa^2} 
    \end{align*}
    and plugging this value back we obtain the desired result.
\end{proof}

\begin{lemma}[Projection Difference]\label{lemma:proj_var_sg}
Let $z=[z_1,\dots,z_d]^\top$ be a sub-Gaussian vector with proxy variance $\sigma^2/d$ and independent coordinates. Let $h$ be a vector-valued function that is Lipschitz continuous with constant $L$. For any $r> 2\sigma$ we have
\begin{align*}
    \bbE\left[\left\|h(\cP_{\bbB_r(d)}(z)) - h(z)\right\|_2^2\right] &\leq L^2\left( \frac{128\sigma^2}{d} + \frac{2\cdot 64^2\sigma^4}{d^2(r-2\sigma)^2}  \right) \exp\left( - \frac{d (r-2\sigma)^2}{128\sigma^2}  \right). 
\end{align*}
Furthermore, if $r> 3\sigma$, then we have $\sigma^2 / (r-2\sigma)^2 \leq 1$ and thus



\begin{align*}
    \bbE\left[\left\|h(\cP_{\bbB_r(d)}(z)) - h(z)\right\|_2^2\right] &\leq  128\frac{L^2\sigma^2}{d^2}(d + 64) \exp\left( - \frac{d (r-2\sigma)^2}{128\sigma^2}  \right).
\end{align*}

\end{lemma}

\begin{proof}
Since $h$ is $L$-Lipschitz, we have
\begin{equation*}
    \bbE\left[\left\|h(\cP_{\bbB_r(d)}(z)) - h(z)\right\|_2^2\right] \leq L^2 \cdot \bbE\left[\left\|\cP_{\bbB_r(d)}(z) - z\right\|_2^2\right],
\end{equation*}
and therefore it remains to bound the expectation $\bbE\left[\left\|\cP_{\bbB_r(d)}(z) - z\right\|_2^2\right]$. If  $z \in \bbB_r(d)$, then we have this expectation equal to $0$. Hence we only need to consider $z \notin \bbB_r(d)$, which means
\begin{align*}
    \bbE\left[\left\|\cP_{\bbB_r(d)}(z) - z\right\|_2^2  \right] &= \bbE\left[\left\|\cP_{\bbB_r(d)}(z) - z\right\|_2^2 \ \big|\  z \notin \bbB_r(d) \right] \\
    &=\bbE\left[\left\| r \cdot \frac{z}{\| z \|_2}  - z\right\|_2^2 \ \big|\    z \notin \bbB_r(d) \right] \\ 
    &=\bbE\left[\left| r - \| z \|_2 \right|^2 \ \big|\  z \notin \bbB_r(d) \right] 
\end{align*}
Applying the \textit{Integral Identity} (\cref{lemma:E=P}) with some basic probability calculation yields
\begin{align*}
    \bbE\left[\left\|\cP_{\bbB_r(d)}(z) - z\right\|_2^2  \right] &=\int_{0}^{\infty} \bbP\left(\left| r - \| z \|_2 \right|^2  > \tau  \ \big|\  z \notin \bbB_r(d) \right)\ d\tau \\ 
    &=\int_{0}^{\infty} \bbP\left( \| z \|_2 - r   > \sqrt{\tau} \ \big|\  z \notin \bbB_r(d)  \right)\ d\tau \\ 
    &\leq \int_{0}^{\infty} \bbP\left( \| z \|_2 - r > \sqrt{\tau} \right)\ d\tau \\ 
    &= \int_{0}^{\infty} \bbP\left( \| z \|_2 - 2\sigma > \sqrt{\tau} + r - 2\sigma \right)\ d\tau \\ 
    &\leq  \int_{0}^{\infty}\exp\left( - \frac{d (\sqrt{\tau} + r - 2\sigma)^2}{128\sigma^2} \right) \ d\tau.
\end{align*}
Here, the last inequality follows from \cref{lemma:norm-concentration} with the assumption that $r>2\sigma$. 
It remains to evaluate this integral. Observe that
\begin{align*}
    \exp\bigg[ -d \cdot \frac{(\sqrt{\tau}+r-2\sigma)^2}{128\sigma^2} \bigg] &= \exp\bigg[ - \frac{d}{128\sigma^2}  \left(\tau + 2(r-2\sigma)\sqrt{\tau} + (r-2\sigma)^2 \right) \bigg]  \\ 
    &=\exp\left( - \frac{d (r-2\sigma)^2}{128\sigma^2}  \right) \cdot \exp\bigg[ - \frac{d}{128\sigma^2} \left(\tau + 2(r-2\sigma)\sqrt{\tau}  \right) \bigg] 
\end{align*}
and that by a change of variable $u^2=\tau$ we have
\begin{align*}
    \int_{0}^{\infty} \exp\bigg[ - \frac{d}{128\sigma^2} \left(\tau + 2(r-2\sigma)\sqrt{\tau}  \right) \bigg] \ d\tau &=  2 \int_{0}^{\infty} \exp\bigg[ - \frac{d}{128\sigma^2} \left(u^2 + 2(r-2\sigma)u  \right) \bigg] u  \ du \\ 
    &=\left( - \frac{128\sigma^2}{d}  \exp\bigg[ - \frac{d}{128\sigma^2} u^2 \bigg] \right) \bigg|_{0}^{\infty} \\ 
    &\quad \ \ +  2 \int_{0}^{\infty} \exp\bigg[ - \frac{d}{64\sigma^2} (r-2\sigma)u  \bigg] u  \ du \\ 
    &= \frac{128\sigma^2}{d} + 2 \int_{0}^{\infty} \exp\bigg[ - \frac{d}{64\sigma^2} (r-2\sigma)u  \bigg] u  \ du \\
    &=\frac{128\sigma^2}{d} + \frac{2\cdot 64^2\sigma^4}{d^2(r-2\sigma)^2} ,
\end{align*}
where the last equality is due to \cref{lemma:basic-integral}. In summary, we have shown
\begin{align*}
    \bbE\left[\left\|\cP_{\bbB_r(d)}(z) - z\right\|_2^2 \right] &\leq \left( \frac{128\sigma^2}{d} + \frac{2\cdot 64^2\sigma^4}{d^2(r-2\sigma)^2}  \right) \exp\left( - \frac{d (r-2\sigma)^2}{128\sigma^2}  \right).
\end{align*}Adding the Lipschitz parameter $L$ back finishes the proof.
\end{proof}


\begin{lemma}\label{lemma:b_sq_n}
    Consider the ball $\bbB_{r_x}(d_x)$ in $\bbR^{d_x}$ of radius $r_x$ with $r_x > 3\sigma$, and recall that $\cP_{\bbB_{r_x}(d_x)}$ denotes projection onto $\bbB_{r_x}(d_x)$. Write $\tilde{G}_{f,t} := G_{f,t} \circ \cP_{\bbB_{r_x}(d_x)}$.  Recall $w_{\text{avg}}:=\frac{1}{T}\sum_{t\in[T]} w_t$. and the definition of $B_{\text{sq}}$ in \cref{eq:def_Bsq}:
    \begin{align*}
        B_{\text{sq}}(r_x) = \sup_{f \in \cF}\left|\bbE\left[\frac{1}{T}\sum_{t \in [T]} w_t\left(\nmm{\tilde{G}_{f,t}(x_1)}_2^2-\nmm{G_{f,t}(x_1)}_2^2\right)\right]\right|.
    \end{align*} 
    Under \cref{assumption:G-Lipschitz,assumption:finite-moment} we have
    \begin{align}\label{eq:bound-Bsq1_n}
    \begin{split}
         B_{\text{sq}}(r_x) &\leq 
        128w_{\text{avg}}\frac{L_G^2\sigma^2}{d_x^2}(d_x + 64) \exp\left( - \frac{d_x(r_x-2\sigma)^2}{128\sigma^2}  \right) \\
        &\ + 16w_{\text{avg}}\frac{L_G\sigma}{d_x}\sqrt{2M_2(d_x + 64)}\exp\left( - \frac{d_x (r_x-2\sigma)^2}{256\sigma^2}  \right).
    \end{split}
    \end{align} 
\end{lemma}
\begin{proof}
    Note that for two random vectors $a,b$ we apply Jensen's and Cauchy-Schwartz inequality and obtain
    \begin{align*}
        \left| \bbE[\| a \|_2^2 - \| b \|_2^2 ] \right| &= \bbE[ \| a -b \|_2^2 ] + 2\cdot \left| \bbE[b^\top (a-b)] \right| \\ 
        &\leq  \bbE[ \| a -b \|_2^2 ] + 2\cdot \bbE[\|b\|_2^2]^{1/2} \cdot \bbE[ \| a - b \|_2^2 ]^{1/2}.
    \end{align*}
    With $a=\tilde{G}_{f,t} (x_1)$ and $b=G_{f,t}(x_1)$, the above inequality and \cref{lemma:proj_var_sg} with the assumption $r\geq 3\sigma$ give
    \begin{align*}
        \bbE\left[\nmm{\tilde{G}_{f,t} (x_1)}_2^2-\nmm{G_{f,t}(x_1)}_2^2\right]&\leq 
        128\frac{L_G^2\sigma^2}{d_x^2}(d_x + 64) \exp\left( - \frac{d_x(r-2\sigma)^2}{128\sigma^2}  \right) \\
        &\ + 16\frac{L_G\sigma}{d_x}\sqrt{2M_2(d_x + 64)}\exp\left( - \frac{d_x (r-2\sigma)^2}{256\sigma^2}  \right).
    \end{align*}
    Weighting the above inequality by $w_t$, averaging it over $t$, and taking the supremum over $\cF$, we finish the proof.
\end{proof}

\subsection{\cref{prop:final_e_2_n} and Its Proof}

\begin{proposition}\label{prop:final_e_2_n}
Recall that $M(f)$ is defined in \cref{eq:residual-term-Mf} as 
\begin{align*}
    M(f)=\frac{1}{T}   \sum_{(t,i)\in \cM} \frac{w_t}{n_t } \left[ 4 \cdot \IP{v_{ti}}{f(x_{ti}) -  f^*(x_{ti}) } - \nmm{ f^*(x_{ti}) - f(x_{ti}) }_2^2\right],
\end{align*}
where $v_{ti}$ conditioned on $x_{1i},\dots,x_{ti}$ is sub-Gaussian with proxy variance $\nu^2$.
Fix $u>0$. Let $r_x>2\sigma, r_v > 2\nu\sqrt{d_y}$, and $\varepsilon$-net $\cF(\varepsilon)$ be the smallest $\varepsilon$-net of $\cF\ \backslash \{f^*\}$. 
Then we have 
\begin{align*}
    \bbP\left( M(f) > w_{\text{avg}} K_G \varepsilon \left( 1 +4 r_v  \right) + u/2,\ \forall f\in \cF \right) 
    &\leq |\cF(\varepsilon)|\cdot \exp \left(- u \cdot \frac{T}{16\nu^2} \min_{t\in [T], w_t> 0} \frac{ n_t }{w_t} \right) \\
    &\ \ \ + m \cdot \exp\left( - \frac{d_x(r_x-2\sigma)^2}{128\sigma^2} \right)+ m \cdot \exp\left( - \frac{(r_v-2\nu \sqrt{d_y})^2}{128\nu^2} \right). 
\end{align*}

\end{proposition}

\begin{proof}
 For any $f\in\cF$, let $f'\in \cF(\varepsilon)$ be such that $\| f' - f \|_{\cF}\leq \varepsilon$. Write $\tilde{G}_{f,t}:= G_{f,t} \circ \cP_{\bbB_r(d_x)}$. Using \cref{assumption:G-Lipschitz}, the triangular inequality, and the assumption $\varepsilon\leq 1$, we arrive at
\begin{align*}
    \IP{v_{ti}}{\tilde{G}_{f,t}(x_{1i}) } - \IP{v_{ti}}{\tilde{G}_{f',t}(x_{1i}) } &\leq K_G\nmm{f - f'}_{\cF} \cdot \max_{(t, i) \in \M}\| v_{ti} \|_2 \leq K_G \varepsilon \cdot \max_{(t, i) \in \M}\| v_{ti} \|_2, \\ 
    \nmm{ \tilde{G}_{f',t}(x_{ti}) }_2^2 - \nmm{ \tilde{G}_{f,t}(x_{ti}) }_2^2 &\leq K_G \nmm{f - f'}_{\cF}^2 \leq K_G \varepsilon,
\end{align*}
which in turn implies
\begin{equation}\label{eq:tmp-G-tilde-lipschitz}
    \begin{split}
        \frac{1}{T}\sum_{(t, i) \in \M}\frac{w_t}{n_t} \cdot \left( 
    \IP{v_{ti}}{\tilde{G}_{f,t}(x_{1i}) } - \IP{v_{ti}}{\tilde{G}_{f',t}(x_{1i}) } \right) &\leq  w_{\text{avg}} K_G \varepsilon \cdot \max_{(t, i) \in \M}\| v_{ti} \|_2, \\ 
    \frac{1}{T}\sum_{(t, i) \in \M}\frac{w_t}{n_t} \cdot \left( \nmm{ \tilde{G}_{f',t}(x_{ti}) }_2^2 - \nmm{ \tilde{G}_{f,t}(x_{ti}) }_2^2  \right) &\leq  w_{\text{avg}} K_G \varepsilon.
    \end{split}
\end{equation}
We now consider three events that hold with high probability:
\begin{itemize}
    \item From \cref{lemma:fixed_Mf}, we see that for a fixed $f' \in \cF$ we have
    \begin{equation*}
        \bbP\left( M(f') \geq \frac{u}{2} \right) \leq \exp\left(- u \cdot \frac{T }{16\nu^2} \min_{t\in [T], w_t> 0} \frac{ n_t }{w_t} \right). 
    \end{equation*}
    Applying the union bound with the $\varepsilon$-net $\cF(\varepsilon)$, we obtain
    \begin{equation*}
        \bbP\left( M(f') \leq \frac{u}{2}, \forall f'\in \cF(\varepsilon) \right) \leq 1 -|\cF(\varepsilon)| \cdot \exp \left(- u \cdot \frac{T }{16\nu^2} \min_{t\in [T], w_t> 0} \frac{ n_t }{w_t} \right). 
    \end{equation*}
    \item For any $r_x > 2\sigma$ from \cref{lemma:delta_X-radius}, we have
    \begin{align*}
        \bbP\left( \|x_{1i} \|_2 \leq r_x,\ \forall i\in[m] \right) &\geq 1 - m\cdot \exp\left( - \frac{d_x(r_x-2\sigma)^2}{128\sigma^2} \right).
    \end{align*}
    \item Similarly, for any $r_{v} > 2\mu \sqrt{d_y}$,  using \cref{lemma:delta_X-radius}, we obtain
    \begin{align*}
        \bbP\left( \|v_{ti} \|_2 \leq r_v,\ \forall (t,i)\in \cM \right) \geq 
        1 - m\cdot \exp\left( - \frac{(r_v-2\nu \sqrt{d_y})^2}{128\nu^2} \right).
    \end{align*}
\end{itemize}
Under these three events, we derive that
\begin{align*}
    M(f) &= \frac{1}{T}   \sum_{(t,i)\in \cM} \frac{w_t}{n_t } \left[ 4 \cdot \IP{v_{ti}}{ \tilde{G}_{f,t}(x_{1i}) } - \nmm{ \tilde{G}_{f,t}(x_{1i})  }_2^2\right] \\ 
    &\overset{\text{(i)}}{\leq} w_{\text{avg}} K_G \varepsilon \left( 1 +4  \max_{(t, i) \in \M}\| v_{ti} \|_2  \right)  + \frac{1}{T}   \sum_{(t,i)\in \cM} \frac{w_t}{n_t } \left[ 4 \cdot \IP{v_{ti}}{ \tilde{G}_{f',t}(x_{1i}) } - \nmm{ \tilde{G}_{f',t}(x_{1i})  }_2^2\right] \\ 
    &\overset{\text{(ii)}}{\leq} w_{\text{avg}} K_G \varepsilon \left( 1 +4 r_v  \right)  + \frac{1}{T}   \sum_{(t,i)\in \cM} \frac{w_t}{n_t } \left[ 4 \cdot \IP{v_{ti}}{ G_{f',t}(x_{1i}) } - \nmm{ G_{f',t}(x_{1i})  }_2^2\right] \\ 
    &\overset{\text{(iii)}}{=}  w_{\text{avg}} K_G \varepsilon \left( 1 +4 r_v  \right)  + M(f') \\ 
    &\overset{\text{(iv)}}{\leq} w_{\text{avg}} K_G \varepsilon \left( 1 +4 r_v  \right) + u/2. 
\end{align*}
Above, (i) follows by weighting the inequalities in \cref{eq:tmp-G-tilde-lipschitz} properly and summing them together; (ii) follows as $\|x_{1i} \|_2 \leq r_x$ and $\| v_{ti} \|_2\leq r_v$, the former indicating $G_{f,t}(x_{1i})=\tilde{G}_{f,t}(x_{1i})$ and $G_{f',t}(x_{1i})=\tilde{G}_{f',t}(x_{1i})$; (iii) follows from the definition of $M(f)$; (iv) follows from the event $M(f') \leq u/2$. 
Now, applying the union bound, we obtain 
\begin{align*}
    \bbP\left( M(f) > w_{\text{avg}} K_G \varepsilon \left( 1 +4 r_v  \right) + u/2,\ \forall f\in \cF \right) 
    &\leq |\cF(\varepsilon)|\cdot \exp \left(- u \cdot \frac{T}{16\nu^2} \min_{t\in [T], w_t> 0} \frac{ n_t }{w_t} \right)\\
    &\ \ \ + m \cdot \exp\left( - \frac{d_x(r_x-2\sigma)^2}{128\sigma^2} \right)+ m \cdot \exp\left( - \frac{(r_v-2\nu \sqrt{d_y})^2}{128\nu^2} \right). 
\end{align*}
The proof is now complete.
\end{proof}

\begin{lemma}\label{lemma:fixed_Mf}
    Fix $f \in \cF$. Recall that $M(f)$ is defined in \cref{eq:residual-term-Mf} as 
    \begin{align*}
        M(f)=\frac{1}{T}   \sum_{(t,i)\in \cM} \frac{w_t}{n_t } \left[ 4 \cdot \IP{v_{ti}}{f(x_{ti}) -  f^*(x_{ti}) } - \nmm{ f^*(x_{ti}) - f(x_{ti}) }_2^2\right],
    \end{align*}
    where $v_{ti}$ conditioned on $x_{1i},\dots,x_{ti}$ is sub-Gaussian with proxy variance $\nu^2$. Then
    \begin{equation*}
        \bbP\left( M(f) \geq u \right) \leq \exp\left(- u \cdot \frac{T }{8\nu^2} \min_{t\in [T], w_t> 0} \frac{ n_t }{w_t} \right). 
    \end{equation*}
\end{lemma}
\begin{proof}
    In the proof we assume $w_t> 0$ for all $t\in [T]$, as the case with some of the $w_t$'s being zero follows immediately. The proof follows the strategy of Proposition F.2 of \citet{ziemann_tutorial_2024} but extends it for the case with weights $w_t$ and partial trajectories. Defining 
    \begin{align}\label{eq:a-range}
        a := \frac{T}{8\nu^2} \min_{t\in [T]} \frac{ n_t }{w_t},
    \end{align}
    we need to prove $\bbP\left( M(f) \geq u \right) \leq \exp\left(- u a \right)$. Since Markov's inequality implies
    \begin{align*}
        \bbP\left( M(f) \geq u \right) &=  \bbP\left( \exp\left(  a \cdot M(f) \right) \geq \exp(au) \right) \\ 
        &\leq \bbE\left[ \exp\left( a\cdot M(f) \right)  \right] \cdot \exp(-au), 
    \end{align*}
    it suffices to show $\bbE\left[ \exp\left( a\cdot M(f) \right)  \right]\leq 1$. To do so, we first bound each summand of $M(f)$.  For any fixed $f \in \cF$, write $F:=f-f^*$. Note that $v_{ti}$ is 
    conditioned on $x_{1i},\dots,x_{ti}$ is sub-Gaussian, thus, for any $(t,i)\in \cM$ we have
    \begin{equation}\label{eq:single-vti}
       \begin{split}
           &\ \bbE\left[ \exp\left(\frac{a}{T} \cdot \frac{w_t}{n_t } \left[ 4\cdot \IP{v_{ti}}{F(x_{ti}) } - \nmm{F(x_{ti})}_2^2\right]\right) \ \big|\ x_{1i},\dots,x_{ti}  \right]  \\ 
       \overset{\text{(i)}}{\leq} &\ \exp\left( \frac{\nu^2}{2}  \cdot \frac{16 a^2w_t^2}{T^2n_t^2 }  \nmm{F(x_{ti})}_2^2 - \frac{aw_t}{Tn_t } \cdot \nmm{F(x_{ti})}_2^2 \right) \\
       \overset{\text{(ii)}}{\leq} &\ 1,
       \end{split}
    \end{equation}
    where (i) follows from the property of sub-Gaussian vectors $v_{ti}$ as shown in \cref{eq:def-subG-vec}, and (ii) follows from the definition of $a$, that is $a \leq \frac{T}{8\nu^2} \cdot \frac{ n_t }{w_t}$, which implies $\frac{\nu^2}{2T}  \cdot \frac{16 a^2w_t^2}{T^2n_t^2 } - \frac{aw_t}{Tn_t }  \leq 0$. Define 
    \begin{align*}
        R_t:= \sum_{i\in \cR_t} \frac{a}{T} \cdot \frac{w_t}{n_t } \left[ 4\cdot \IP{v_{ti}}{F(x_{ti}) } - \nmm{F(x_{ti})}_2^2\right]
    \end{align*}
    and write $X_{1:t}:=\{ x_{1i},\dots,x_{ti} \}_{i\in[m]}$. Since for any $i\in \cR_t$ we have that $v_{ti}$'s are independent of each other, from \cref{eq:single-vti} it follows that 
    \begin{align}\label{eq:eq:a-row-vti}
        \bbE\left[ \exp\left( R_t \right) \ \big|\ X_{1:t} \right] \ \leq 1.
    \end{align}
    Note that $ a \cdot M(f) = \sum_{t\in [T]} R_t$, and we have
    \begin{equation}\label{eq:a-Mf<1}
        \begin{split}
            \bbE\left[ \exp\left( a \cdot M(f) \right)  \right] &=  \bbE\left[ \exp\left( \sum_{t\in [T]} R_t \right) \right]  \\
        &\overset{\text{(i)}}{=} \bbE \left[ \bbE\left[ \exp\left( \sum_{t\in [T]} R_t \right) \ \big| \  X_{1:T} \right] \right] \\ 
        &\overset{\text{(ii)}}{=}   \bbE \left[ \bbE\left[ \exp\left( \sum_{t\in [T-1]} R_t \right) \ \big| \  X_{1:T} \right] \right] \cdot  \bbE \left[ \bbE\left[ \exp\left(  R_T \right) \ \big| \  X_{1:T} \right] \right] \\ 
        &\overset{\text{(iii)}}{\leq} \bbE \left[ \bbE\left[ \exp\left( \sum_{t\in [T-1]} R_t \right) \ \big| \  X_{1:T} \right] \right] \\
        &\overset{\text{(iv)}}{=}  \bbE\left[ \exp\left( \sum_{t\in [T-1]} R_t \right)  \right] \\
        &\leq \cdots \leq 1.
        \end{split}
    \end{equation}
    Above, (i) and (iv) follow from the \textit{law of total expectation},  (ii) follows from the fact that $v_{Ti}$'s are independent of other noise vectors $v_{ti}$'s when conditioned on $X_{1:T}$, and (iii) follows from \cref{eq:eq:a-row-vti}, and the rest of the derivations follows a recursive application of the previous steps.
\end{proof}

\section{Full Statement of \cref{theorem:regularization} and Its Proof (Data-Dependent Regularization)}\label{section:regularization-full}
\begin{theorem}[Data-Dependent Regularization, Full Version of \cref{theorem:regularization}]\label{theorem:regularization-full}
    Recall the definitions of $M_2$ and $\kappa$ in \cref{eq:def-M2} and \cref{eq:def:kappa}: 
    \begin{align*}
        M_2=\sup_{f\in \cF} \sup_{t\in [T], w_t>0} \bbE\left[ \|G_{f,t}(x_1) \|_2^2 \right], \quad \kappa = \sup_{f \in \cF\ \backslash \{f^*\}}\sup_{t \in [T],w_t>0}\frac{\bbE\left[\nmm{G_{f,t}(x_1)}_2^4\right]^{1/2}}{\bbE\left[\nmm{G_{f,t}(x_1)}_2^2\right]}. 
    \end{align*}
    Let $\delta\in(0,1]$.  Let $\hat{\theta}_1, \dots, \hat{\theta}_T \in \Theta \subset \bbR^p$ be global minimizers of \cref{eq:obj-distill1} with $\beta_t=1/4^{T-t}$. Define $w_t:= 4^{T-t} \cdot \frac{n_t}{m} \cdot \beta_t$ and $w_{\text{avg}}:=\frac{1}{T}\sum_{t\in[T]} w_t$. Let $C$ be some constant with $C>1$. Define $n' := \min_{t \in [T]}n_t$,
    and $n'' := \min_{t: w_t > 0}(n_t / w_t)$.
    Suppose \cref{assumption:data,assumption:Theta,assumption:G-Lipschitz,assumption:finite-moment} hold and let $\sigma,\nu, r_x, K_G, L_G$ be defined therein. Let $K_G(r) \leq k_G r^\alpha$ for some $\alpha>0$, $k_G>0$, and $\underbar{r}_v := 2\nu \left[ \sqrt{d_y} + 8\sqrt{2\ln(4m/\delta)}\right].$

    Assume that 
    \begin{align}
    \begin{split}
    n'\geq\frac{2\kappa^2C^2}{C-1}&\left[\ln(4/\delta) + p\ln\left(1+\frac{2B(C+1)}{C(1+\underbar{r}_v)}Tn''\right)\right],\\
    &m \geq \frac{\sqrt{2}L_G\sigma \delta \sqrt{(d_x + 64)}}{e^{d_x/256}\sqrt{M_2}}.
    \end{split}
    \end{align}
    Then, with probability at least $1-\delta$ the quantity $\bbE\left[\frac{\sum_{t \in [T]} n_t \beta_t \cdot  \nmm{f^*(x_t) - \hat{f}_T(x_t)}_2^2}{n_1+\cdots + n_T}\right]$ is bounded above by 

   \begin{align}
        \begin{split}
        &\Big\{2\alpha C w_{\text{avg}} k_G\left(1 + 8\nu\left[1 + 8\sqrt{2\ln(4m/\delta)}\right]\right)
        \Big[\left(\left[3 + 16\sqrt{\frac{\ln(4m/\delta)}{d_x}}\right]^{\alpha} + \frac{1}{16^{\alpha}d_x^{\alpha/2}}\right) \\
        &\ \ + \left({\max\left\{\frac{256}{d_x}\ln\left(\frac{32L_G\sqrt{2M_2}}{\alpha 16^{\alpha}(C+1)w_{\text{avg}}k_G}\frac{d_x^{\alpha/2 - 1}\sqrt{d_x+64}}{\sigma^{\alpha-1}}m^2\right), 0\right\}}\right)^{\alpha/2}\Big]\Big\}\frac{\sigma^{\alpha}}{n_1+\cdots+n_T}  \max_{t\in [T]} \{ 2^{2(T-t)} \beta_t  \}\\
        &\ \ + 8C\nu^2\frac{pT \ln\left(1+\frac{2B(C+1)}{C(1+8\nu\left[1 + 8\sqrt{2\ln(4m/\delta)}\right])} \frac{Tm}{\max_{t\in T} 2^{2(T-t)}\beta_t} \right)}{n_1+\cdots+n_T}  + 8C\nu^2\frac{\ln(4/\delta)}{n_1+\cdots+n_T}  \max_{t\in [T]} \{ 2^{2(T-t)} \beta_t  \},
        \end{split}
    \end{align}
\end{theorem}

\begin{proof}[Proof for Global Minimizers of \cref{eq:obj-distill1}]
    We first analyze the case with objective \cref{eq:obj-distill2} that is associated with the regularizer $\Omega_T(\theta)= \sum_{t\in[T-1]} \sum_{i\in \cR_t} \beta_t \cdot \| f_{\theta}(x_{ti}) - f_{\hat{\theta}_{T-1}}(x_{ti}) \|_2^2$. Suppose we have just computed $\hat{\theta}_T$ in \textit{Step 1}. At that moment, we have access to all the $m$ samples of task $T$, but only to part of the samples from previous tasks. To unify the notation, we define $\cR_T:=[m]$ and $n_T:=m$. Write $\hat{f}_t:= f_{\hat{\theta}_t}$. Applying the inequality $\| a + b \|_2^2 - \| b \|_2^2 \geq (1-\frac{1}{s})\cdot \| a \|_2^2 - s \cdot \| b \|_2^2$ with $a= \hat{f}_T(x_{ti}) - f^*(x_{ti})$ and $b=f^*(x_{ti}) - \hat{f}_{T-1}(x_{ti})$, we obtain
    \begin{align*}
        &\ \| \hat{f}_T(x_{ti}) - \hat{f}_{T-1}(x_{ti}) \|_2^2 - \| f^*(x_{ti}) - \hat{f}_{T-1}(x_{ti}) \|_2^2 \\ 
        \geq &\ \left( 1 - \frac{1}{s} \right) \cdot \| \hat{f}_T(x_{ti}) -  f^*(x_{ti}) \|_2^2 - s \cdot \| f^*(x_{ti}) - \hat{f}_{T-1}(x_{ti}) \|_2^2.
    \end{align*}
    Applying this inequality to the regularization term $\Omega_T(\theta)= \sum_{t\in[T-1]} \sum_{i\in \cR_t} \beta_t \cdot \| f_{\theta}(x_{ti}) - f_{\hat{\theta}_{T-1}}(x_{ti}) \|_2^2$, we obtain 
    \begin{align*}
        &\ \Omega_T(f^*) - \Omega_T(\hat{f}_T) \\
        =& \  \sum_{t\in[T-1]} \sum_{i\in \cR_t}    \beta_t \left( \| f^*(x_{ti}) - \hat{f}_{T-1}(x_{ti}) \|_2^2 - \| \hat{f}_T(x_{ti}) - \hat{f}_{T-1}(x_{ti}) \|_2^2 \right) \\ 
        \leq &\ \sum_{t\in[T-1]} \sum_{i\in \cR_t}    \beta_t \left( s \cdot \| f^*(x_{ti}) - \hat{f}_{T-1}(x_{ti}) \|_2^2 - \left( 1 - \frac{1}{s} \right)  \cdot \| \hat{f}_T(x_{ti}) - f^*(x_{ti}) \|_2^2 \right).
    \end{align*}
    Similarly to the beginning of the proof of \cref{theorem:general-appendix_n}, we have
    \begin{align*}
        \sum_{i \in \cR_T} \beta_T \cdot \nmm{ f^*(x_{Ti}) - \hat{f}_T(x_{Ti}) }_2^2  \leq  \sum_{i \in [m]} 2\beta_T \cdot \IP{v_{Ti}}{\hat{f}_T(x_{Ti}) -  f^*(x_{Ti}) } + \Omega_T(f^*) - \Omega_T(\hat{f}_T).
    \end{align*}
    Defining 
    \begin{align}\label{eq:Mf_t}
        M_t(f) := \sum_{i \in [m]}  \left[ 2s \cdot \IP{v_{ti}}{f(x_{ti}) -  f^*(x_{ti}) } -  \nmm{ f^*(x_{ti}) - f(x_{ti}) }_2^2\right],
    \end{align}
    and combining the above with some rearrangements of terms and rescaling by $\frac{s}{s-1}$ yields
    \begin{equation}\label{eq:reg-recurrence}
        \begin{split}
            &\ \sum_{t\in [T]} \sum_{i \in \cR_t} \beta_t \cdot \nmm{ f^*(x_{ti}) - \hat{f}_T(x_{ti}) }_2^2 \\ 
            \leq &\ \frac{\beta_T}{s-1}\cdot M_T\big(\hat{f}_T \big) + \frac{s^2}{s-1} \sum_{t\in[T-1]} \sum_{i\in \cR_t}    \beta_{t-1} \cdot  \| f^*(x_{ti}) - \hat{f}_{T-1}(x_{ti}) \|_2^2
        \end{split}
    \end{equation}
    We can now unroll the above recurrence relation  and arrive at
    \begin{align}\label{eq:reg-recurrence-final}
        \sum_{t\in [T]} \sum_{i \in \cR_t} \beta_t \cdot \nmm{ f^*(x_{ti}) - \hat{f}_T(x_{ti}) }_2^2 
        &\leq \sum_{t\in [T]}  \frac{s^{2(T-t)}}{(s-1)^{T-t+1}}\beta_t \cdot M_t\big(\hat{f}_t\big). 
    \end{align}
    Set $s=2$, and the above becomes
    \begin{align}\label{eq:reg-recurrence-final-s=2}
         \sum_{t\in [T]} \sum_{i \in \cR_t} \beta_t \cdot \nmm{ f^*(x_{ti}) - \hat{f}_T(x_{ti}) }_2^2 
        &\leq \sum_{t\in [T]}  2^{2(T-t)} \beta_t \cdot M_t\big(\hat{f}_t\big). 
    \end{align}
    Write $\beta_t':=\frac{n_t}{m}\cdot \beta_t$, and we obtain
    \begin{align*}
        \sum_{t\in [T]} \sum_{i \in \cR_t} \frac{\beta_t'}{n_t } \cdot \nmm{ f^*(x_{ti}) - \hat{f}_T(x_{ti}) }_2^2 \leq  \sum_{t\in [T]}  \frac{2^{2(T-t)}\beta_t'}{ n_t } \cdot M_t\big(\hat{f}_t\big).
    \end{align*}
        
    

    Multiplying both sides by constant $C>1$, dividing it by $T$, adding $\bbE\left[ \sum_{t\in[T]}\beta_t' \cdot  \nmm{ f^*(x_{ti}) - \hat{f}_T(x_{ti}) }_2^2 \right]$, we obtain
    \begin{align*}
        \bbE\left[\frac{1}{T}\sum_{t \in [T]}\beta_t' \cdot \nmm{f^*(x_t) - \hat{f}_T(x_t)}_2^2\right] \leq Q(\hat{f}_T,C,\beta_t') +   C \cdot \sum_{t\in [T]} \frac{2^{2(T-t)} \beta_t'}{ T n_t } \cdot M_t\big(\hat{f}_t\big), 
    \end{align*}
    where $Q(f, C,\beta_t')$ is defined as 
    \begin{align*}
        Q(f, C,\beta_t'):=\frac{1}{T} \sum_{(t,i)\in \cM} \frac{\beta_t'}{n_t }  \left(\bbE\left[\nmm{f(x_t) - f^*(x_t)}_2^2\right] - C \cdot \nmm{f(x_{ti}) - f^*(x_{ti})  }_2^2 \right).
    \end{align*}
    To obtain a high probability bound of $Q(f, C,\beta_t')$, we can now invoke \cref{prop:final_e_1_n}, and to obtain a probabilistic bound of  $\sum_{t\in[T]} \frac{2^{2(T-t)} \beta_t'}{ T n_t } \cdot M_t\big(\hat{f}_t\big)$ we invoke \cref{prop:Mf_t}. These respectively give
    \begin{align*}
        \bbP\left( Q(f,C, \beta_t') >  2 B_{\text{sq}}  + (C+1) w_{\text{avg}} K_G\varepsilon^2,  \forall f\in \cF \right) &\leq \left|\cF(\varepsilon) \right| \cdot \exp\left( - \frac{(C-1)\min_{t\in[T]} n_t}{2C^2\kappa^2} \right) \\
        &\ \ + m \cdot \exp\left( - \frac{d_x(r_x-2\sigma)^2}{128\sigma^2} \right)
    \end{align*}
    and
    \begin{align*}
    \sum_{t\in[T]} \frac{2^{2(T-t)} \beta_t'}{Tn_t} \cdot M_t(f_t) \leq w_{\text{avg}} K_G \varepsilon \left( 1 +4 r_v  \right) + u/2,\quad \quad  \forall f_1,\dots,f_T\in \cF,
    \end{align*}
    with probability at least 
    $$|\cF(\varepsilon)|^T\cdot \exp \left(- u \cdot \frac{T}{16\nu^2} \min_{t\in [T]} \frac{ n_t }{2^{2(T-t)}\beta_t'} \right) + m \cdot \exp\left( - \frac{d_x(r_x-2\sigma)^2}{128\sigma^2} \right)+ m \cdot \exp\left( - \frac{(r_v-2\nu \sqrt{d_y})^2}{128\nu^2} \right).$$

    Then, following the identical idea of proving \cref{theorem:general-appendix_n} we obtain that, with probability at least $1-\delta$, the quantity $\bbE\left[\frac{1}{T}\sum_{t \in [T]}\beta_t' \cdot \nmm{f^*(x_t) - f_{\hat{\theta}_T}(x_t)}_2^2\right]$ is upper bounded by
     \begin{align}
        \begin{split}
        &\Big\{2\alpha C w_{\text{avg}} k_G\left(1 + 8\nu\left[1 + 8\sqrt{2\ln(4m/\delta)}\right]\right)
        \Big[\left(\left[3 + 16\sqrt{\frac{\ln(4m/\delta)}{d_x}}\right]^{\alpha} + \frac{1}{16^{\alpha}d_x^{\alpha/2}}\right) \\
        &\ \ + \left({\max\left\{\frac{256}{d_x}\ln\left(\frac{32L_G\sqrt{2M_2}}{\alpha 16^{\alpha}(C+1)w_{\text{avg}}k_G}\frac{d_x^{\alpha/2 - 1}\sqrt{d_x+64}}{\sigma^{\alpha-1}}m^2\right), 0\right\}}\right)^{\alpha/2}\Big]\Big\}\frac{\sigma^{\alpha}}{Tn''}\\
        &\ \ + 8C\nu^2\frac{pT \ln\left(1+\frac{2B(C+1)}{C(1+8\nu\left[1 + 8\sqrt{2\ln(4m/\delta)}\right])}Tn''\right)}{Tn''} + 8C\nu^2\frac{\ln(4/\delta)}{Tn''},
        \end{split}
    \end{align}
    with $n''$ now defined as $n''=\min_{t\in[T]} \frac{n_t}{2^{2(T-t)}\beta_t'}$.

    Substituting the identity $\beta_t'= \frac{n_t}{m } \beta_t$ with $n_T=m$  to the above result and rescaling by $T/(n_1+\cdots + n_T)$, we obtain that $\bbE\left[\frac{\sum_{t \in [T]} n_t \beta_t \cdot  \nmm{f^*(x_t) - \hat{f}_T(x_t)}_2^2}{n_1+\cdots + n_T}\right]$ is bounded above by 

   \begin{align}
        \begin{split}
        &\Big\{2\alpha C w_{\text{avg}} k_G\left(1 + 8\nu\left[1 + 8\sqrt{2\ln(4m/\delta)}\right]\right)
        \Big[\left(\left[3 + 16\sqrt{\frac{\ln(4m/\delta)}{d_x}}\right]^{\alpha} + \frac{1}{16^{\alpha}d_x^{\alpha/2}}\right) \\
        &\ \ + \left({\max\left\{\frac{256}{d_x}\ln\left(\frac{32L_G\sqrt{2M_2}}{\alpha 16^{\alpha}(C+1)w_{\text{avg}}k_G}\frac{d_x^{\alpha/2 - 1}\sqrt{d_x+64}}{\sigma^{\alpha-1}}m^2\right), 0\right\}}\right)^{\alpha/2}\Big]\Big\}\frac{\sigma^{\alpha}}{n_1+\cdots+n_T}  \max_{t\in [T]} \{ 2^{2(T-t)} \beta_t  \}\\
        &\ \ + 8C\nu^2\frac{pT \ln\left(1+\frac{2B(C+1)}{C(1+8\nu\left[1 + 8\sqrt{2\ln(4m/\delta)}\right])} \frac{Tm}{\max_{t\in T} 2^{2(T-t)}\beta_t} \right)}{n_1+\cdots+n_T}  + 8C\nu^2\frac{\ln(4/\delta)}{n_1+\cdots+n_T}  \max_{t\in [T]} \{ 2^{2(T-t)} \beta_t  \},
        \end{split}
    \end{align}
    We finish the proof with $\beta_t=1/4^{T-t}$.
\end{proof}

\begin{proof}[Proof for Global Minimizers of \cref{eq:obj-distill2}]
    We now consider objective \cref{eq:obj-distill2} that is associated with the regularizer $\Omega_T(\theta)= \sum_{t\in[T-1]} \sum_{i\in \cR_t} \beta_t \cdot \| f_{\theta}(x_{ti}) - f_{\hat{\theta}_{t}}(x_{ti}) \|_2^2$.  Write $\hat{f}_t:= f_{\hat{\theta}_t}$. Similarly to the proof of \cref{theorem:regularization-full}, we unify the notation by defining $\cR_T:=[m]$ and $n_T:=m$. Similarly to \cref{eq:reg-recurrence}, we have
    \begin{equation}\label{eq:reg-recurrence2}
        \begin{split}
            &\ \sum_{t\in [T]} \sum_{i \in \cR_t} \beta_t \cdot \nmm{ f^*(x_{ti}) - \hat{f}_T(x_{ti}) }_2^2 \\ 
        \leq&\ \frac{\beta_T}{s-1}\cdot M_T\big(\hat{f}_T \big) + \frac{s^2}{s-1} \sum_{t\in[T-1]} \sum_{i\in \cR_t}  \beta_t \cdot  \| f^*(x_{ti}) - \hat{f}_{t}(x_{ti}) \|_2^2. 
        \end{split}
    \end{equation}
    Note here that the rightmost term is with $\hat{f}_{t}$, not $\hat{f}_{T-1}$.
    
    Similarly to the beginning of the proof of \cref{theorem:general-appendix_n}, we have for every $t\in[T-1]$ that
    \begin{align*}
        \sum_{i \in \cR_t} \beta_t \cdot \nmm{ f^*(x_{ti}) - \hat{f}_t(x_{ti}) }_2^2 &\leq \sum_{i \in [m]} \beta_t \cdot \nmm{ f^*(x_{ti}) - \hat{f}_t(x_{ti}) }_2^2 \\
        &\leq \sum_{i \in [m]} 2\beta_t \cdot \IP{v_{ti}}{\hat{f}_t(x_{ti}) -  f^*(x_{ti}) } + \Omega_t(f^*) - \Omega_t(\hat{f}_t) \\ 
        &\leq \sum_{i \in [m]} 2\beta_t \cdot \IP{v_{ti}}{\hat{f}_t(x_{ti}) -  f^*(x_{ti}) } + \Omega_t(f^*) \\ 
        &= \sum_{i \in [m]} 2\beta_t \cdot \IP{v_{ti}}{\hat{f}_t(x_{ti}) -  f^*(x_{ti}) } \\ 
        & \quad \quad + \sum_{\tau\in[t-1]} \sum_{i\in \cR_t} \beta_{\tau} \cdot \| f^*(x_{\tau i}) - \hat{f}_{\tau}(x_{\tau i}) \|_2^2.
    \end{align*}
    Multiplying both sides by $s/(s-1)$ and rearranging the terms with the definition of $M_t(f)$ in \cref{eq:Mf_t}, we get
    \begin{align*}
        \sum_{i \in \cR_t} \beta_t \cdot \nmm{ f^*(x_{ti}) - \hat{f}_t(x_{ti}) }_2^2  &\leq  \frac{\beta_t}{s-1} \cdot M_t(\hat{f}_t)  + \frac{s}{s-1} \sum_{\tau\in[t-1]} \sum_{i\in \cR_{\tau}}  \beta_{\tau} \cdot \| f^*(x_{\tau i}) - \hat{f}_{\tau}(x_{\tau i}) \|_2^2.
    \end{align*}
    This inequality is of the form $a_t \leq c_t + \frac{s}{s-1} \sum_{\tau\in [t-1]} a_{\tau}$, and it implies 
    \begin{align*}
        \sum_{t\in [T-1]} a_t \leq c_{T-1} + \left(\frac{s}{s-1}+1\right) \sum_{t\in [T-2]} a_{t}=c_{T-1} + \frac{2s-1}{s-1} \sum_{t\in [T-2]} a_{t}. 
    \end{align*}
    Unrolling this recurrence relation gives
    \begin{align*}
        \sum_{t\in [T-1]} a_t \leq \sum_{t\in [T-1]}  \left( \frac{2s-1}{s-1} \right)^{T-t-1} c_t.
    \end{align*}
    We have just shown that
    \begin{align*}
        \sum_{t\in[T-1]} \sum_{i\in \cR_t}  \beta_t \cdot  \| f^*(x_{ti}) - \hat{f}_{t}(x_{ti}) \|_2^2 \leq \sum_{t\in [T-1]} \left( \frac{2s-1}{s-1} \right)^{T-t-1} \beta_t \cdot M_t(\hat{f}_t).
    \end{align*}
    Substituting this back to \cref{eq:reg-recurrence2} with $s=2$ yields
    \begin{align*}
        &\ \sum_{t\in [T]} \sum_{i \in \cR_t} \beta_t \cdot \nmm{ f^*(x_{ti}) - \hat{f}_T(x_{ti}) }_2^2 \\ 
        \leq&\ \frac{\beta_T}{s-1}\cdot M_T\big(\hat{f}_T \big) + \frac{s^2}{s-1} \sum_{t\in [T-1]} \left( \frac{2s-1}{s-1} \right)^{T-t-1} \beta_t \cdot M_t(\hat{f}_t) \\ 
        = &\  \beta_T\cdot M_T\big(\hat{f}_T \big) + 4 \sum_{t\in [T-1]} 3^{T-t-1} \beta_t \cdot M_t(\hat{f}_t) \\ 
        \leq &\ \sum_{t\in [T]} 4^{T-t} \beta_t \cdot M_t(\hat{f}_t) .
    \end{align*}
    This is identical to \cref{eq:reg-recurrence-final-s=2}, thus the rest of the steps follow directly from the previous proof.
\end{proof}

\subsection{\cref{prop:Mf_t} and Its Proof}
\begin{proposition}\label{prop:Mf_t}
    Recall that $M_t(f)$ is defined in \cref{eq:Mf_t} as 
\begin{align*}
    M_t(f)=\sum_{i \in [m]}  \left[ 2s \cdot \IP{v_{ti}}{f(x_{ti}) -  f^*(x_{ti}) } - \nmm{ f^*(x_{ti}) - f(x_{ti}) }_2^2\right]
\end{align*}
where $s=2$ and $v_{ti}$ is conditionally sub-Gaussian with proxy variance $\nu^2$.
Fix $u>0$. Recall $\beta_t>0$  and define $w_t:=s^{2(T-t)} \beta_t$ ($\forall t\in[T]$). Let $r_x>2\sigma, r_v > 2\nu\sqrt{d_y}$, and $\varepsilon$-net $\cF(\varepsilon)$ be the smallest $\varepsilon$-net of $\cF\ \backslash \{f^*\}$. Then we have
\begin{align*}
    \sum_{t\in[T]} \frac{w_t}{Tn_t} \cdot M_t(f_t) \leq w_{\text{avg}} K_G \varepsilon \left( 1 +4 r_v  \right) + u/2,\quad \quad  \forall f_1,\dots,f_T\in \cF,
\end{align*}
with probability at least 
$$|\cF(\varepsilon)|^T\cdot \exp \left(- u \cdot \frac{T}{4\nu^2s^2} \min_{t\in [T]} \frac{ n_t }{s^{2(T-t)}\beta_t} \right) + m \cdot \exp\left( - \frac{d_x(r_x-2\sigma)^2}{128\sigma^2} \right)+ m \cdot \exp\left( - \frac{(r_v-2\nu \sqrt{d_y})^2}{128\nu^2} \right).$$
\end{proposition}
\begin{proof}
    The proof is almost the same as \cref{prop:final_e_2_n}, except that we apply union bound
    for each $M_t(f)$ using $\varepsilon$-net; this gives us the exponent $T$ in the
    failure probability.
\end{proof}
\begin{lemma}\label{lemma:fixed_Mf_t}
    Fix $f_1,\dots,f_T \in \cF$. Recall $\beta_t>0$ and that $M_t(f)$ is defined in \cref{eq:Mf_t} as 
    \begin{align*}
        M_t(f)= \sum_{i \in [m]}  \left[ 2s \cdot \IP{v_{ti}}{f(x_{ti}) -  f^*(x_{ti}) } - (s-1)\cdot \nmm{ f^*(x_{ti}) - f(x_{ti}) }_2^2\right],
    \end{align*}
    where $v_{ti}$ conditioned on $x_{1i},\dots,x_{ti}$ is sub-Gaussian with proxy variance $\nu^2$. Then
    \begin{equation*}
        \bbP\left( \sum_{t\in [T]} \frac{s^{2(T-t)} \beta_t}{ T n_t } \cdot M_t\big(f_t\big) \geq u \right) \leq \exp\left(- u \cdot \frac{T(s-1) }{2\nu^2 s^2} \min_{t\in [T]} \frac{ n_t }{s^{2(T-t)} \beta_t} \right). 
    \end{equation*}
\end{lemma}
\begin{proof}
    This follows from a similar proof to \cref{lemma:fixed_Mf}. 
\end{proof}

\section{Discussion on Extension of Results}\label{section:extension_results}
In this section, we will discuss modifications of existing proofs to
accommodate noise in \eqref{eq:dependency}, i.e., noisy 
transformation of tasks rather than deterministic transformation. 
Finally, our \cref{assumption:Theta} requires
that the true predictor $f^*$ to be realizable in the class of learnable 
functions; here we will provide pointer to extend our results.

\myparagraph{Noisy tasks setting}
Consider the following model
\begin{align}
    y_t &= f^*(x_t) + v_t, \quad \quad \ \ \ \ \forall t\geq 1; \label{eq:regress_noisy} \\ 
    x_t &= g_t(x_1,\dots,x_{t-1}) + \eta_t, \quad \forall t >1 \label{eq:dependency_noisy},
\end{align}
where $\eta_t$ is conditionally independent sub-Gaussian noise with proxy variance 
$\sigma_{\eta}^2$. Our analysis relies primarily on invoking
\cref{prop:final_e_1_n}, \cref{prop:final_e_2_n}, 
\cref{prop:Mf_t} and \cref{prop:final_e_2_co}. The crucial step in establishing these intermediate results is obtaining tail bounds for
\begin{align}
Q(f,C)= \frac{1}{T}   \sum_{(t,i)\in \cM} \frac{w_t}{n_t }  \left(\bbE\left[\nmm{ f(x_t) }_2^2\right] - C \cdot \nmm{f(x_{ti})   }_2^2  \right).
\end{align}
The key difficulty arises from the temporal dependencies in the noisy setting. 
In the noiseless case analyzed previously, we could deterministically relate 
$f(x_t)$ to $x_1$,
and directly apply Lipschitz concentration bounds. For the case here with additive noise $\eta_t$, we can in principle make use of and extend the results extensively studied in the system identification
literature (see \cite{matni-ti-cdc19} for a detailed survey).



\myparagraph{Non-realizable setting}
We now consider the case where $f^*$ does not belong to the realizable 
function class, meaning our model is misspecified. 
In this setting, \eqref{eq:check_point} becomes
\begin{align}
\bbE\left[\frac{1}{T}\sum_{t \in [T]}w_t \cdot \nmm{\tilde{f}(x_t) - \hat{f}(x_t)}_2^2\right] 
&\leq \frac{1}{T}   \sum_{(t,i)\in \cM} \frac{w_t}{n_t }  \left(\bbE\left[\nmm{f(x_t) - \tilde{f}(x_t)}_2^2\right] - C \cdot \nmm{f(x_{ti}) - \tilde{f}(x_{ti})  }_2^2 \right) \\
&\ +   \frac{C}{T}   \sum_{(t,i)\in \cM} \frac{w_t}{n_t } \left[ 4 \cdot \IP{v_{ti}-f^*(x_{ti})}{f(x_{ti}) -  \tilde{f}(x_{ti}) } - \nmm{ \tilde{f}(x_{ti}) - f(x_{ti}) }_2^2\right] \\
&\ + 2C \lambda\left[\Omega_T(\tilde{f}) - \Omega_T(\hat{f})\right],
\end{align}
where $\tilde{f} = f_{\tilde{\theta}}$ and
$\tilde{\theta} \in \argmin_{\theta\in \Theta} \mathbb{E}\left[\frac{1}{T}   \sum_{(t,i)\in \cM} \frac{w_t}{n_t } \cdot \cL \Big( y_{ti},  f_{\theta}(x_{ti}) \Big) + \lambda \cdot \Omega_T(\theta)\right]$.

The key challenge arises in the term  
$\IP{v_{ti}-f^*(x_{ti})}{f(x_{ti}) -  \tilde{f}(x_{ti}) } - \nmm{ \tilde{f}(x_{ti}) - f(x_{ti}) }_2^2$, where the first argument in the inner product is no longer independent of $x_t$. This dependence is critical because the proof of \cref{prop:final_e_2_n} relied crucially on the independence between the first and second arguments of the inner product. 

To obtain concentration bounds for such complex dependent terms, we require sophisticated mixed-tail concentration results. The framework developed by \citet{maurer-pontil-et-al-nips21} provides such tools, and \citet{ziemann-et-al-icml24} present concrete results that appear well-suited to our setting. Adapting their approach to handle the dependence structure in our problem remains a promising direction for future work.

\section{Related Work on Distance Measures}\label{section:related-work-MTL-TL}
Here we review some relevant works on multitask learning, transfer learning, and domain adaptation, with a focus on the distance measures that arise in generalization bounds. Such bounds typically depend linearly on the distance $\dist(\pi_1,\pi_2)$ between two distributions, $\pi_1$ and $\pi_2$. In the main paper, we discussed the discrepancy distance in \cref{remark:discrepancy-intro,remark:discrepancy}. Furthermore, there have been multiple choices of such distance. Here, we examine a few other choices and evaluate whether they are suitable for our case.  

The first distance we consider is the so-called $\cH$-divergence \citep{ben2010theory,shui2019principled}. In \citet{ben2010theory} the $\cH$-divergence is defined for classification tasks, while here we consider regression losses. In \citet{shui2019principled}, the $\cH$-divergence is defined exactly the same as the discrepancy distance in \cref{remark:discrepancy}, with a difference that the $\ell_1$ loss is considered in \citet{shui2019principled}. Similarly to \cref{remark:discrepancy}, the discrepancy with the $\ell_1$ loss could grow unbounded, e.g., if the two distributions are $\cN(0,1)$ and $\cN(a, 1)$, respectively, with $a$ growing polynomially with the problem size (e.g., dimension, number of samples). The second distance is called $\cY$\textit{-discrepancy} \citep{Mohri-ICALT2012,Wang-JMLR2023}. It is similar to the discrepancy distance in \cref{remark:discrepancy}, and we can show it is unbounded for simple distributions as well.

\section{Basic Definitions and Auxiliary Lemmas}\label{section:basics}
Here we present some basic definitions and lemmas that serve as fundamental building blocks for our proof. In \cref{subsection:basic-def-lemma} we develop several lemmas tailored for our proof. In \cref{subsection:HDP}, we collect notations and results from high dimensional statistics.

\subsection{Notations and Lemmas}\label{subsection:basic-def-lemma}

Denote by $\cP_{\bbB_{r}(d)}(\cdot)$ the projection of its input onto $\bbB_{r}(d)$. Define the supremum of the second order moment of $G_{f,t}(x_1)$:
\begin{align}\label{eq:def-M2}
    M_2:=\sup_{f\in \cF} \sup_{t\in [T], w_t>0} \bbE\left[ \|G_{f,t}(x_1) \|_2^2 \right].
\end{align}
Note that $M_2$ is finite under \cref{assumption:finite-moment} (see \cref{lemma:finite-kappa-M2}).

\begin{lemma}
\label{lemma:independent_trajectory}
Let $\{x_{1i}\}_{i=1}^n$ be $n$ independent random variables in $\mathbb{R}^{d_x}$, each distributed according to $\pi_1$. Suppose we have deterministic functions $g_2,\dots,g_t$ and generate for each $i \in \{1,\dots,n\}$:
\begin{align}
x_{ki} = g_k \bigl( x_{1i}, x_{2i}, \dots, x_{k-1,i} \bigr)
\quad \text{for } k = 2,\dots,t.
\end{align}
Then, for each $i$, we have trajectories $\bigl(x_{1i}, x_{2i}, \dots, x_{ti}\bigr)$. These trajectories are also mutually independent, that is, with $i\neq j$ we have
\begin{align}
    (x_{1i}, x_{2i}, \dots, x_{ti}) \;\perp\; (x_{1j}, x_{2j}, \dots, x_{tj}).
\end{align}
\end{lemma}

\begin{proof}
By assumption, the random variables \(x_{1i}\) are mutually independent for \(i = 1,\dots,n\).  Then, for each $i$, note that the trajectory $\bigl(x_{1i}, x_{2i}, \dots, x_{ti}\bigr)$ is obtained via a deterministic transformation of $x_{1i}$, namely,
\begin{align}
x_{2i} = g_2(x_{1i}), 
\quad
x_{3i} = g_3(x_{1i}, x_{2i}),
\quad
\dots, 
\quad
x_{ti} = g_t\bigl(x_{1i}, \dots, x_{t-1,i}\bigr).
\end{align}
Hence the trajectory $\bigl(x_{1i}, x_{2i}, \dots, x_{ti}\bigr)$ is a measurable (deterministic) function of the initial variable $x_{1i}$. 
Since trajectory $\bigl(x_{1i}, \dots, x_{ti}\bigr)$ depends \emph{only} on $x_{1i}$, and $\bigl(x_{1j}, \dots, x_{tj}\bigr)$ depends \emph{only} on $x_{1j}$, the mutual independence of $x_{1i}$ and $x_{1j}$ for $i\neq j$ directly implies 
\begin{align}
\bigl(x_{1i}, \dots, x_{ti}\bigr) \;\perp\; \bigl(x_{1j}, \dots, x_{tj}\bigr).
\end{align}
This finishes the proof.
\end{proof}

\begin{lemma}\label{lemma:basic-integral} For all $c\neq 0$ we have the following identity:
    \begin{align*}
        \int \tau\cdot  \exp(c\tau) \ d \tau =  \frac{c\tau - 1}{c^2}\cdot  \exp(c\tau). 
    \end{align*}
    Moreover, if $c<0$ then we have
    \begin{align*}
        \int_0^{\infty} \tau\cdot  \exp(c\tau) \ d \tau =  \frac{1}{c^2}.
    \end{align*}
\end{lemma}
\begin{proof}
    The first identity follows from \textit{integration by parts}, or by observing that the derivative of $\frac{c\tau - 1}{c^2}\cdot  \exp(c\tau)$ with respect to $\tau$ is equal to $\tau\cdot  \exp(c\tau)$. The second identity follows by evaluating the integral at infinity and at $0$.
\end{proof}

\begin{lemma}\label{lemma:Ez2<Ez2}
    For $k$ non-negative random variables $z_1,\dots,z_k$ with $\bbE[z_i]\neq 0$ we have
    \begin{align*}
        \bbE \left[  \left( \sum_{i=1}^k  z_i \right)^2 \right] \leq \left(\sum_{i=1}^k \bbE[z_i] \right)^2 \cdot \max_{i\in [k]} \frac{\bbE[z_i^2 ]}{\bbE[z_i ]^2}. 
    \end{align*}
\end{lemma}
\begin{proof}
    We have 
    \begin{align*}
        \bbE (z_1+\dots+ z_k)^2 &= \sum_{i=1}^k \bbE[z_i^2] + \sum_{i\neq j} 2 \bbE[z_i z_j] \\ 
        &\overset{\text{(i)}}{\leq} \sum_{i=1}^k \bbE[z_i^2] + \sum_{i\neq j} 2 \sqrt{ \bbE[z_i^2 ] \bbE[z_j^2] } \\ 
        &\leq \sum_{i=1}^k \bbE[z_i]^2 \cdot \max_{i\in [k]} \frac{\bbE[z_i^2 ]}{\bbE[z_i ]^2} + \sum_{i\neq j} 2 \sqrt{ \bbE[z_i ]^2 \bbE[z_j]^2 } \cdot \max_{i\in [k]} \frac{\bbE[z_i^2 ]}{\bbE[z_i ]^2} \\ 
        &= \sum_{i=1}^k \bbE[z_i]^2 \cdot \max_{i\in [k]} \frac{\bbE[z_i^2 ]}{\bbE[z_i ]^2} + \sum_{i\neq j} 2  \bbE[z_i ] \bbE[z_j] \cdot \max_{i\in [k]} \frac{\bbE[z_i^2 ]}{\bbE[z_i ]^2} \\ 
        &=\left(\sum_{i=1}^k \bbE[z_i] \right)^2 \cdot \max_{i\in [k]} \frac{\bbE[z_i^2 ]}{\bbE[z_i ]^2} 
    \end{align*}
    where (i) follows from the Cauchy-Schwarz inequality.
\end{proof}
\begin{lemma}\label{lemma:single-trajectory}
    For $k$ non-negative random variables $z_1,\dots,z_k$  with $\bbE[z_i]\neq 0$, and for any fixed positive constant $\alpha > 0, C > 1$ we have
    \begin{align*}
         \bbE\left[ \exp\left(-C \alpha \sum_{i=1}^k z_i \right) \right] \leq \exp\left( -C\alpha \sum_{i=1}^k \bbE[ z_i ] + \frac{C^2\alpha^2}{2}  \left(\sum_{i=1}^k \bbE[z_i] \right)^2 \cdot \max_{i\in [k]} \frac{\bbE[z_i^2 ]}{\bbE[z_i ]^2} \right), 
     \end{align*}
     and moreover we have
    \begin{align*}
        \bbP\left( C\sum_{i=1}^k z_i \leq  \sum_{i=1}^k \bbE[ z_i ] \right) \leq \exp \left(- \frac{C-1}{2C^2 \max_{i\in [k]} \frac{\bbE[z_i^2 ]}{\bbE[z_i ]^2}} \right).
    \end{align*}
\end{lemma}
\begin{proof}
    For any $\alpha >0$, Markov's inequality gives
    \begin{equation*}
        \begin{split}
            \bbP \left( C\sum_{i=1}^k z_i \leq  \sum_{i=1}^k \bbE[ z_i ]  \right) &= \bbP \left( \exp \left(-C\alpha \sum_{i=1}^k z_i \right) \geq \exp\left(- \alpha \sum_{i=1}^k \bbE[ z_i ] \right) \right) \\ 
        &\leq \exp\left(\alpha \sum_{i=1}^k \bbE[ z_i ]  \right) \cdot \bbE\left[ \exp\left(-C \alpha \sum_{i=1}^k z_i \right) \right]. 
        \end{split}
    \end{equation*}
     We upper bound the rightmost terms:
     \begin{align*}
         \bbE\left[ \exp\left(-C \alpha \sum_{i=1}^k z_i \right) \right] &\overset{\text{(i)}}{\leq} \bbE\left[ 1 - C \alpha \sum_{i=1}^k z_i + \frac{C^2 \alpha^2}{2} \left( \sum_{i=1}^k z_i \right)^2  \right] \\ 
         &\overset{\text{(ii)}}{\leq} \exp\left( -C\alpha \sum_{i=1}^k \bbE[ z_i ] + \frac{C^2\alpha^2}{2} \bbE \left[  \left( \sum_{i=1}^k  z_i \right)^2 \right]  \right) \\ 
         &\overset{\text{(iii)}}{\leq} \exp\left( -C\alpha \sum_{i=1}^k \bbE[ z_i ] + \frac{C^2\alpha^2}{2}  \left(\sum_{i=1}^k \bbE[z_i] \right)^2 \cdot \max_{i\in [k]} \frac{\bbE[z_i^2 ]}{\bbE[z_i ]^2} \right) \\ 
     \end{align*}
     where (i) follows from the inequality $\exp(-a) \leq 1-a + a^2/2$ ($\forall a\geq0$), (ii) from the inequality $1+a\leq \exp(a)$ ($\forall a$), and (iii) from \cref{lemma:Ez2<Ez2}. Combining the above gives
     \begin{align*}
         \bbP \left( C\sum_{i=1}^k z_i \leq \sum_{i=1}^k \bbE[ z_i ] \right) \leq  \exp\left( -(C-1)\alpha \sum_{i=1}^k \bbE[ z_i ] + \frac{C^2\alpha^2}{2}  \left(\sum_{i=1}^k \bbE[z_i] \right)^2 \cdot \max_{i\in [k]} \frac{\bbE[z_i^2 ]}{\bbE[z_i ]^2} \right).
     \end{align*}
     Since the above holds for any $\alpha$ and since $C>1$, we set $\alpha$ to 
     \begin{align*}
         \alpha = \frac{C-1}{C^2\sum_{i=1}^k \bbE[ z_i ]} \cdot \frac{1}{\max_{i\in [k]} \frac{\bbE[z_i^2 ]}{\bbE[z_i ]^2}}
     \end{align*}
     and plug this back into the above inequality to obtain the desired result.
\end{proof}

\begin{lemma}\label{lemma:quad<quad}
    Let $S_1,\dots,S_m$ be subsets of $[T]$, and denote by $n_t$ the number of times that the index $t$ is contained in $S_1,\dots,S_m$, that is
    \begin{align*}
        n_t:= \sum_{i\in[m]} \mathbbm{1} ( t \in S_i).
    \end{align*}
    where $\mathbbm{1}(\cdot)$ is the \textit{binary} indicator function that outputs $1$ if its input statement is true, or outputs $0$ otherwise. Then, for $T$ non-negative numbers $a_1,\dots,a_T$, we have 
    \begin{align*}
        \sum_{i\in [m]} \left( \sum_{t\in S_i} \frac{a_t}{n_t} \right)^2 \leq  \frac{1}{\min_{t\in[T]}\{ n_t \}} \cdot \left( \sum_{i\in[T]} a_t \right)^2.
    \end{align*}
\end{lemma}
\begin{proof}
    We expand the square on the left-hand side and make the following observations:
    \begin{itemize}
        \item The left-hand side has the term $a_t^2/ n_t$, as each square gives $a_t^2 / n_t^2$ and there are $n_t$ such sets that contains index $t$. Therefore we have
        \begin{align*}
            \frac{a_t^2}{n_t} \leq \frac{a_t}{\min_{t\in [T]} \{ n_t \}}.
        \end{align*}
        \item For any $t_1,t_2\in [T]$ The left-hand side has the term $\frac{2a_{t_1}a_{t_2}}{n_{t_1} n_{t_2}} \cdot c_{t_1,t_2} $, where $c_{t_1,t_2}$ denotes the number of times that both indices $t_1,t_2$ are contained in the sets $S_1,\dots,S_m$, that is
        \begin{align*}
            c_{t_1,t_2} := \sum_{i\in [m]} \mathbbm{1} \left( t_1\in S_i \text{ and } t_2\in S_i \right).
        \end{align*}
        By definition we have $c_{t_1,t_2}\leq \min\{ n_{t_1},n_{t_2} \}$. Therefore 
        \begin{align*}
            \frac{2a_{t_1}a_{t_2}}{n_{t_1} n_{t_2}} \cdot c_{t_1,t_2} \leq \frac{2a_{t_1}a_{t_2}}{\max \{ n_{t_1}, n_{t_2} \}} \leq  \frac{2a_{t_1}a_{t_2}}{\min_{t\in [T]} \{ n_t \}}.
        \end{align*}
    \end{itemize}
    These prove the desired inequality.
\end{proof}

\begin{lemma}\label{lemma:finite-kappa-M2}
    Under \cref{assumption:finite-moment}, we have $\kappa<\infty$ and $M_2<\infty$, where $\kappa$ is defined in \cref{eq:def:kappa} and $M_2$ in \cref{eq:def-M2}.
\end{lemma}
\begin{proof}
    From Jensen's inequality $\bbE\left[ \|G_{f,t}(x_1) \|_2^2 \right] \leq \sqrt{\bbE\left[ \|G_{f,t}(x_1) \|_2^4 \right]}$ and \cref{assumption:finite-moment} that $\bbE\left[ \|G_{f,t}(x_1) \|_2^4 \right]$ is finite for all $f\in \cF$ and $t\in [T]$, we see that $M_2$ is indeed finite. That $\kappa$ is finite also follows from \cref{assumption:finite-moment}.
\end{proof}

\begin{proposition}\label{lemma:opt_r}
Suppose $\alpha, \beta, \gamma> 0$.
Whenever $r = \max\{\zeta, 0\} + \sqrt{\max\{\ln(\beta \gamma^{\alpha/2})/\gamma, 0\}}$
it holds that:
\begin{align}
\min_{r} r^{\alpha} + \beta\exp(- \gamma (r - \zeta)^2)
\leq \min\left\{\beta + \max\{\zeta^{\alpha}, 0\}, \frac{1}{\gamma^{\alpha/2}} + \left[\max\{\zeta, 0\}+ \sqrt{\max\left\{\frac{\ln\left(\beta \gamma^{\alpha/2}\right)}{\gamma}, 0\right\}}\right]^{\alpha}\right\}.
\end{align}
\end{proposition}
\begin{proof}

We first consider the case when $\zeta > 0$,
we obtain:
\begin{align*}
\left(r^{\alpha} + \beta\exp(- \gamma (r - \zeta)^2)\right)
&= \left[\zeta + \sqrt{\max\{\ln(\beta \gamma^{\alpha/2})/\gamma, 0\}}\right]^{\alpha}
+ \frac{1}{\gamma^{\alpha/2}}.
\end{align*}

Whenever, $\zeta \leq 0$ we can upper
bound the function on $r$ by dropping $\zeta$:
\begin{align*}
\left(r^{\alpha} + \beta\exp(- \gamma (r - \zeta)^2)\right)
& \leq \left(r^{\alpha} + \beta\exp(- \gamma r^2)\right).
\end{align*}
Now this reduces to the case when $\zeta = 0$.
This completes the proof.
\end{proof}

\subsection{Definitions and Lemmas from High-Dimensional Statistics}\label{subsection:HDP}
Here we collect some basic definitions used and auxiliary lemmas from high-dimensional statistics and probability. For a more detailed account on this subject, see, e.g., \citet{Rigollet-arXiv2023,Vershynin-HDP2018,Wainwright-book2019}. We mostly follow the presentation of \citet{Rigollet-arXiv2023}. For a few lemmas that we did not find exact statements or proofs in \citet{Rigollet-arXiv2023,Vershynin-HDP2018,Wainwright-book2019}, we provide independent proofs.

\myparagraph{Sub-Gaussian Random Variables} A random variable $\xi$ is called \textit{sub-Gaussian} with proxy variance $\sigma^2$ if it has mean zero and  satisfies
\begin{align}\label{eq:def-subG}
    \bbE[  \exp( u \xi ) ] \leq \exp \left( \frac{\sigma^2 u^2}{2} \right).
\end{align}

A random vector $z=[z_1,\dots,z_d]^\top$ is called sub-Gaussian with proxy variance $\sigma^2$ if it has mean zero and if $z^\top \omega $ is sub-Gaussian with proxy variance $\sigma^2$ for every unit vector $\omega\in \bbR^d$ (with $\| \omega \|_2=1$). It follows from \cref{eq:def-subG} that, if a sub-Gaussian vector $z$ has independent coordinates then 
\begin{align}\label{eq:def-subG-vec}
    \bbE[  \exp( z^\top \omega ) ] \leq \exp \left( \frac{\sigma^2 \cdot \| \omega \|_2^2 }{2} \right),\quad \forall \omega\in \bbR^d.
\end{align}
In particular, taking $\omega$ to be standard basis vectors of $\bbR^d$, we see that every entry $z_i$ is a sub-Gaussian random variable with proxy variance $\sigma^2$.
 
\myparagraph{Sub-exponential Random Variables} A random variable $\xi$ is called \textit{sub-exponential} with parameter $s$ if it has zero mean and satisfies
\begin{align}
    \bbE[ \exp(u\xi)] \leq \exp\left(  \frac{s^2 u^2}{ 2 } \right), \quad \forall u \in \left[ -\frac{1}{s},  \frac{1}{s} \right] 
\end{align}
This defining property is identical to \cref{eq:def-subG}, with one difference that $u$ is now constrained to lie in the interval $\left[ -\frac{1}{s},  \frac{1}{s} \right]$.

We now state a few lemmas about sub-Gaussian and sub-exponential random variables.

\begin{lemma}[see, e.g., Lemmas 1.3 and 1.4 of \citet{Rigollet-arXiv2023}]\label{lemma:subG-basic}
    For a sub-Gaussian random variable $\xi$ with proxy variance $\sigma^2$, it holds that
    \begin{align*}
        \bbP \left(\xi > u \right) \leq \exp \left( - \frac{u^2}{2\sigma^2} \right), \forall u \in \bbR.
    \end{align*}
    Furthermore, we have $\bbE[\xi^2] \leq 4\sigma^2$.
\end{lemma}


\begin{lemma}[see, e.g., Lemma 1.12 of \citet{Rigollet-arXiv2023}]\label{lemma:subG->subE}
    If $\xi$ is sub-Gaussian with proxy variance $\sigma^2$, then $\xi^2 - \bbE[\xi^2]$ is sub-exponential with parameter $16\sigma^2$.
\end{lemma}

\begin{lemma}[Bernstein's inequality, see, e.g., Theorem 1.13 of \citet{Rigollet-arXiv2023} ]\label{lemma:Bernstein}
    Let $\xi_1,\dots,\xi_d$ be independent sub-exponential random variables, each with parameter $s$. Then, for any $u> 0$, we have
    \begin{align*}
        \bbP \bigg[ \Big| \sum_{i=1}^d \xi_i \Big| \geq u \bigg] \leq \exp\bigg[ - \frac{1}{2} \min\Big( \frac{u^2}{ s^2d }, \frac{u}{ s }  \Big)\bigg].
    \end{align*}
\end{lemma}

\begin{lemma}\label{lemma:norm-concentration}
    Suppose $z=[z_1,\dots,z_d]^\top$ is a sub-Gaussian vector with proxy variance $\sigma^2/d$ and independent coordinates. Then we have
    \begin{align*}
        \bbP\bigg[ \| z \|_2 - 2\sigma \geq u \bigg] &\leq \exp\bigg[ - \frac{du^2}{128\sigma^2} \bigg], \quad \forall u>0. 
    \end{align*}
\end{lemma}
\begin{proof}
    From \cref{lemma:subG->subE} we know that $z_i^2 - \bbE[z_i^2]$ is sub-exponential with parameter $16\sigma^2 /d$. From \cref{lemma:subG-basic} we know that $\bbE[\| z\|_2^2 ] \leq 4\sigma^2$. Moreover, by our assumption, $z_1^2 - \bbE[z_1^2], \dots, z_d^2 - \bbE[z_d^2]$ are independent. Then we apply \cref{lemma:Bernstein} (Bernstein's inequality) and obtain
    \begin{align*}
        \bbP\bigg[ \frac{\| z \|_2^2}{4\sigma^2} - 1 \geq  \frac{ u }{4\sigma^2} \bigg] &= \bbP\bigg[ \| z \|_2^2 - 4\sigma^2 \geq u \bigg] \\
        &\leq \bbP\bigg[ \| z \|_2^2 - \bbE[\| z \|_2^2] \geq u \bigg] \\ &\leq \bbP\bigg[  \sum_{i=1}^d \left( z_i^2 - \bbE[z_i^2] \right)  \geq u \bigg] \\  
        &\leq \bbP \bigg[ \Big| \sum_{i=1}^d \left( z_i^2 - \bbE[z_i^2] \right) \Big| \geq u \bigg] \\
        &\leq \exp\bigg[ - \frac{d}{2} \min\Big( \frac{u^2}{ 256\sigma^4 }, \frac{u}{ 16\sigma^2 }  \Big)\bigg].
    \end{align*}
    We now use the technique of \cite[Theorem 3.1.1]{Vershynin-HDP2018} to obtain a bound for $\| z \|_2$. Since $\frac{\| z \|_2}{2\sigma}- 1 \geq \frac{u_1}{2\sigma}$ implies $\frac{\| z \|_2^2 }{4\sigma^2} - 1 \geq \frac{u_1^2}{4\sigma^2} + \frac{u_1}{\sigma} \geq \max\{ \frac{u_1^2}{4\sigma^2}, \frac{u_1}{2\sigma} \}$ for all $u_1>0$, we choose $u$ such that $\frac{u}{4\sigma^2} = \max\{ \frac{u_1^2}{4\sigma^2}, \frac{u_1}{2\sigma} \}$ and obtain
    \begin{align*}
        \bbP\bigg[ \| z \|_2 - 2\sigma \geq u_1 \bigg] &= \bbP\bigg[ \frac{\| z \|_2}{2\sigma} - 1 \geq  \frac{ u_1 }{2\sigma} \bigg] \\
        &\leq \bbP\bigg[ \frac{\| z \|_2^2}{4\sigma^2} - 1 \geq  \max\left\{ \frac{u_1^2}{4\sigma^2}, \frac{u_1}{2\sigma} \right\} \bigg] \\ 
        &\leq \exp\bigg[ - \frac{d}{2} \min\Big( \frac{\max\left\{ \frac{u_1^4}{16\sigma^4}, \frac{u_1^2}{4\sigma^2} \right\}}{ 16 }, \frac{\max\left\{ \frac{u_1^2}{4\sigma^2}, \frac{u_1}{2\sigma} \right\}}{ 4 }  \Big)\bigg] \\
        &\leq \exp\bigg[ - \frac{d}{8} \min\Big( \max\left\{ \frac{u_1^4}{16\sigma^4}, \frac{u_1^2}{4\sigma^2} \right\}, \max\left\{ \frac{u_1^2}{4\sigma^2}, \frac{u_1}{2\sigma} \right\}  \Big)\bigg] \\ 
        &= \exp\bigg[ - \frac{d}{32} \cdot \frac{u_1^2}{4\sigma^2} \bigg]. 
    \end{align*}
    The last equality follows from a simple discussion on whether $\frac{u_1}{2\sigma}$ is greater than $1$ or not. The proof is now complete.
\end{proof}

\begin{lemma}\label{lemma:delta_X-radius}
    Let $z_1,\dots,z_n$ be $d$-dimensional sub-Gaussian vectors, each with proxy variance $\sigma^2/d$ and independent coordinates. Then for any $r>2\sigma$ and some universal constant $c>0$ we have
    \begin{align*}
        \bbP\left( \| z_i \|_2 \leq r, \forall i\in [n]   \right) \geq 1 - n\cdot \exp\left( - \frac{d(r-2\sigma)^2}{128\sigma^2} \right).
    \end{align*}
\end{lemma}
\begin{proof}
    In \cref{lemma:norm-concentration}, take $u=\sigma(a-2)$ with $a>2$, and we obtain
    \begin{align*}
        \bbP\bigg[ \| z_i \|_2 \geq a \sigma \bigg] &\leq \exp\bigg[ - \frac{d(a-2)^2}{128} \bigg], \quad \forall a > 2. 
    \end{align*}
    Then take $r=a\sigma$, and we obtain
    \begin{align*}
        \bbP\bigg[ \| z_i \|_2 \geq r \bigg] &\leq \exp\bigg[ - \frac{d(r-2\sigma)^2}{128\sigma^2} \bigg], \quad \forall r > 2\sigma. 
    \end{align*}
    Applying the union bound finishes the proof.
\end{proof}

\myparagraph{Nets} In our proof, we will often consider some specific subset of the function class $\cF$, called $\epsilon$-net. The defining property of an $\epsilon$-net $\cF(\epsilon)$ is that, for every $f\in \cF$, there exists some $f'\in \cF(\varepsilon)$ satisfying $\| f' - f \|_{\cF}\leq \varepsilon$; thus, the definition of $\cF(\varepsilon)$ implicitly depends on the norm $\| \cdot \|_{\cF}$. 
\begin{lemma}\label{lemma:|F|<|Theta|}
    Suppose \cref{eq:f-Lipschiz} holds. Let $\Theta(\varepsilon):=\{ \theta_1,\dots,\theta_N \}$ be the smallest $\varepsilon$-net of $\Theta$. Then $\{ f_{\theta_1},\dots, f_{\theta_N}  \}$ is an ($L_{\cF} \varepsilon$)-net of $\cF$. Thus, the smallest $\varepsilon$-net of $\cF\ \backslash \{f^*\}$ has smaller size than the smallest ($\varepsilon/L_{\cF}$)-net of $\Theta$.
\end{lemma}
\begin{proof}
    For any $f_\theta\in \cF\ \backslash \{f^*\}$, there is some $\theta_i\in \Theta$ such that $\|\theta - \theta_i \|_{2}\leq \varepsilon$. By \cref{eq:f-Lipschiz} we have 
    \begin{align*}
        \| f_{\theta} - f_{\theta_i} \|_{\cF} \leq L_{\cF} \cdot \|\theta - \theta_i \|_{2}\leq L_{\cF} \cdot \varepsilon.
    \end{align*}
    The proof is finished.
\end{proof}
\begin{lemma}\label{lemma:epsilon-net-size}
    Assume $\Theta$ is a bounded subset of $\bbR^{p}$ with bounded diameter $\diam(\Theta)$ (\cref{assumption:Theta}). Then the smallest $\varepsilon$-net of $\Theta$ is at most of size 
    \begin{align*}
        \exp \left(   p \cdot O\left( \ln \frac{\diam(\Theta)}{\varepsilon} \right)   \right).
    \end{align*}
\end{lemma}
\begin{proof}
    Note that $\Theta$ can be covered by a ball of radius $\diam(\Theta)$. Then invoke, e.g., Proposition 5 of \citet{Cucker-2002}.
\end{proof}

\begin{lemma}[Integral Identity, cf. Lemma 1.2.1 of \citet{Vershynin-HDP2018}]\label{lemma:E=P}
    For a non-negative random variable $\xi$ we have
    \begin{equation*}
        \bbE[\xi] = \int_{0}^{\infty} \bbP(\xi >  \tau )\ d\tau.
    \end{equation*}
\end{lemma}

\end{document}